\documentclass[twoside]{article}

\usepackage{amsmath,amsfonts,bm}

\def\eqref#1{equation~\ref{#1}}

\def\1{\bm{1}}

\DeclareMathAlphabet{\mathsfit}{\encodingdefault}{\sfdefault}{m}{sl}
\SetMathAlphabet{\mathsfit}{bold}{\encodingdefault}{\sfdefault}{bx}{n}

\def\cA{{\mathcal{A}}}
\def\cB{{\mathcal{B}}}
\def\cC{{\mathcal{C}}}
\def\cD{{\mathcal{D}}}

\def\cO{{\mathcal{O}}}

\newcommand{\R}{\mathbb{R}}

\DeclareMathOperator*{\argmax}{arg\,max}
\DeclareMathOperator*{\argmin}{arg\,min}

\usepackage{url}

\usepackage[utf8]{inputenc} 
\usepackage[T1]{fontenc}    
\usepackage{url}            
\usepackage{booktabs}       

\usepackage{amsmath, amsfonts, amssymb}
\usepackage{amsthm}   
   
\usepackage{nicefrac}       
\usepackage{microtype}      
\usepackage{color,xcolor}         
\usepackage{xcolor}
\usepackage{makecell}
\usepackage{soul} 
\definecolor{myyellow}{rgb}{1,1,0.7}
\sethlcolor{myyellow}

\newcommand{\myref}[1]{$\left(\ref{#1}\right)$} 

\newcommand{\eqdef}{\coloneqq}

\usepackage{mathtools}   
\usepackage{algorithm}
\usepackage{algorithmic}

\usepackage{multirow}

\usepackage{colortbl}
\definecolor{bgcolor}{rgb}{0.93,0.99,1}
\definecolor{bgcolor2}{rgb}{0.8,1,0.8}
\definecolor{bgcolor3}{rgb}{0.50,0.90,0.50}

\usepackage{relsize} 
\newcommand{\algname}[1]{{\sf\red\relscale{0.90}#1}\xspace}

\newcommand{\gammaM}{{\gamma}}
\newcommand{\tauM}{{\tau}}

\newcommand{\nclients}{{M}}
\newcommand{\iclient}{m}
\newcommand{\kstep}{t}
\newcommand{\kcohort}{{C}}

\newcommand{\Koper}{{{H}}}
\newcommand{\localsteps}{K}
\newcommand{\comm}{T}
\newcommand{\betaM}{\delta}
\newcommand{\deltaM}{\alpha}

\newcommand{\localsolver}{\mathcal{A}}

\newcommand{\localsolmk}{y_{\iclient}^{\star,\kstep}}
\newcommand{\lastlocittermk}{y_{\iclient}^{\localsteps,\kstep}}
\newcommand{\lastlocitterk}{y^{\localsteps,\kstep}}
\newcommand{\localsolk}{y^{\star,\kstep}}

\newcommand{\set}{S^{t}}
\newcommand{\Exps}[1]{\mathbb{E}_S\!\left[ #1 \right]}

\newcommand{\localfun}{\psi^t}
\newcommand{\localfuni}{\localfun_\iclient}

\newtheorem{theorem}{Theorem}[section]

\newtheorem{corollary}[theorem]{Corollary}

\newtheorem{assumption}{Assumption}
\newtheorem{definition}{Definition}

\newcommand{\squeeze}{\textstyle} 
\usepackage[flushleft]{threeparttable} 

\usepackage{tcolorbox}
\usepackage{pifont}
\definecolor{mydarkgreen}{RGB}{39,130,67}
\definecolor{mydarkred}{RGB}{192,25,25}
\newcommand{\green}{\color{mydarkgreen}}
\newcommand{\red}{\color{mydarkred}}
\newcommand{\cmark}{{\green\ding{51}}}%
\newcommand{\xmark}{{\red\ding{55}}}%

\usepackage{color}

\usepackage[colorinlistoftodos,bordercolor=orange,backgroundcolor=orange!20,linecolor=orange,textsize=scriptsize]{todonotes}
\newcommand{\michal}[1]{\todo[inline]{\textbf{Michal: }#1}}
\newcommand{\peter}[1]{\todo[inline]{\textbf{Peter: }#1}}

\usepackage{thmtools} 
\usepackage{thm-restate}

\definecolor{lgray}{rgb}{0.95,0.95,0.95}
\definecolor{yel}{rgb}{1,0.98,0.92}
\definecolor{mydarkblue}{rgb}{0,0.2,0.9}

\definecolor{mydarkred}{rgb}{0.8,0.0,0.0}

\definecolor{mydarkgreen}{rgb}{0,0.55,0}

\usepackage{xspace}
\usepackage{scalefnt}
\definecolor{myorange}{RGB}{255,100,0}
\newcommand{\algn}[1]{{\sf\color{mydarkred}\scalefont{0.96}{#1}}\xspace}

\newcommand{\sqnorm}[1]{\left\| #1 \right\|^2}
\newcommand{\Exp}[1]{\mathbb{E}\!\left[ #1 \right]}
\newcommand{\EE}[2]{\mathbb{E}_{#1}\!\left[ #2 \right]}
\newcommand{\Rop}{\mathcal{P}}

\newcommand{\oma}{\omega_{\mathrm{ran}}}

\usepackage{hyperref}

\newcommand{\Kx}{\localsteps}
\newcommand{\Cx}{\kcohort}
\newcommand{\Mx}{\nclients}
\newcommand{\Tx}{\comm}

\usepackage[accepted]{aistats2023}

\usepackage[round,comma]{natbib}

\begin{document}

\twocolumn[

\aistatstitle{Can 5${}^{\rm th}$ Generation Local Training Methods Support Client Sampling? Yes!}

\aistatsauthor{Micha\l{} Grudzie\'{n} \And Grigory Malinovsky \And  Peter Richt\'arik }

\aistatsaddress{KAUST \& University of Oxford\footnote{} \And KAUST \And KAUST } ]

\begin{abstract}
The celebrated \algname{FedAvg} algorithm of \citet{FL2017-AISTATS} is based on three components: client sampling (CS), data sampling (DS) and local training (LT). While the first two are reasonably well understood, the third component, whose role is to  reduce the number of communication rounds needed to train the model, resisted all attempts at a satisfactory theoretical explanation.  \citet{ProxSkip-VR} identified four distinct generations of LT methods based on the quality of the provided theoretical communication complexity guarantees. Despite a lot of progress in this area, none of the existing works were able to show that it is theoretically better to employ multiple local gradient-type steps (i.e., to engage in LT) than to rely on a single local gradient-type step only in the important heterogeneous data regime.  In a recent breakthrough embodied in their \algname{ProxSkip} method and its theoretical analysis, \citet{ProxSkip} showed that LT indeed leads to provable communication acceleration for arbitrarily heterogeneous data, thus jump-starting the $5^{\rm th}$ generation of LT methods. However, while these latest generation LT methods are compatible with DS, none of them  support CS. We resolve this open problem in the affirmative. In order to do so, we had to base our algorithmic development on new algorithmic and theoretical foundations.

\end{abstract}

\section{Introduction} 

 {\em Federated learning} (FL) is an emerging paradigm for the training of supervised machine learning models over geographically distributed and often private datasets stored across a potentially very large number of clients' devices, such as mobile phones, edge devices and hospital servers. 
 
 The roots of this young field can be traced to four foundational papers dealing with federated optimization \citep{FEDOPT}, communication compression \citep{FEDLEARN},  federated averaging \citep{FL2017-AISTATS} and secure aggregation \citep{FL-secure_aggreg}\footnote{These four works are cited in the Google AI blog \citep{FLblog2017} which  originally announced FL to the general public.}.  
 
 Federated learning has grown massively
 since its inception---in volume, depth and breadth alike---with many advances in theory, algorithms, systems and practical applications~\citep{FL-big,FL_survey_2020,FieldGuide2021}. 

In this work we study the standard
optimization formulation of  federated learning, which has  the form
\begin{equation}
	\label{eq:main}
	\squeeze 
		\min \limits_{x \in \mathbb{R}^d}\left[f(x) \eqdef \frac{1}{\Mx} \sum \limits_{\iclient=1}^{\Mx}  f_\iclient(x)\right],
\end{equation}
where $\Mx$ is the number of clients/devices and each function 
$f_\iclient(x)\eqdef \EE{\xi\sim \cD_\iclient}{\ell(x,\xi)}$ 
represents the average loss, measured via the loss function $\ell$, of the model parameterized by $x\in \R^d$ over the training data $\cD_\iclient$ owned by client $\iclient \in [\Mx]\eqdef \{1,\dots,\Mx\}$.

\subsection{Federated averaging}

Proposed by \citet{FL2017-AISTATS}, federated averaging (\algname{FedAvg}) is an immensely popular method   specifically  designed to  solve problem \myref{eq:main} while being mindful of several constraints characteristic of practical federated environments. In particular, \algname{FedAvg} is based on gradient descent (\algname{GD}), 

but introduces three modifications: 

\phantom{XX} a) client sampling (CS), \\ 
\phantom{XX} b) data sampling (DS),  and \\ 
\phantom{XX} c) local training (LT). 

Training via \algname{FedAvg} proceeds in a number of communication rounds. Each round $t$ starts with the selection of a subset/cohort  $\set \subseteq [\Mx]$ of the clients of size $\Cx^t = |\set|$; these will  participate in the training in this round. The aggregating server then broadcasts the current version of the model, $x^t$,  to all clients $\iclient \in \set$ in the current cohort. Subsequently, each  client $\iclient \in \set$ performs $\Kx$ iterations  of    \algname{SGD} on its local loss function $f_\iclient$,  initiated with $x^t$, using minibatches $\cB_\iclient^{k,t}\subseteq \cD_m$ of size $b_\iclient=|\cB_\iclient^{k,t}|$  for $k=0,\dots,\Kx-1$. Finally, all participating devices send their updated models to the server for aggregation into a new model $x^{t+1}$, and the process is repeated. 

All three modifications can be turned on or off, individually, or in any combination. For example, if we set $\Cx^t=M$ for all $t$, then {\em all} clients are participating in all rounds, i.e., CS is turned off. 
Further, if we set $b_\iclient=|\cD_\iclient|$ for each client $\iclient \in [M]$, then all clients  use {\em all} their data to compute the local gradient estimator needed to perform each \algname{SGD} step, i.e., DS is turned off.  
Finally, if we set $\Kx=1$, then only a {\em single} \algname{SGD} step is taken by each participating client, i.e., LT is turned off. 
If all of these modifications are turned off, \algname{FedAvg} reduces to vanilla \algname{GD}.

\subsection{Client and data sampling}
While \citet{FL2017-AISTATS} provided convincing empirical evidence for the efficacy of  \algname{FedAvg}, their work did not contain any theoretical results. Much progress in FL in the last five years can be attributed to the efforts by the FL community to understand, analyze, and improve upon these mechanisms, often first in isolation, as this is easier when deep understanding is desired. 

Since  {\em unbiased} client and data sampling mechanisms  are  intimately linked to the stochastic approximation literature dating back to the work of \citet{RobbinsMonro:1951}, it is not surprising that CS and DS are relatively well understood. For example,  variants of \algname{SGD} supporting  virtually arbitrary unbiased CS and DS mechanisms have been analyzed by \citet{SGD-AS}  in the smooth strongly convex regime and by \citet{ES-SGD-nonconvex,OptClientSampling2020} in the smooth nonconvex regime. Oracle optimal\footnote{See also the earlier work of \citet{nonconvex_arbitrary}, who analyzed arbitrary sampling mechanisms in the smooth nonconvex regime with suboptimnal variance-reduced methods.} (in the smooth  nonconvex regime) variants of \algname{SGD} supporting  virtually arbitrary unbiased CS and DS mechanisms were proposed and analyzed by \citet{PAGE-AS}, who built upon the previous works of  \citet{PAGE2021,SPIDER} and \citet{SARAH}.  

However, all the works mentioned above  analyze  \algname{GD} + CS/DS only, with LT turned off. If LT is included in the mix as well, or even considered in isolation as a single add-on to vanilla \algname{GD}, significant technical issues arise. These issues have kept the FL community uneasy and therefore busy and immensely productive for many years. Since, as we shall see, this will be of crucial importance for us to motivate the contributions of this paper, we will now   outline the  development of the theoretical understanding of the LT mechanism by the FL community over the last seven years.

\subsection{Local training}
Local training---the practice of requiring each participating client to perform {\em multiple}  local optimization steps (as opposed to performing a {\em single} step only) based on their local data before communication-expensive parameter synchronization is allowed to take place---is one of the most practically useful algorithmic ingredients in the training of FL models. In fact, LT is so central to the practical success of FL, and so unique and novel within the trio (CS, DS and LT) of techniques forming the \algname{FedAvg} method, that many authors attach the prefix ``Fed'' (meaning ``federated'') to any optimization method performing some version of LT, whether CS and DS are present as well or not.

While LT was popularized by \citet{FL2017-AISTATS}, it was proposed in the same form before  \citep{Povey2015,SparkNet2016},  also without any theoretical justification\footnote{However, the even earlier and closely related line of work on the \algname{CoCoA} framework,  which is based on solving the dual problem using arbitrary local solvers, comes with solid theoretical justification \citep{cocoa,COCOA+,COCOA+journal}. Finally, we would be remiss if we did not mention that another related method was proposed and studied more than 25 years ago by \citet{Mang1995}.}. 
However, until recently, the empirically observed and often very significant communication-saving potential of LT remained elusive, escaping all attempts at a satisfying theoretical justification.  

\subsection{Five generations of local training methods}

We shall now briefly review the development of the theoretical understanding of LT  in the smooth strongly convex regime. We  follow the classification proposed by \citet{ProxSkip-VR}, who identified five distinct generations of LT methods---1) heuristic, 2) homogeneous, 3) sublinear, 4) linear, and 5) accelerated---each new improving upon the previous one in a certain important way.  

{\bf  1${}^{\rm st}$ generation of LT methods (heuristic).}  The 1${}^{\rm st}$ generation methods offer ample empirical evidence, but do not come with any convergence rates~\citep{Povey2015,SparkNet2016,FL2017-AISTATS}. 

{\bf  2${}^{\rm nd}$ generation of LT methods (homogeneous).}  The 2${}^{\rm nd}$ generation LT methods do provide guarantees, but their analysis crucially depends on one or another of the many incarnations of  data homogeneity assumptions, such as i) bounded gradients, i.e., requiring $\| \nabla f_\iclient(x)\|\leq c$ for all $\iclient\in [\Mx]$ and $x\in \R^d$ \citep{FedAvg-nonIID}, or ii) bounded gradient dissimilarity (a.k.a.\ strong growth), i.e., requiring $\frac{1}{\Mx}\sum_{\iclient=1}^\Mx \| \nabla f_\iclient(x) \|^2\leq c\|\nabla f(x)\|^2$ for all  $x\in \R^d$ \citep{LocalDescent2019}. This is  problematic since such assumptions  are prohibitively restrictive; indeed, they are typically not satisfied in real FL environments \citep{FL-big,FieldGuide2021}. 

{\bf  3${}^{\rm rd}$ generation of LT methods (sublinear).}  The 3${}^{\rm rd}$ generation LT theory managed to succeed in disposing of the problematic data homogeneity assumptions \citep{localGD,localSGD-AISTATS2020}. \citet{woodworth2020minibatch} and \citet{glasgow2022sharp} subsequently provided lower bounds for \algname{LocalGD} with DS, showing that its communication complexity is not better than that of minibatch \algname{SGD} in the heterogeneous data setting.  Additionally, \citet{LFPM} analyzed LT methods for general fixed point problems. 

Unfortunately, these results suggest that LT-enhanced \algname{GD}, often called \algname{LocalGD}, suffers from a  sublinear convergence rate, which is clearly inferior to the linear convergence rate of vanilla \algname{GD}. While removing the reliance on data homogeneity assumptions was clearly an important step forward, this rather pessimistic theoretical result seems to suggest that LT makes \algname{GD} worse. However, this is at odds with the empirical evidence, which maintains that LT enhances \algname{GD}, and often significantly so. For these reasons, theoreticians continued to soldier on, with the quest to at least close the theoretical gap between LT-based methods and vanilla \algname{GD}.

{\bf  4${}^{\rm th}$ generation of LT methods (linear).} These efforts led to the identification of the {\em client drift} phenomenon as the culprit responsible for the  gap, and to a solution based on various techniques for the reduction of client drift. This development  marks the start of the  4${}^{\rm th}$  generation of LT methods. The first\footnote{If we do not count the closely related works belonging to the \algname{CoCoA} framework~\citep{cocoa,COCOA+,COCOA+journal}.} method belonging to this generation, called \algname{Scaffold}, and due to \citet{Scaffold}, employs a \algname{SAGA}-like variance reduction technique \citep{SAGA} to tame the client drift caused by LT. As a result, \algname{Scaffold} has the same communication complexity as \algname{GD}. 
\citet{LSGDunified2020} subsequently proposed a unified framework for designing and analyzing 3${}^{\rm rd}$ and 4${}^{\rm th}$ generation in a single theorem, including new 4${}^{\rm th}$ generation  LT methods such \algname{S-Local-GD} and \algname{S-Local-SVRG}. Finally, \citet{FEDLIN} proposed the \algname{FedLin} method, which can be seen as a variant of one of the methods from \citet{LSGDunified2020} allowing for the clients to take different number of local steps (without this leading to any theoretical benefit).

{\bf  5${}^{\rm th}$ generation of LT methods (accelerated).}  In a recent breakthrough, \citet{ProxSkip} proved that a certain new and simple form of local training, embodied in their \algname{ProxSkip} method, leads to {\em provable communication acceleration} in the smooth strongly convex regime, even in the notoriously difficult heterogeneous data setting in which the client data $\{\cD_\iclient\}_{\iclient=1}^{\Mx}$ is allowed to be arbitrarily different. In particular, if each $f_\iclient$ is $L$-smooth and $\mu$-strongly convex, then \algname{ProxSkip} solves \myref{eq:main} in $\cO(\sqrt{\nicefrac{L}{\mu}} \log \nicefrac{1}{\varepsilon})$ communication rounds, which is a significant acceleration when compared with the $\cO(\nicefrac{L}{\mu} \log \nicefrac{1}{\varepsilon})$ complexity of \algname{GD}. 
According to \citet{scaman2019optimal}, this accelerated communication complexity is optimal.
\citet{ProxSkip} provided several extensions of  their method. In particular,  \algname{ProxSkip} was enhanced with a very flexible  DS mechanism which can capture virtually any form of (unbiased and non-variance-reduced) data sampling scheme\footnote{The \algname{ProxSkip} method of \citet{ProxSkip}  can incorporate all forms of DS strategies captured by the {\em arbitrary sampling} approach of \citet{gower2019sgd} which is enabled by their {\em expected smoothness} inequality.}. 
Motivated by this progress, several other methods belonging to the  5${}^{\rm th}$ generation of LT methods were recently proposed. 

First,  \citet{ProxSkip-VR} extended the  \algname{ProxSkip} method via the inclusion of virtually arbitrary {\em variance-reduced} \algname{SGD} methods~\citep{gorbunov2020unified} in lieu of simple \algname{SGD}, inlcuding \algname{SVRG}~\citep{SVRG,S2GD}, \algname{SAGA}~\citep{SAGA}, \algname{JacSketch}~\citep{JacSketch}, \algname{L-SVRG}~\citep{hofmann2015variance,L-SVRG} or \algname{DIANA}~\citep{DIANA,horvath2019stochastic}. 

Second, \citet{RandProx} observed that the Bernoulli-type randomness employed in the \algname{ProxSkip} method whose role is to avoid the computation of an expensive proximity operator  is a special case of a more general principle: the application of an unbiased compressor to the proximity operator, combined with a bespoke variance reduction mechanism to tame the variance introduced by the compressor. \citet{RandProx} further generalized the forward-backward setting used by \citet{ProxSkip} to more complex splitting schemes involving the sum of three operators (e.g., \algname{ADMM}~\citep{ADMM-Hestenes-1969,ADMM-Powel-1969} and \algname{PDDY}~\citep{DY,PDDY2020}), and besides analyzing the smooth strongly convex regime, provided results in the  convex regime as well. 

Finally, \citet{sadiev2022communication} pioneered an alternative approach, based on an LT-friendly modification of the celebrated \algname{Chambolle-Pock} method~\citep{CP2011}. In their \algname{APDA-Inexact} method, the accelerated communication complexity is preserved, but compared to \algname{ProxSkip}, the \# of gradient-type LT steps in each communication round is improved from $\cO(\kappa^{1/2})$ to $\cO(\kappa^{1/3})$ and $\cO(\kappa^{1/4})$, where $\kappa=\nicefrac{L}{\mu}$ is the condition number. They further improve on some results of \citet{ProxSkip} related to the decentralized regime where communication happens along the edges of a connected network.

\section{Contributions}

\begin{table*}[!t]
	\centering
	\footnotesize
	\caption{Comparison of all 5${}^{\rm th}$ generation local training (LT) methods. Our \algname{5GCS} method is the first that supports client sampling (CS). Moreover, similarly to \algname{APDA-Inexact}, our theory allows for the LT solver to be chosen virtually arbitrarily.}
	\label{tbl:main}
	\begin{threeparttable}
		\begin{tabular}{ccccc}
&&&&  \\			
			{\bf 5${}^{\rm th}$ generation LT Method}  &  \bf \shortstack{LT Solver}&    {\bf \shortstack{Data Sampling}}&  {\bf \shortstack{Client Sampling}} & \bf  \shortstack{Reference} \\			
			\hline
			\algname{ProxSkip} 
			& \algname{GD}, \algname{SGD}	
			& \cmark \tnote{\color{blue}(a)}					
			& \xmark
			& \citet{ProxSkip}\\ 
			\hline
			\algname{ProxSkip-VR} 
			& \algname{GD}, \algname{SGD}, \algname{VR-SGD}			
			& \cmark \tnote{\color{blue}(b)}						
			& \xmark
			& \citet{ProxSkip-VR}\\ 
			\hline
			\algname{APDA-Inexact} 
			& any				
			& \xmark
			& \xmark 
			& \citet{sadiev2022communication}\\ 
			\hline
			\algname{RandProx} 
			& \algname{GD} 						
			& \xmark
			& \xmark 
			& \citet{RandProx}\\ 
			\hline
			\algname{5GCS} 
			& any						
			& \cmark
			& \cmark 
			& this work \\ 
		\end{tabular}
		\begin{tablenotes}
			{\scriptsize
				\item [{\color{blue}(a)}]  Only supports non-variance reduced DS on clients.
				\item [{\color{blue}(b)}]  Supports non-variance reduced {\em and} variance-reduced DS on clients.
				}
		\end{tablenotes}
	\end{threeparttable}
\end{table*}

Now that the FL community finally managed to show that (appropriately designed) LT techniques, which as we have seen are key behind the success of modern federated optimization methods for solving \myref{eq:main}, 
lead to provable communication acceleration guarantees (in the smooth strongly convex regime), we adopt the stance that further algorithmic and theoretical progress in FL should be focused on advancing the 5${}^{\rm th}$ generation of LT methods. 

To the best of our knowledge, there are only a handful of papers providing methods and results that belong to this latest generation of LT methods~\citep{ProxSkip, ProxSkip-VR, sadiev2022communication, RandProx}. A close examination of these works reveals that much is yet to be discovered.

\subsection{The open problem we address in this work}

\begin{quote}\em \footnotesize The starting point of our work is the observation that  none of the 5${}^{\rm th}$ generation local training (LT) methods support client sampling (CS). In other words, it is not  known whether it is possible to design a method that would  enjoy communication acceleration via LT and at the same time also support CS.
\end{quote}

 The problem is harder than one may initially think. We have talked to several people about this, including the authors of the \algname{ProxSkip} method. It turns out that they have tried---``very hard'' in their own words---but their efforts did not bear any fruit. We have tried as well, and failed. The analysis of \algname{ProxSkip} is remarkably tight, and every adaptation towards supporting CS seems to either lead to technical problems during the proof construction, or to a loss of communication acceleration. In fact, it is not even clear how should a CS variant of \algname{ProxSkip} look like. Our attempts at guessing what such a method could look like failed as well, and the variants we brainstormed diverged in our numerical experiments as soon as CS was enabled. 
 
Fortunately, it turns out that these negative results were helpful to us after all. Indeed, they led us to the idea that we should try to develop an entirely different method; one that is not based on either \algname{ProxSkip} nor \algname{APDA-Inexact}. Once we started to think outside the box created by our pre-conceived solution path, we eventually managed to succeed.

\subsection{Summary of contributions}

\begin{table*}[!t]
	\centering
	\footnotesize
	\caption{Variants of \algname{5GCS} (Algorithm~\ref{alg:5GCS}) depending on the choice of the LT procedure  run by clients $\iclient \in \set$ in the current cohort. $\Mx = $ number of clients; $\Cx = $ cohort size. }
	\label{tbl:variants}
	\begin{threeparttable}
		\begin{tabular}{cccc}
			{\bf Algorithm}  &  \bf \shortstack{Local Training via Subroutine $\cA$} & \bf  \shortstack{Communication Complexity}  &   \bf  \shortstack{Theorem} \\			
			\hline
			\algname{5GCS${}_\infty$} \tnote{\color{blue}(a)} 
			& $K=\infty$ steps of \algname{GD}				& $\cO\left(\left(\frac{\Mx}{\Cx}+\sqrt{\frac{\Mx}{\Cx}\frac{L}{\mu }}\right)\log \frac{1}{\varepsilon}\right)$
			&  \ref{thm:5GCS-infty} \\				
			\hline
			\algname{5GCS${}_K$} 
			& $K=\cO(\sqrt{\frac{\Cx}{\Mx}\frac{L}{\mu}})$ steps of \algname{GD}				& $\cO\left(\left(\frac{\Mx}{\Cx}+\sqrt{\frac{\Mx}{\Cx}\frac{L}{\mu }}\right)\log \frac{1}{\varepsilon}\right)$
			&  \ref{thm:5GCS} \\
			\hline
			\algname{5GCS${}_0$} \tnote{\color{blue}(b)} 
			& $K=0$ steps of \algname{GD}						
			& $ \cO\left(\frac{\Mx}{\Cx}\frac{L}{\mu} \log \frac{1}{\varepsilon}\right)$ \tnote{\color{blue}(c)}	
			& \ref{thm:5GCS-0} 	
			\\			
			\hline
			\algname{5GCS${}_\cA$} 
			& any method $\cA$ (as long as it satisfies Assumption~\ref{ass:GTPS})			& $\cO\left(\left(\frac{\Mx}{\Cx}+\sqrt{\frac{\Mx}{\Cx}\frac{L}{\mu }}\right)\log \frac{1}{\varepsilon}\right)$
			&  \ref{thm:INEXACTPPanyM}					
		\end{tabular}
		\begin{tablenotes}
			{\scriptsize
				\item [{\color{blue}(a)}] This method can be found in the appendix as Algorithm~\ref{alg:5GCS-infty}.
				\item [{\color{blue}(b)}] This method can be found in the appendix as Algorithm~\ref{alg:5GCS-0}.				
				\item [{\color{blue}(c)}]  
			Does not have accelerated communication complexity. Indeed, the communication complexity is $\cO\left(\nicefrac{L}{\mu} \log \nicefrac{1}{\varepsilon}\right)$ instead of $\cO\left(\sqrt{\nicefrac{L}{\mu}} \log \nicefrac{1}{\varepsilon}\right)$ in the $\Cx=\Mx$ regime.
			
				}
		\end{tablenotes}
	\end{threeparttable}
\end{table*}

We are now ready to outline the key insights and contributions of our work.  Our main idea is to start our development with the remarkable \algname{Point-SAGA} method of \citet{defazio2016simple}. The key appealing property of this method is that it can solve \myref{eq:main} with an accelerated rate in the smooth strongly convex regime. However, \algname{Point-SAGA} has two critical drawbacks: 

\phantom{X} (i) In each communication round, \algname{Point-SAGA}  samples  a single client only, uniformly at random, which means it supports a very rudimentary and hence not practically interesting form of CS only.

\phantom{X} (ii) \algname{Point-SAGA} requires a prox-oracle for each $f_\iclient$, where $\iclient$ is the active client, i.e., 
\[\squeeze 
\mathrm{prox}_{\frac{1}{\tau} f_\iclient} (x) \eqdef \arg \min \limits_{u\in \R^d} \left \{  f_\iclient(u) + \frac{\tau}{2}\|x-u\|^2 \right\}\]
for some $x\in \R^d$ and $\tau>0$ in each communication round, and do it exactly. This is problematic, since exact evaluation of the proximity operator is rarely possible, and inexact evaluation (with a small error) may be overly expensive, imparting an excessive computational burden on the clients. 

Our main contributions  can be summarized as follows.

$\diamond$ We propose a new LT method for FL, which we call \algname{5GCS} (Algorithm~\ref{alg:5GCS}), which achieves accelerated communication complexity, and also supports client sampling. To the best of our knowledge, this is the first 5${}^{\rm th}$ generation LT method which works with client sampling (see Table~\ref{tbl:main}). Moreover, according to \citet{NIPS2016_645098b0}, the communication complexity of \algname{5GCS} is optimal.

$\diamond$ Our method supports arbitrary LT subroutines as long as they satisfy a certain technical assumption (Assumption~\ref{ass:GTPS}). See Table~\ref{tbl:variants} for a list of four variants of \algname{5GCS} depending on what LT subroutine is applied, and the associated communication complexities.

$\diamond$ When an infinity of \algname{GD} steps is used as the LT subroutine, our method  \algname{5GCS} in each communication round evaluates the prox of  $f_\iclient$ for all clients $\iclient$ in the cohort, and  reduces to  a minibatch version of \algname{PointSAGA}, which is  new\footnote{There is one exception: this method was recently analyzed by \citet{RandProx}.}. While this method enjoys accelerated communication complexity, its reliance on a prox oracle puts a heavy computation burden on the clients. On the other hand, when zero \algname{GD} steps are used as a subroutine, our method achieves linear but nonaccelerated communication complexity only. Fortunately, it is sufficient to apply a relatively small number of \algname{GD} steps as the LT subroutine while preserving the accelerated communication complexity of minibatch \algname{PointSAGA}.

$\diamond$ Several further contributions are mentioned in the remaining text.

\section{Main Results}

\begin{algorithm*}[!t]
	\caption{\algn{5GCS}}
	\footnotesize
	\begin{algorithmic}[1]\label{alg:5GCS}
		\STATE  \textbf{Input:} initial primal iterates $x^0\in\mathbb{R}^d$; initial dual iterates $u_1^0, \dots,u_{\Mx}^0 \in\mathbb{R}^d$; primal stepsize $\gammaM>0$; dual stepsize $\tauM>0$; cohort size $\Cx\in \{1,\dots,\Mx\}$
		\STATE  \textbf{Initialization:}  $v^0\eqdef \sum_{\iclient=1}^\Mx u_\iclient^0$  \hfill {\color{gray} \footnotesize $\diamond$ The server initiates $v^{0}$ as the sum of the initial dual iterates} 
		\FOR{communication round $t=0, 1, \ldots$} 
		\STATE Choose a cohort $\set\subset \{1,\ldots,\Mx\}$ of clients of cardinality $\Cx$, uniformly at random \hfill {\color{gray} \footnotesize $\diamond$  CS step} 
		\STATE Compute $\hat{x}^{t} = \frac{1}{1+\gammaM\mu} \left(x^t - \gammaM v^t\right)$ and broadcast it to the clients  in the cohort 
		\FOR{$\iclient\in \set$}
		\STATE Find $\lastlocittermk$ as the final point after $\Kx$ iterations of some local optimization algorithm $\mathcal{A}_\iclient$, initiated with $y_\iclient^0=\hat{x}^t$, for solving the optimization problem \hfill {\color{gray} \footnotesize $\diamond$ Client $\iclient$ performs $\Kx$ LT steps} 
		\begin{eqnarray}
			\squeeze 
			\lastlocittermk \approx \argmin \limits_{y\in\mathbb{R}^d}\left\{\localfuni(y) \eqdef  F_\iclient(y)+\frac{\tauM}{2} \sqnorm{y-\left(\hat{x}^\kstep+\frac{1}{\tauM}u_\iclient^t\right)}\right\}\label{localprob}
		\end{eqnarray}
		\STATE Compute $u_\iclient^{t+1}=\nabla F_\iclient(\lastlocittermk)$ and send it to the server \hfill {\color{gray} \footnotesize $\diamond$ Client $\iclient$ updates its dual iterate} 
		\ENDFOR
		\FOR{$\iclient\in\{1, \ldots,\Mx\}\backslash \set$}
		\STATE $u_{\iclient}^{t+1}\eqdef u_{\iclient}^t $ \hfill {\color{gray} \footnotesize $\diamond$ Non-participating clients do nothing} 
		\ENDFOR
		\STATE $v^{t+1} \eqdef \sum_{\iclient=1}^{\Mx} u_\iclient^{t+1} $ \hfill {\color{gray} \footnotesize $\diamond$ The server maintains $v^{t+1}$ as the sum of the dual iterates} 
		\STATE  $x^{t+1} \eqdef \hat{x}^{t}- \gammaM \frac{\Mx}{\Cx} (v^{t+1}-v^t)$ \hfill {\color{gray} \footnotesize $\diamond$ The server updates the primal iterate} 
		\ENDFOR
	\end{algorithmic}
\end{algorithm*}
In this section we describe our new method, \algname{5GCS} (Algorithm~\ref{alg:5GCS}) for solving \myref{eq:main}, and formulate our main convergence results (see Table~\ref{tbl:variants} for a summary).

\subsection{Convexity and smoothness}
In our analysis  we focus on the regime when each $f_\iclient$ is $L$-smooth and $\mu$-strongly convex, which are standard assumptions in the convex optimization literature\footnote{While many practical FL models  involve neural networks which lead to nonconvex problems instead, in our work we focus on resolving a certain key open problem in the foundations of FL for which there is no answer even in the regime we consider. }. 

\begin{assumption}\label{ass:main} The functions $f_\iclient$ are $L$-smooth and $\mu$-strongly convex for all $\iclient \in \{1,\dots,\Mx\}$.
\end{assumption}

We shall use this assumption in what follows without explicitly mentioning this.  Recall that a continuously differentiable function $\phi :\mathbb{R}^{d}\to\mathbb{R}$ is $L$-smooth if
$\phi(x)-\phi(y)-\langle\nabla \phi(y),x-y\rangle \leq \frac{L}{2}\|x-y\|^2$ for all $x,y\in \R^d$, 
and $\mu$-strongly convex if
$\phi(x)-\phi(y)-\langle\nabla \phi(y),x-y\rangle \geq \frac{\mu}{2}\|x-y\|^2$ for all $x,y\in \R^d$.

\subsection{Problem reformulation and its dual}\label{sec:H}

Our method applies to a certain reformulation of \myref{eq:main} which we shall now describe. Let $\Koper\colon\R^d\to\R^{\Mx  d}$ be the linear operator which maps $x\in \R^{d}$ into the vector $(x,\dots,x)\in \R^{\Mx d}$ consisting of $\Mx$ copies of $x$. First, notice that $F_\iclient(x)\eqdef \frac{1}{\Mx}(f_\iclient(x)-\frac{\mu}{2}\sqnorm{x})$ is convex and $L_F$-smooth with  $L_F \eqdef \frac{1}{\Mx}(L-\mu)$. Further,  define $F:\R^{\Mx d}\to \R$ via
$F(x_1, \dots,x_\Mx) \eqdef \sum_{\iclient=1}^{\Mx} F_\iclient(x_\iclient).$ 

Having established the necessary notation, we consider the following reformulation of problem~\myref{eq:main}: 
\begin{equation}
	\squeeze 
	x^\star = \argmin \limits_{x\in\R^d}  \left[f(x) \eqdef F(\Koper x) + \frac{\mu}{2}\sqnorm{x}\right].\label{eq:main-new}
\end{equation}
It is straightforward to see that $f$ from \myref{eq:main} and \myref{eq:main-new} are identical functions. The  dual problem to \myref{eq:main-new} is
\begin{align*}\label{dualnew}
	\squeeze 
	u^\star = \argmax\limits_{u\in\R^{\Mx d}} \, \left( \frac{1}{2\mu}\sqnorm{\sum \limits_{\iclient=1}^{\Mx} u_\iclient}+\sum \limits_{\iclient=1}^{\Mx} F_\iclient^*(u_\iclient)\right),\notag
\end{align*}
where $F_\iclient^*$ is the Fenchel conjugate of $F_\iclient$, defined by $F_\iclient^*(y)\eqdef\sup_{x\in\R^{d}}\{\langle x,y \rangle-F_\iclient(x) \}.$
Under Assumption~\ref{ass:main}, the primal and dual problems have unique optimal solutions $x^\star$ and $u^\star$, respectively.

\subsection{The 5GCS algorithm}

Our proposed algorithm, \algname{5GCS}, is formalized as Algorithm~\ref{alg:5GCS}. The method produces a sequence of primal iterates $x^t$, and a sequence of dual iterates
$u^t = (u_1^t, \dots, u_\Mx^t)$. We have added several comments explaining the steps, and believe that the method should be easy to parse without additional commentary.  
 In each communication round $t$, the participating clients $\iclient \in \set$ in parallel perform LT via $K$ steps of \algname{GD} applied to minimizing function $\localfuni$; see  \myref{localprob}. Below we outline four special variants of \algname{5GCS}, depending on the choice of the LT subroutines $\{\cA_\iclient\}_{\iclient=1}^{\Mx}$.

\subsection{LT subroutine: GD with $\Kx = +\infty$ steps (i.e., prox)}

The choice $\Kx = +\infty$ corresponds to exact minimization of function $\localfuni$ defined in \myref{localprob}, i.e., to the evaluation of the prox operator of $F_\iclient$ for all $\iclient \in \set$. In this case, \algname{5GCS} reduces to \algname{Minibatch-Point-SAGA} (see  Algorithm~\ref{alg:5GCS-infty}), and its convergence properties are  described by the next result.

\begin{theorem}\label{thm:5GCS-infty} Consider Algorithm~\ref{alg:5GCS} (\algname{5GCS}) with the LT solver being \algname{GD} run for $K=+\infty$ iterations (this is equivalent to Algorithm~\ref{alg:5GCS-infty}; we shall also call the method \algname{5GCS${}_\infty$}).
	Let $\gammaM>0$, $\tauM>0$ and $\gammaM \tauM \leq \frac{1}{\Mx}$. Then for the Lyapunov function
	\begin{equation*}
	\squeeze	
	\Psi^{\kstep}\eqdef  \frac{1}{\gammaM}\sqnorm{x^{\kstep}-x^\star}+\frac{\Mx}{\Cx}\left(\frac{1}{\tauM}+2\frac{1}{L_F}\right)\sqnorm{u^{\kstep}-u^\star},
	\end{equation*}
 the iterates of the method satisfy
	$
		\Exp{\Psi^{\Tx}}\leq (1-\rho)^\Tx \Psi^0,
$
	where 
$ 	\rho \eqdef  \min\left(\frac{\gammaM\mu}{1+\gammaM\mu}, \frac{\Cx}{\Mx}\frac{2\tauM }{L_F+2\tauM }\right)<1.
$

\end{theorem}

The following corollary gives a bound on the number of communication rounds needed to solve the problem.
\begin{corollary}\label{cor:5GCS-infty}
Choose any $0<\varepsilon<1$. If
	we choose $\gammaM=\sqrt{\frac{2\Cx}{L_F\mu \Mx^2}}$ and $\tauM=\sqrt{\frac{L_F\mu}{2\Cx}}$, then  in order to guarantee $\Exp{\Psi^{\Tx}}\leq \varepsilon \Psi^0$, it suffices to take
\[
	\squeeze	
	\Tx \geq	\left(\frac{\Mx}{\Cx}+\sqrt{\frac{\Mx}{ \Cx} \frac{L-\mu}{2\mu}}\right)\log \frac{1}{\varepsilon} =\tilde{\cO}\left(\frac{\Mx}{\Cx}+\sqrt{\frac{\Mx}{\Cx}\frac{L}{\mu}}\right)
\]
	communication rounds.
\end{corollary}

Note that the communication complexity improves as the cohort size $\Cx$ increases, and becomes $\tilde{\cO}(\sqrt{\nicefrac{L}{\mu }}) $ for $\Cx=\Mx$. This recovers the accelerated communication complexity of existing 5${}^{\rm th}$ generation local training (LT) methods \algname{ProxSkip}, \algname{ProxSkip-VR}  and \algname{APDA-Inexact} in the regime when \algname{GD} is used as the LT method. However, unlike these methods, \algname{5GCS${}_\infty$} supports client sampling (CS). In the opposite extreme, i.e., when the cohort size is minimal ($\Cx=1$), the communication complexity of \algname{5GCS${}_\infty$} becomes $ \tilde{\cO}(\Mx+\sqrt{\nicefrac{\Mx L}{\mu}})$. If $\nicefrac{L}{\mu}\leq \Mx$, which will typically be the case in FL settings with a very large number of clients (e.g., cross-device FL), the complexity simplifies to $ \tilde{\cO}(\Mx)$, which says that we need as many communication rounds as there are clients, which makes sense, since we do not assume any form of data homogeneity, and this means that all clients may contain valuable data. In general, as the cohort size $\Cx$ increases, the communication complexity improves, and interpolates between these two extreme cases.

\begin{figure}[!t]
	\centering
	\includegraphics[width=0.99\linewidth]{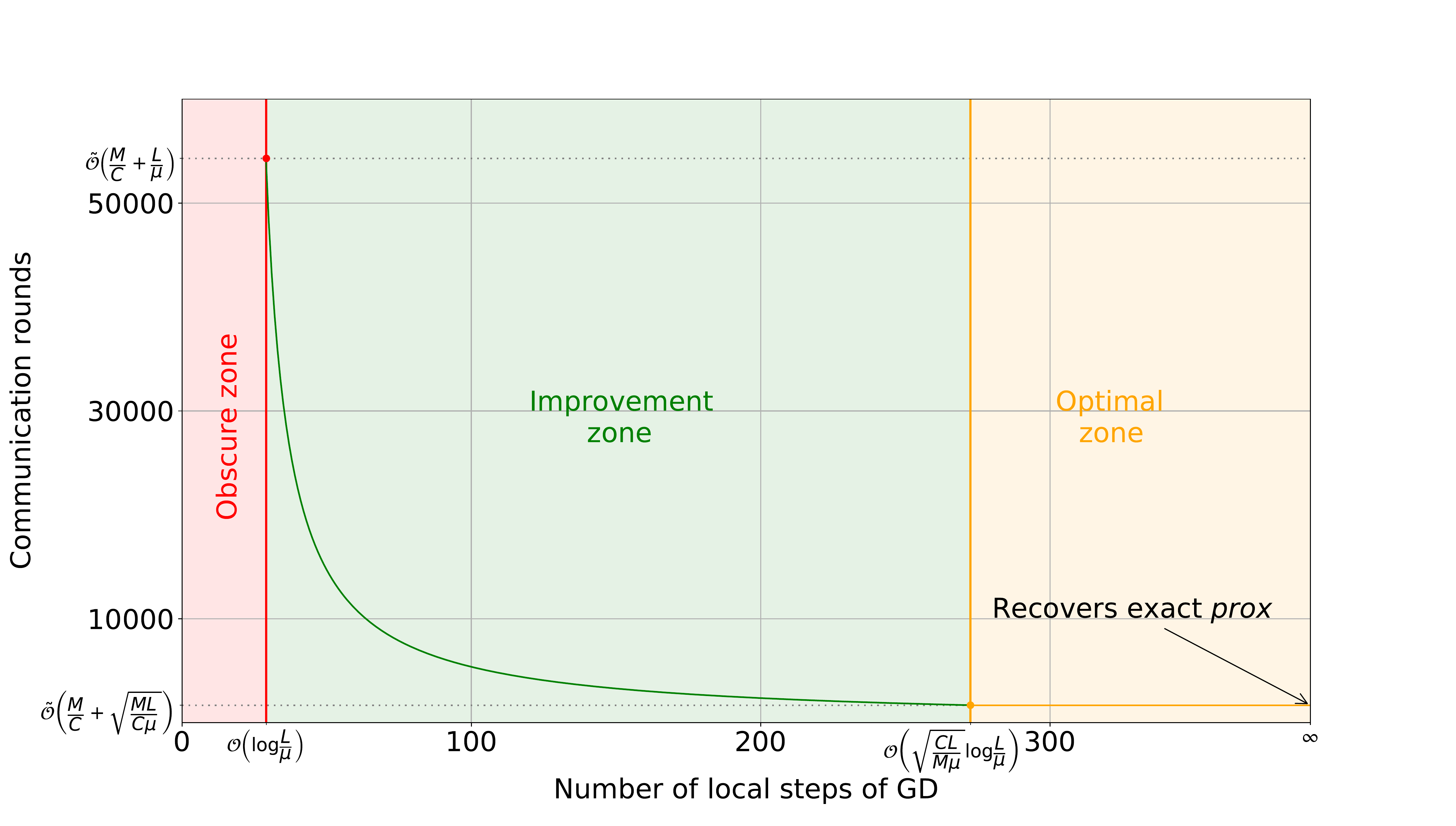}
	\caption{The number of communication rounds of \algname{5GCS} as a function of the number of \algname{GD} steps forming the LT subroutine $\cA$ with $\nicefrac{L}{\mu}=10^4$ and $\nicefrac{\Cx}{\Mx}=0.1$. The key observation is that it is enough to choose $\Kx = \cO(\sqrt{\frac{\Mx}{\Cx}\frac{L}{\mu}})$, which is at the left end-point of the ``optimal zone''. More steps do {\em not} lead to better communication complexity. }
	\label{ris:image}	
\end{figure}

\subsection{LT subroutine: GD with $\Kx = \cO(\sqrt{\frac{\Cx}{\Mx}\frac{L}{\mu}})$ steps}

The key drawback of \algname{5GCS${}_\infty$} is that the LT subroutine needs to take an infinite number of \algname{GD} steps, or equivalently, the method requires the exact evaluation of the prox of $F_\iclient$.  We now show that it is possible to obtain the same accelerated communication complexity as in the $\Kx=+\infty$ case with a finite, and in fact surprisingly small, number of \algname{GD} iterations.
 
\begin{theorem}\label{thm:5GCS} Consider Algorithm~\ref{alg:5GCS} (\algname{5GCS}) with the LT solver being \algname{GD} run for $\Kx \geq \left(\frac{3}{4}\sqrt{\frac{\Cx}{\Mx}\frac{L}{\mu}}+2\right)\log\left(4\frac{L}{\mu}\right)$ iterations.
Let $0<\gammaM\leq \frac{3}{16}\sqrt{\frac{\Cx}{L\mu \Mx}}$ and $\tauM=\frac{1}{2\gammaM \Mx}$. Then for the Lyapunov function
	\begin{equation*}
		\squeeze
		 	\Psi^{\kstep}\eqdef \frac{1}{\gammaM}\sqnorm{x^{\kstep}-x^\star}+\frac{\Mx}{\Cx}\left(\frac{1}{\tauM}+\frac{1}{L_F}\right)\sqnorm{u^{\kstep}-u^\star},
	\end{equation*}
 the iterates of  the method satisfy
$		\Exp{\Psi^{\Tx}}\leq (1-\rho)^\Tx \Psi^0,
$	where 
$
	\rho\eqdef  \min\left\{\frac{\gammaM\mu}{1+\gammaM\mu},\frac{\Cx}{\Mx}\frac{\tauM}{(L_F+\tauM)}\right\}<1.
$

\end{theorem}

Note that \algname{GD} needs to be run for $\Kx = \cO\left(\sqrt{\frac{\Cx}{\Mx}\frac{L}{\mu}}\right)$ local steps on each client in the cohort. This quantity depends on the square root of the condition number only, and is smaller for smaller cohort size $\Cx$. 

It turns out that this result can be improved using a finer analysis. In particular, we can show that some clients can get away with fewer LT steps than this, provided that their local datasets are favorable\footnote{To the best of our knowledge, a result of this type does not exist in the FL literature.}.  To see this, assume that each $f_\iclient$ is $L_\iclient$-smooth. Clearly, this implies that each $f_\iclient$ is $L$-smooth with $L=\max_{\iclient} L_\iclient$, and Theorem~\ref{thm:5GCS} holds with this $L$. However, recall that client $\iclient$ applies \algname{GD}  to (approximately) minimize $\localfuni$ from \myref{localprob}, and this function happens to be $\left(\frac{1}{\Mx}\left(L_\iclient-\mu\right)+\tauM\right)$-smooth and $\tauM$-strongly convex.
It can be easily seen that $\tauM \geq \frac{8}{3}\sqrt{\frac{\mu L}{\Mx\Cx}}$, and hence the condition number of $\localfuni$ is
$\frac{1}{\Mx}\left(L_\iclient-\mu\right)\frac{1}{\tauM}+1 \leq \frac{3}{8}\sqrt{\frac{\Cx}{\Mx } \frac{\nicefrac{L_\iclient^2}{L}}{\mu}} +1
$. So,  \algname{GD} only needs  $\Kx_\iclient = \cO\left( \sqrt{\frac{\Cx}{\Mx} \frac{\nicefrac{L_\iclient^2}{L} }{\mu}}\right)$ iterations on client $\iclient$, which can be much smaller than the worst-case bound $\Kx= \cO\left( \sqrt{\frac{\Cx}{\Mx} \frac{L}{\mu}}\right)$.

The following corollary gives a bound on the number of communication rounds needed to solve the problem.

\begin{corollary}\label{cor:5GCS}
	Choose any $0<\varepsilon<1$ and $\gammaM = \frac{3}{16}\sqrt{\frac{\Cx}{L\mu \Mx}}$. In order to guarantee $\Exp{\Psi^{\Tx}}\leq \varepsilon \Psi^0$, it suffices to take 
	\begin{eqnarray}
		\squeeze 
			T &\geq	& 
			\squeeze
			 \max\left\{1+\frac{16}{3}\sqrt{\frac{\Mx}{\Cx}\frac{L}{\mu}},\frac{\Mx}{\Cx}+\frac{3}{8}\sqrt{\frac{\Mx}{\Cx}\frac{L}{\mu} }\right\}\log\frac{1}{\varepsilon} \notag \\
		&=& \squeeze 
		\tilde{\cO}\left(\frac{\Mx}{\Cx}+\sqrt{\frac{\Mx}{\Cx}\frac{L}{\mu }}\right) \notag
	\end{eqnarray}
	communication rounds.
\end{corollary}

This is the same expression as that from Corollary~\ref{cor:5GCS-infty}, and hence the same comments we've made there apply here, too.

\subsection{LT subroutine: GD with $\Kx = 0$ steps}

\begin{theorem}\label{thm:5GCS-0}
Consider Algorithm~\ref{alg:5GCS} (\algname{5GCS}) with the LT solver being \algname{GD} run for $\Kx=0$ iterations (this is equivalent to Algorithm~\ref{alg:5GCS-0}; we shall also call the method \algname{5GCS${}_0$}).	Let $0<\gammaM\leq \frac{\Cx}{4L\Mx}$. Then for the Lyapunov function
	\begin{equation*}
		\squeeze 	\Psi^{\kstep}\eqdef  \frac{\Cx}{\Mx^2\gammaM^2}\left(1-\sqrt{\frac{\gammaM \Mx L_F}{2}}\right)\sqnorm{x^{\kstep}-x^\star}+\sqnorm{u^{\kstep}-u^\star},
	\end{equation*}
the iterates of the method satisfy
$		\Exp{\Psi^{\Tx}}\leq \left(1-\rho\right)^\Tx \Psi^0,
$	
where 
$
	\rho\eqdef  \min\left(\frac{\gammaM\mu}{1+\gammaM\mu},\frac{\Cx}{\Mx+2\gammaM  L_F\Mx^2}\right)<1.
$

\end{theorem}

The following corollary gives a bound on the number of communication rounds needed to solve the problem.

\begin{corollary}\label{cor:5GCS-0}
	Choose any $0<\varepsilon<1$ and
	 $\gammaM=\frac{\Cx}{4 L \Mx}$. In order to guarantee $\Exp{\Psi^{\Tx}}\leq \varepsilon \Psi^0$, it suffices to take
\[
		\squeeze 	
		\Tx \geq 	\max\left\{1+\frac{4\Mx}{\Cx}\frac{L}{\mu},\frac{\Mx}{\Cx}+\frac{L_F\Mx}{L}\right\}\log \frac{1}{\varepsilon}= \tilde{O}\left(\frac{\Mx}{\Cx}\frac{L}{\mu}\right)
\]
	communication rounds.
\end{corollary}

In this case, we do {\em not} obtain communication acceleration. This is because LT  with $\Kx=0$ is not extensive enough.

\subsection{LT subroutine: any method $\cA$}

Finally, we now show that \algname{5GCS} is not limited to exclusively using \algname{GD} as the LT solver. To the contrary, \algname{5GCS} works with any subroutine $\cA$ as long as it is possible to guarantee that, after a sufficiently large number $\Kx$ of iterations, a certain inequality  holds. 

\begin{assumption}\label{ass:GTPS}
	Let $\{\mathcal{A}_1,\dots,\mathcal{A}_\Mx\}$ be any LT subroutines for minimizing functions $\{\psi_1^t, \dots, \psi_\Mx^t\}$ defined in \myref{localprob}, capable of finding points $\{y_{1}^{\Kx,\kstep},\dots,y_{\Mx}^{\Kx,\kstep}\}$ in $\Kx$ steps, from the starting point $y_{\iclient}^{0,\kstep}=\hat{x}^\kstep$ for all $\iclient \in \{1,\dots,\Mx\}$, which satisfy the inequality
	\begin{align*}
		\squeeze 	\sum\limits_{\iclient=1}^{\Mx}\frac{4}{\tauM^2}\frac{\mu L_F^2}{3\Mx}\sqnorm{\lastlocittermk-\localsolmk}&	\squeeze +\sum\limits_{\iclient=1}^{\Mx}\frac{L_F}{\tauM^2}\sqnorm{\nabla\localfuni(\lastlocittermk)}\\
		&	\squeeze  \leq\sum\limits_{\iclient=1}^{\Mx}\frac{\mu}{6\Mx}\sqnorm{\hat{x}^\kstep-\localsolmk},
	\end{align*}
	where $\localsolmk$ is the unique minimizer of $\localfuni$, and $\tauM\geq \frac{8}{3}\sqrt{\frac{L\mu}{\Mx \Cx}}.$
\end{assumption}

Our most general result follows:

\begin{theorem}\label{thm:INEXACTPPanyM}
Consider Algorithm~\ref{alg:5GCS} (\algname{5GCS}) with the LT solvers $\{\cA_1,\dots,\cA_\Mx\}$ satisfying Assumption~\ref{ass:GTPS}. Let $0<\gammaM$ and $0<\tauM$ satisfy $\gammaM\leq\frac{1}{\tauM\Mx}\left(1-\frac{4\mu}{3\Mx\tauM}\right)$. 
Then for the Lyapunov function
	\begin{equation*}
		\squeeze 	
		\Psi^{\kstep}\eqdef \frac{1}{\gammaM}\sqnorm{x^{\kstep}-x^\star}+\frac{\Mx}{\Cx}\left(\frac{1}{\tauM}+\frac{1}{L_F}\right)\sqnorm{u^{\kstep}-u^\star}\label{PPINEXACTpsianyM},
	\end{equation*}
	the iterates of the method satisfy
$
		\Exp{\Psi^{\Tx}}\leq (1-\rho)^\Tx \Psi^0,
$
	where 
$	\rho\eqdef  \min\left\{\frac{\gammaM\mu}{1+\gammaM\mu},\frac{\Cx}{\Mx}\frac{\tauM}{(L_F+\tauM)}\right\}<1.
$

\end{theorem}

Note that the convergence rate in this result is identical to the convergence rate from Theorem~\ref{thm:5GCS}. Therefore, the same conclusions apply here as well.

\begin{figure*}[t]
	\centering
	\begin{tabular}{ccc}
		\includegraphics[width=0.31\linewidth]{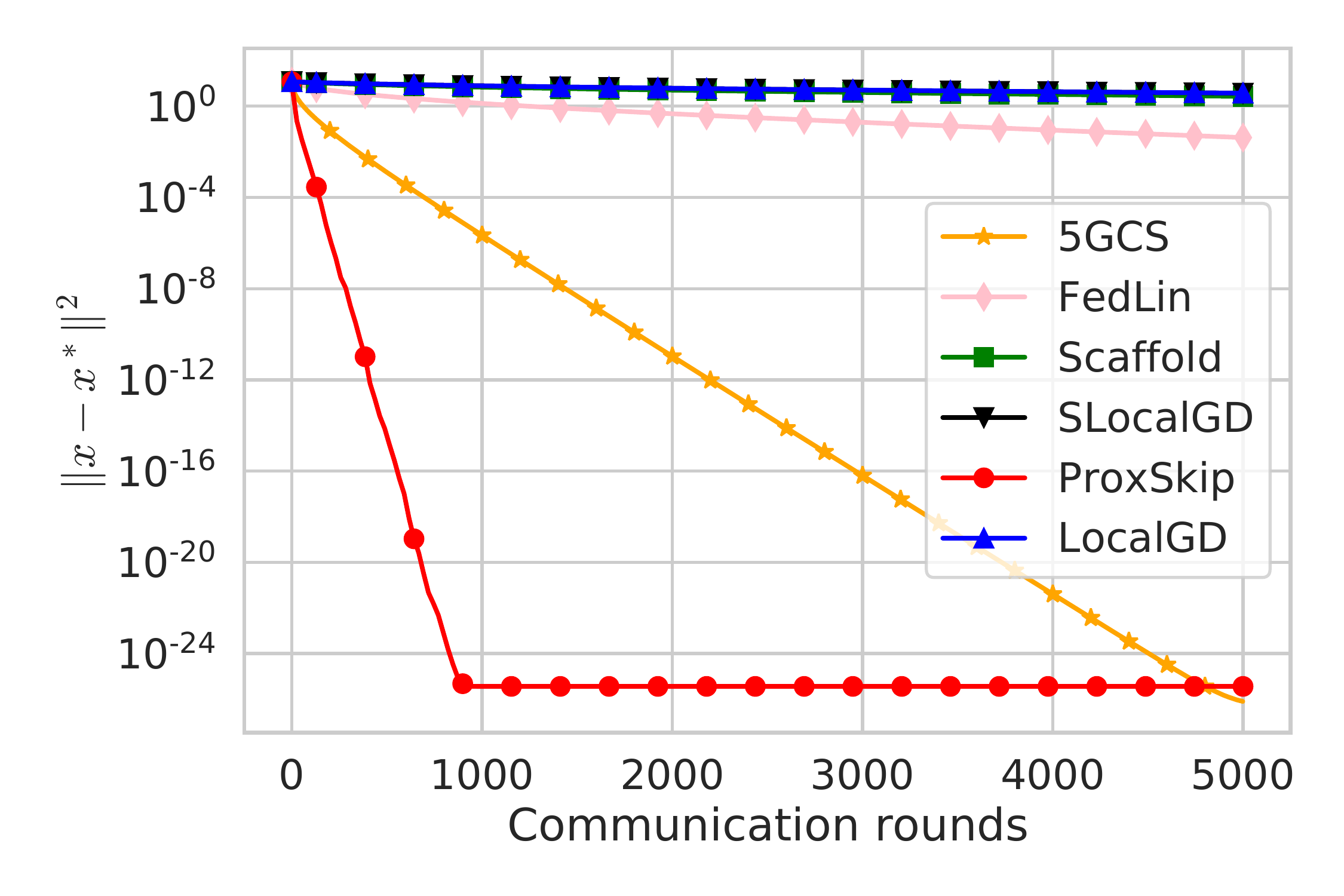}&	\includegraphics[width=0.31\linewidth]{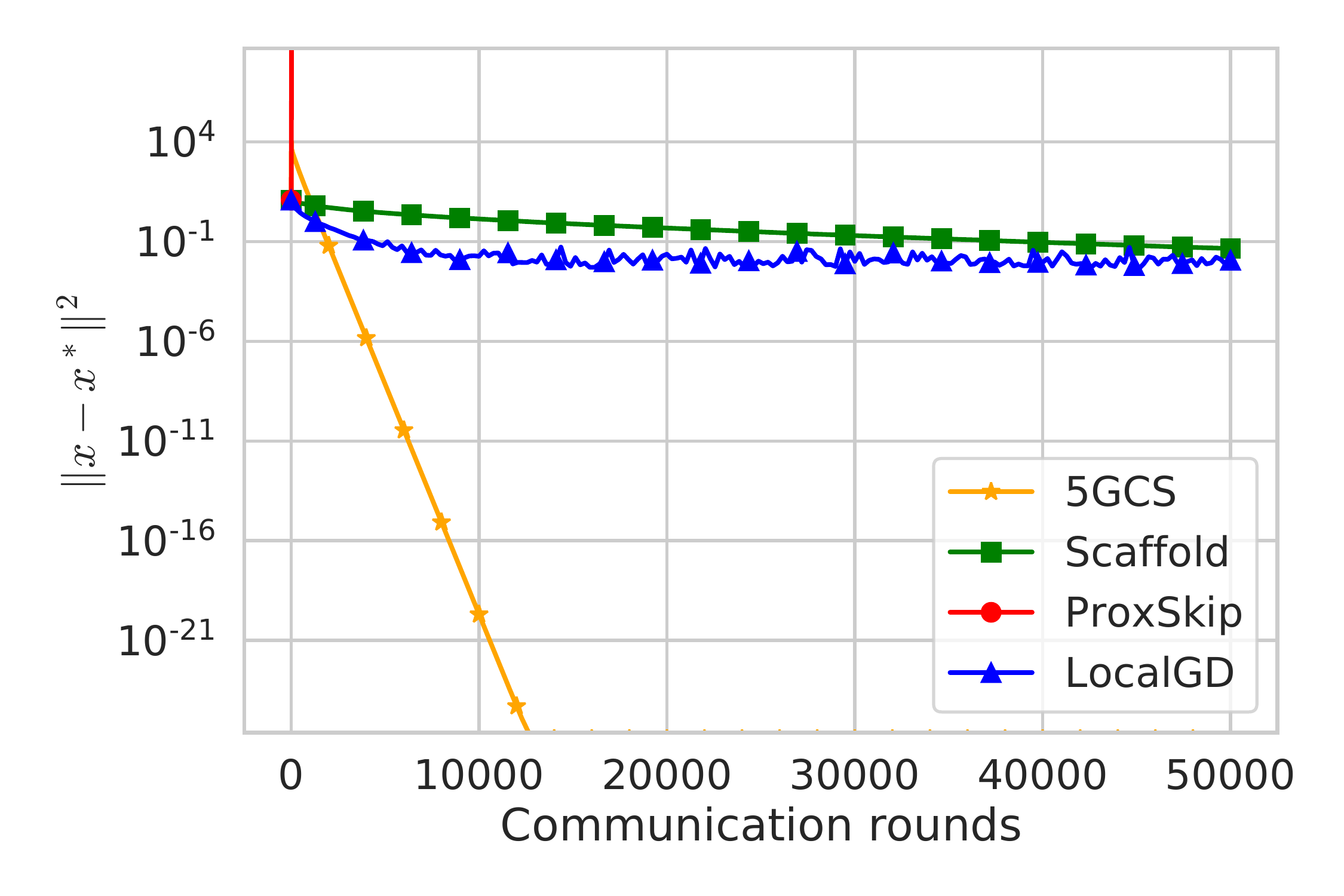}&
		\includegraphics[width=0.31\linewidth]{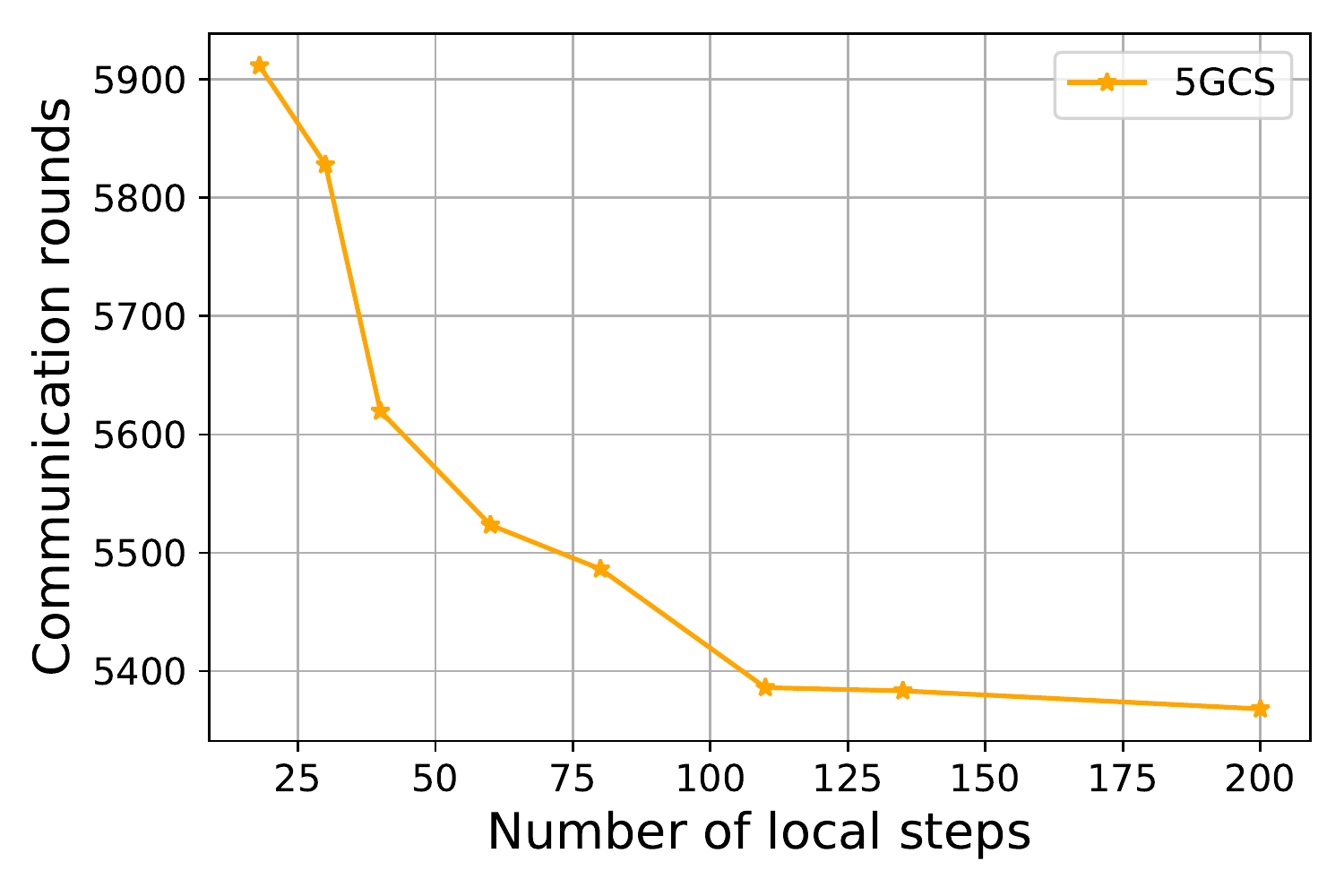}
	\end{tabular}
	\centering
	\caption{Performance of our \algname{5GCS} method without (left) and with (middle) CS. The plot on the right shows that \algname{5GCS} achieves optimal communication complexity with a (relatively) small number of local \algname{GD} steps, as predicted by  Theorem~\ref{thm:5GCS}.}
	\label{ris:image}
\end{figure*}

\subsection{Relation between the \# of communication rounds $\Tx$ on the \# of local steps $\Kx$}
We now study the dependence of the \# of communication rounds $\Tx$ on the \# of local steps $\Kx$ used by \algname{GD} as the LT subroutine. We first show in Theorem~\ref{thm:relation2} that with merely  $\Kx = \mathcal{O}\left(\log \frac{L}{\mu}\right)$ local \algname{GD} steps  we can improve the communication complexity from $\Tx = \tilde{O}\left(\frac{\Mx}{\Cx}\frac{L}{\mu}\right)$ (provided in Theorem~\ref{thm:5GCS-0}) to $\Tx=\tilde{O}\left(\frac{\Mx}{\Cx}+\frac{L}{\mu}\right)$. 
\begin{theorem}\label{thm:relation2}
	Consider Algorithm~\ref{alg:5GCS} (\algname{5GCS}) with the LT solver being \algname{GD}. Let $\gammaM=\frac{3}{16L}$  and $\tauM=\frac{8L}{3\Mx}.$
	With these stepsizes, if  LT is performed via  
	\begin{eqnarray*}
		\squeeze		\Kx\geq \left(2+\frac{3\Mx L_F}{4L}\right)\log\left(4\frac{L}{\mu}\right)=\mathcal{O}\left(\log\frac{L}{\mu}\right)
	\end{eqnarray*}
	steps of \algname{GD}, then  
	\begin{align*}
		\squeeze	T&\squeeze \geq\max\left\{1+\frac{16}{3}\frac{L}{\mu}, \frac{\Mx}{\Cx}+\frac{3\Mx}{8\Cx}\frac{\Mx L_F}{L}\right\}\log\frac{1}{\epsilon}\\
		&\squeeze =\tilde{\mathcal{O}}\left(\frac{\Mx}{\Cx}+\frac{L}{\mu}\right)
	\end{align*}
communication rounds suffice to find an $\varepsilon$-solution.	
\end{theorem}
In Theorem~\ref{thm:5GCS} we showed that an accelerated communication complexity can be achieved with merely $\Kx = \cO\left(\sqrt{\frac{\Cx}{\Mx}\frac{L}{\mu}}\log\frac{L}{\mu}\right)$ local \algname{GD} steps. However, the behavior of $\Tx$  on the interval between $\Kx = \mathcal{O}\left(\log \frac{L}{\mu}\right)$ (studied in Theorem~\ref{thm:relation2}) and $\Kx = \cO\left(\sqrt{\frac{\Cx}{\Mx}\frac{L}{\mu}}\log\frac{L}{\mu}\right)$ was not studied there.  We shall do so now.
\begin{theorem}\label{thm:relation}
	Consider Algorithm~\ref{alg:5GCS} (\algname{5GCS}) with the LT solver being \algname{GD}, which we run for \begin{equation}\label{eq:h9y98fd8hfd}\squeeze \Kx \geq \Kx(\deltaM)\eqdef 2\deltaM \log\left(\frac{4L}{\mu}\right)\end{equation} iterations, where $\deltaM$ is any constant satisfying $$\squeeze 1<\deltaM< 1+\frac{3}{8}\sqrt{\frac{\Cx}{\Mx} \frac{L}{\mu}}.$$  Let $\gammaM=\frac{1}{2\Mx\tauM}$  and $\tauM=\max\left\{\frac{L}{\Mx (\deltaM-1)},\frac{8}{3}\sqrt{\frac{L\mu}{\Mx \Cx}}\right\}.$
	Then for the Lyapunov function
	\begin{equation*}
		\squeeze 	\Psi^{\kstep}\eqdef \frac{1}{\gammaM}\sqnorm{x^{\kstep}-x^\star}+\frac{\Mx}{\Cx}\left(\frac{1}{\tauM}+\frac{1}{L_F}\right)\sqnorm{u^{\kstep}-u^\star}\label{PPINEXACTpsianyM},
	\end{equation*}
	the iterates of the method satisfy
	$
	\Exp{\Psi^{\Tx}}\leq (1-\rho)^\Tx \Psi^0,
	$
	where 
	$	\rho\eqdef  \min\left\{\frac{\gammaM\mu}{1+\gammaM\mu},\frac{\Cx}{\Mx}\frac{\tauM}{(L_F+\tauM)}\right\}<1.
	$
\end{theorem}

\begin{corollary}\label{cor:rela}
	Choose any $0<\varepsilon<1$. In order to guarantee $\Exp{\Psi^{\Tx}}\leq \varepsilon \Psi^0$, it suffices to take 
	\begin{equation*}
	\squeeze	T\geq	\max\left\{1+\frac{2L}{\left(\deltaM-1\right)\mu},\frac{\Mx}{\Cx}\deltaM\right\}\log\frac{1}{\epsilon}.
	\end{equation*}
	Note that if $\deltaM\leq \frac{\Mx+\Cx}{2\Mx}+\sqrt{\frac{2L\Cx}{\mu \Mx}+\left(\frac{\Mx-\Cx}{2\Mx}\right)^2}$, then
	\begin{equation*}
		\squeeze 	\Tx \geq	T(\alpha)\eqdef \left(1+\frac{2}{\deltaM-1} \frac{L}{\mu}\right)\log \frac{1}{\varepsilon}.
	\end{equation*}
\end{corollary}

Theorem~\ref{thm:relation}
and Corollary~\ref{cor:rela} imply that as long as $K\geq  K(\alpha)$ and $T\geq T(\alpha)$, then ${\mathbb{E}}[\Psi^T] \leq \varepsilon \Psi^0$. By substituting  $\alpha = \frac{K(\alpha)}{2\log \frac{4L}{\mu}}$ (see \myref{eq:h9y98fd8hfd}) to the expression for $T(\alpha)$, we get $$\squeeze T(\alpha) = \left(1+ \frac{4\log \frac{4L}{\mu}}{K(\alpha)-2\log \frac{4L}{\mu}} \frac{L}{\mu}\right) \log \frac{1}{\varepsilon} = \cO(\frac{1}{K(\alpha)}) \log \frac{1}{\varepsilon} .$$  This inverse dependence of $T(\alpha)$ on $K(\alpha)$ can  be observed empirically; see Figure~\ref{ris:image} (right). 


\section{Experiments}

We consider $\ell_2$-regularized logistic regression,
\begin{equation*}
	\squeeze f(x) =	\frac{1}{\Mx N} \sum \limits_{m=1}^{\Mx}\sum \limits_{i=1}^{N}\log  \left(1+ e^{\left(-b_{m,i} a_{m,i}^{\top} x\right)}\right)+\frac{\lambda}{2}\|x\|^2,
\end{equation*}
where $a_{m,i}\in \R^d$ and $b_{m,i}\in \{-1, +1\}$ are the data samples and labels, $\Mx$ is the number of clients and $N$ is the number of data points per client.  Following \citet{ProxSkip-VR}, we set  $\lambda = 10^{-3}L$, where $L$ is as in  Assumption~\ref{ass:main}. We chose to highlight a representative  experiment on the \texttt{a1a} dataset from the LibSVM library~\citep{chang2011libsvm}. All algorithms were implemented in Python utilizing the RAY package to simulate parallelization.

\subsection{Full participation ($\Cx = \Mx$)}

As a sanity check, we first perform an experiment in the full participation regime $\Cx = \Mx =5$, comparing our method \algname{5GCS}  with \algname{LocalGD} (3${}^{\rm rd}$ generation), \algname{Scaffold}, \algname{SLocalGD} and \algname{FedLin} (4${}^{\rm th}$ generation) and \algname{ProxSkip}(5${}^{\rm th}$ generation). We used theoretical stepsizes. For \algname{ProxSkip} we used the optimal communication probability parameter $p = \nicefrac{1}{\sqrt{\kappa}}$, where $\kappa=\nicefrac{L}{\mu}$. In the case of all 4${}^{\rm th}$ generation LT methods and \algname{LocalGD}, the theoretical rate does {\em not} depend on number of local steps $\Kx$. In our experiments we used the same number of local steps $\Kx = \nicefrac{1}{p} = \sqrt{\kappa}$ for all competing methods.  Figure~\ref{ris:image} (left) clearly shows that \algname{5GCS} has accelerated communication complexity, outperforming all 4${}^{\rm th}$ and 3${}^{\rm rd}$ generation LT methods by a large margin. However, due to a small numerical constant for the stepsize in our theory $\left(\nicefrac{3}{16}\right)$, \algname{5GCS} converges more slowly than \algname{ProxSkip}, which shows excellent performance.

\subsection{Client sampling ($\Cx < \Mx$)}
Our key contribution is to bring client sampling (CS) to the world of 5${}^{\rm th}$ generation LT methods. Once CS is required, \algname{ProxSkip} and \algname{APDA-Inexact} fall out of the competition as they do not support CS. We therefore compare our method \algname{5GCS} with 4${}^{\rm th}$ and 3${}^{\rm rd}$ generation LT methods supporting CS: we have chosen \algname{Scaffold} and \algname{LocalGD}. We set $\Mx=15$ and $\Cx=3$ and used theoretical parameters. Figure~\ref{ris:image} (middle) shows that \algname{ProxSkip} diverges in the CS regime, as expected. Moreover,  \algname{5GCS}  significantly outperforms the competing methods.  

\subsection{Dependence of $\Tx$ on $\Kx$}
In this experiment we set $\lambda = 5\cdot 10^{-4} L$ and run our  method \algname{5GCS} (with \algname{GD} as the LT solver) with  $\Kx$ ranging from $2\log \left(4\frac{L}{\mu}\right)$ to $\left(\frac{3}{4} \sqrt{\frac{\Cx}{\Mx} \frac{L}{\mu}}\right) \log \left(4 \frac{L}{\mu}\right)$; plus an extra choice of $\Kx=200$ steps. Figure~\ref{ris:image} (right) shows that after  $\Kx\approx 110$ local \algname{GD} steps, the comm.\ complexity virtually stops improving. This corresponds with our theory which shows that to obtain the optimal complexity, it suffices to take $\Kx = \mathcal{O}\left(\left(\frac{3}{4} \sqrt{\frac{\Cx}{\Mx} \frac{L}{\mu}}\right) \log \left(4 \frac{L}{\mu}\right)\right)$ local \algname{GD} steps (see Theorem~\ref{thm:5GCS}).

\bibliography{5GCS}
\bibliographystyle{plainnat}

\clearpage
\appendix
\onecolumn
\part*{Appendix}


\section{Basic Inequalities}

\subsection{Young's inequalities} For all $x,y\in \mathbb{R}^d$ and all $a>0$, we have
	\begin{eqnarray}
		&&\langle x, y\rangle\leq \frac{a\sqnorm{x}}{2}+\frac{\sqnorm{y}}{2a},\label{yi1}\\
		&&\sqnorm{x+y}\leq2\sqnorm{x}+2\sqnorm{y}\label{yi2},\\
		&&\frac{1}{2}\sqnorm{x}-\sqnorm{y}\leq \sqnorm{x+y}.\label{yi3}
	\end{eqnarray}

\subsection{Variance decomposition} For a random vector $\mathrm{X}\in\mathbb{R}^d$ (with finite second moment) and any $c\in\mathbb{R}^d$, the variance of $X$ can be decomposed as
	\begin{eqnarray}
		\Exp{\sqnorm{\mathrm{X}-\Exp{\mathrm{X}}}}=\Exp{\sqnorm{\mathrm{X}-c}}-\sqnorm{\Exp{\mathrm{X}}-c}.\label{vardec}
	\end{eqnarray}

\subsection{Compressor variance}	An unbiased randomized mapping  $\cC: \R^d\to \R^d$ has conic variance if there exists $\omega\geq 0$ such that
	\begin{equation}
		\Exp{\sqnorm{\mathcal{C}(x)-x}}\leq \omega \sqnorm{x}\label{cvar}
	\end{equation}
	for all $x\in \R^d.$

\subsection{Convexity and $L$-smoothness}

	Suppose $\phi\colon\mathbb{R}^d\to\mathbb{R}$ is $L$-smooth and convex. Then
	\begin{equation}
		\frac{1}{L}\sqnorm{\nabla \phi(x)-\nabla \phi(y)}\leq \langle \nabla \phi(x)-\nabla \phi(y),x-y \rangle\label{strmono}
	\end{equation}
	 for all $x,y\in\mathbb{R}^d$.

\subsection{Client Sampling Operator}

\begin{definition}[Client Sampling Operator]
	The client sampling operator is the randomized mapping $\Rop:\R^{\Mx d}\to \R^{\Mx d}$ defined as follows. We choose a random subset $S\subseteq \left\{1,\dots,\Mx\right\}$ of size $\Cx \in \{1,\dots,\Mx\}$ uniformly at random, and for  $v=\left(v_1,\dots,v_\Mx\right) \in \R^{\Mx  d}$, where $v_\iclient\in\mathbb{R}^d$ for all $\iclient$, we define $$\Rop(v)\eqdef \left(\Rop_1(v_1),\dots,\Rop_\Mx(v_\Mx)\right),$$ where
	\begin{eqnarray*}
		\Rop_\iclient(v_\iclient)\eqdef \begin{cases}
			\frac{\Mx}{\Cx}v_\iclient \in \R^d & \quad\text{for}\quad \iclient\in S,\\
			 0  \in \R^d &\quad\text{otherwise.}
		\end{cases}
	\end{eqnarray*}
\end{definition}
The client sampling operatoradmits the following identity:
	\begin{equation}
		\Exp{\sqnorm{\Koper^\top\left(\Rop(v)-v\right)}}=\frac{\Mx}{\Cx}\frac{\Mx-\Cx}{\Mx-1}\sum_{\iclient=1}^{\Mx}\sqnorm{v_\iclient}-\frac{\Mx-\Cx}{\Cx\left(\Mx-1\right)}\sqnorm{\sum_{\iclient=1}^{\Mx}v_\iclient},\label{PP}
	\end{equation} 
	where $\Koper$ was defined in Section~\ref{sec:H}, and $v=\left(v_1,\dots,v_\Mx\right)\in \R^{\Mx  d}$ and $v_{\iclient}\in\mathbb{R}^d$.
\begin{proof}
	Let $\mathbb{E}_S$ denote expectation with respect to the random set $S$. We can write 
	\begin{eqnarray*}
		\Exp{\sqnorm{\Koper^\top\left(\Rop(v)-v\right)}}&=&\Exp{\sqnorm{\sum_{\iclient=1}^{\Mx}\left(\Rop_\iclient(v_\iclient)-v_\iclient\right)}}=\Exps{\sqnorm{\sum_{\iclient\in S}\frac{\Mx}{\Cx}v_\iclient-\sum_{\iclient=1}^{\Mx}v_\iclient}}\\
		&=&\frac{\Mx^2}{\Cx^2}\Exps{\sqnorm{\sum_{\iclient\in S}v_\iclient}}+\sqnorm{\sum_{\iclient=1}^{\Mx}v_\iclient}-\frac{2\Mx}{\Cx}\Exps{\left\langle \sum_{\iclient\in S}v_\iclient, \sum_{\iclient=1}^{\Mx}v_\iclient\right\rangle}\\
		&=&\frac{\Mx^2}{\Cx^2}\Exps{\sum_{\iclient\in S}\sqnorm{v_\iclient}}+\frac{\Mx^2}{\Cx^2}\Exps{\sum_{\iclient\in S}\sum_{\iclient'\in S,\neq\iclient} \left\langle v_\iclient,v_{\iclient'}\right\rangle}-\sqnorm{\sum_{\iclient=1}^{\Mx}v_\iclient}.
	\end{eqnarray*}
	By computing the expectation on the right hand side, we get
	\begin{eqnarray*}
		\Exp{\sqnorm{\sum_{\iclient=1}^{\Mx}\left(\Rop_\iclient(v_\iclient)-v_\iclient\right)}}
		&=&\frac{\Mx}{\Cx}\sqnorm{\sum_{\iclient=1}^{\Mx}v_\iclient}+\frac{\Mx}{\Cx}\frac{\Cx-1}{\Mx-1}\sum_{\iclient=1}^{\Mx}\sum_{\iclient'=1,\neq\iclient}^{\Mx} \left\langle v_\iclient,v_{\iclient'}\right\rangle-\sqnorm{\sum_{\iclient=1}^{\Mx}v_\iclient}\\
		&=&\frac{\Mx}{\Cx}\left(1-\frac{\Cx-1}{\Mx-1}\right)\sqnorm{\sum_{\iclient=1}^{\Mx}v_\iclient}+\left(\frac{\Mx\left(\Cx-1\right)}{\Cx\left(\Mx-1\right)}-1\right)\sqnorm{\sum_{\iclient=1}^{\Mx}v_\iclient}\\
		&=&\frac{\Mx}{\Cx}\left(\frac{\Mx-\Cx}{\Mx-1}\right)\sqnorm{\sum_{\iclient=1}^{\Mx}v_\iclient}-\frac{\Mx-\Cx}{\Cx\left(\Mx-1\right)}\sqnorm{\sum_{\iclient=1}^{\Mx}v_\iclient}.
	\end{eqnarray*}
	
\end{proof}

\subsection{Dual Problem and Saddle-Point Reformulation}

Then the saddle function reformulation of \myref{eq:main-new} is:
\begin{equation}
		\mathrm{Find} \ (x^\star,(u_\iclient^\star)_{\iclient=1}^\Mx) \in \arg\min_{x\in\R^d}\max_{u\in\R^{\Mx d}} \, \left( \frac{\mu}{2}\sqnorm{x}+\sum_{\iclient=1}^\Mx \left\langle  x,u_\iclient \right\rangle -\sum_{\iclient=1}^\Mx F_\iclient^*(u_\iclient)\right).
\label{saddlenew}
\end{equation}
To ensure well-posedness of these problems, we need to assume that there exists $x^\star\in\mathbb{R}^d$ s.t.:
\begin{align}
	0=\mu x^\star+\sum_{\iclient=1}^{\Mx} \nabla F_\iclient(x^\star).
\end{align}
Which is equivalent to \myref{eq:main}, having a solution, which it does (unique in fact) as each $f_\iclient$ is $\mu$-strongly convex.
By first order optimality condition $x^\star$ and $u^\star$ that are solution to \myref{saddlenew}, satisfy:
\begin{equation}
	\left\{ \begin{array}{l}
		0= \mu x^\star + \sum_{\iclient=1}^\Mx u_\iclient^{\star}\\
		\Koper x^\star\in\partial F^* (u^{\star})
	\end{array}\right..\label{fooc}
\end{equation}
Where the latter in \myref{fooc} is equivalent to:
\begin{equation}
	\nabla F(\Koper x^\star)=u^\star.
\end{equation}
Throughout, this section we will denote by $\mathcal{F}_t$ for all $ t\geq 0$ the $\sigma$-algebra generated by the collection of $\left(\R^d\times\R^{d \Mx}\right)$-valued random variables $\left(x^0,u^0\right),\dots,\left(x^t,u^t\right).$

\clearpage
\section{Analysis of 5GCS${}_\infty$}

\begin{algorithm*}
	\caption{\algn{5GCS} with $\infty$ local \algname{GD} steps (a.k.a.\ \algn{Minibatch Point-SAGA})}
	\begin{algorithmic}[1]\label{alg:5GCS-infty}
		\STATE  \textbf{input:} initial points $x^0\in\mathbb{R}^d$, $u_\iclient^0\in\mathbb{R}^d$ for all $\iclient=\{1,\dots,\Mx\}$; 
		\STATE stepsize $\gammaM>0$, $\tauM>0$; $\Cx\in \{1,\dots,\Mx\}$
		\STATE $v^0\eqdef \sum_{\iclient=1}^\Mx u_\iclient^0$
		\FOR{$\kstep=0, 1, \ldots$}
		\STATE $\hat{x}^{\kstep} \eqdef \frac{1}{1+\gammaM\mu} \left(x^\kstep - \gammaM v^\kstep\right)  $
		\STATE Pick $\set \subset \{1,\ldots,\Mx\}$ of size $\Cx$ uniformly at random
		\FOR{$\iclient\in\set$}
		\STATE $u_{\iclient}^{\kstep+1}\eqdef  u_{\iclient}^\kstep + \tauM  \hat{x}^{\kstep} - \tauM \mathrm{prox}_{\frac{1}{\tauM} F_\iclient} \left(\hat{x}^{\kstep}+ \frac{1}{\tauM} u_{\iclient}^\kstep  \right)$
		\ENDFOR
		\FOR{$\iclient\in\{1, \ldots,\Mx\}\backslash \set$}
		\STATE $u_{\iclient}^{\kstep+1}\eqdef u_{\iclient}^\kstep $
		\ENDFOR
		\STATE $v^{\kstep+1}\eqdef \sum_{\iclient=1}^\Mx u_\iclient^{\kstep+1}$
		\STATE  $x^{\kstep+1} \eqdef \hat{x}^{\kstep}- \gammaM \frac{\Mx}{\Cx} (v^{\kstep+1}-v^\kstep)$
		\ENDFOR
	\end{algorithmic}
\end{algorithm*}
\begin{theorem} Consider Algorithm~\ref{alg:5GCS} (\algname{5GCS}) with the LT solver being \algname{GD} run for $K=+\infty$ iterations (this is equivalent to Algorithm~\ref{alg:5GCS-infty}; we shall also call the method \algname{5GCS${}_\infty$}).
	Let $\gammaM>0$, $\tauM>0$ and $\gammaM \tauM \leq \frac{1}{\Mx}$. Then for the Lyapunov function
	\begin{equation*}
		\Psi^{\kstep}\eqdef  \frac{1}{\gammaM}\sqnorm{x^{\kstep}-x^\star}+\frac{\Mx}{\Cx}\left(\frac{1}{\tauM}+2\frac{1}{L_F}\right)\sqnorm{u^{\kstep}-u^\star},
	\end{equation*}
	the iterates of the method satisfy
	$
	\Exp{\Psi^{\Tx}}\leq (1-\rho)^\Tx \Psi^0,
	$
	where 
	$ 	\rho \eqdef  \min\left(\frac{\gammaM\mu}{1+\gammaM\mu}, \frac{\Cx}{\Mx}\frac{2\tauM }{L_F+2\tauM }\right)<1.
	$

\end{theorem}
\begin{proof}
	Noting that updates for $u^{t+1}$ and $x^{t+1}$ can be written as
	\begin{eqnarray}
		&u^{t+1}\eqdef u^t + \tfrac{1}{1+\omega} \Rop^t\left(\hat{u}^{t+1}-u^t\right),&\\
		&x^{t+1}=\hat{x}^t-\gammaM\frac{\Mx}{\Cx}\Koper^\top\left(u^{t+1}-u^t \right)&\label{xupdate}
	\end{eqnarray}
	where $\Rop^t$ is the client sampling operator, $\omega=\frac{\Mx}{\Cx}-1$ and  $\hat{u}^{t+1} = \mathrm{prox}_{\tauM F^*}\left(u^{t}+\tauM \Koper \hat{x}^{t}\right)$. 
	We can use variance decomposition and Proposition 1 from \cite{condat2021murana} to write
	\begin{eqnarray}
		\Exp{\sqnorm{x^{t+1}-x^\star}\;|\;\mathcal{F}_t}&\overset{(\ref{vardec})}{=}&\sqnorm{\Exp{x^{t+1}\;|\;\mathcal{F}_t}-x^\star}+\Exp{\sqnorm{x^{t+1}-\Exp{x^{t+1}\;|\;\mathcal{F}_t}}\;|\;\mathcal{F}_t}\notag\\
		&\overset{(\ref{xupdate})}{=}& \sqnorm{\Exp{\hat{x}^{t}-\gammaM\frac{\Mx}{\Cx} (v^{t+1}-v^t)\;|\;\mathcal{F}_t}-x^\star}+\Exp{\sqnorm{x^{t+1}-\Exp{x^{t+1}\;|\;\mathcal{F}_t}}\;|\;\mathcal{F}_t}\notag\\
		&=& \sqnorm{\hat{x}^{t}-x^\star-\gammaM\frac{\Mx}{\Cx} \Exp{\Koper^\top\left(u^{t+1}-u^t \right)\;|\;\mathcal{F}_t}}+\Exp{\sqnorm{x^{t+1}-\Exp{x^{t+1}\;|\;\mathcal{F}_t}}\;|\;\mathcal{F}_t}\notag\\
		&=& \sqnorm{\hat{x}^{t}-x^\star-\gammaM\Koper^\top\left(\hat{u}^{t+1}-u^t \right)}+\Exp{\sqnorm{x^{t+1}-\Exp{x^{t+1}\;|\;\mathcal{F}_t}}\;|\;\mathcal{F}_t}\notag\\
		&\overset{(\ref{PP})}{=}& \underbrace{ \sqnorm{\hat{x}^{t}-x^\star-\gammaM \Koper^\top\left(\hat{u}^{t+1}-u^t \right)}}_X+\gammaM^2\oma\sqnorm{\hat{u}^{t+1}-u^t}\notag\\
		&\quad &- \gammaM^2\zeta\sqnorm{\Koper^\top\left(\hat{u}^{t+1}-u^t \right)}.\label{xbound31}
	\end{eqnarray}
	where
	\begin{equation*}
		\oma=\frac{\Mx(\Mx-\Cx)}{\Cx(\Mx-1)},\quad \zeta=\frac{\Mx-\Cx}{\Cx(\Mx-1)}.
	\end{equation*}
	Moreover, using \myref{fooc} and the definition of $\hat{x}^t$, we have
	\begin{align}
		&(1+\gammaM\mu)\hat{x}^t=x^t-\gammaM \Koper^\top u^{t},\label{opt13}\\
		&(1+\gammaM\mu)x^\star= x^\star -\gammaM\Koper^\top u^\star.\label{opt23}
	\end{align}
	Using \myref{opt13} and \myref{opt23} we obtain
	\begin{eqnarray}
		X&=&\sqnorm{\hat{x}^{t}-x^\star} +\gammaM^2\sqnorm{\Koper^\top\left(\hat{u}^{t+1}-u^t \right)}-2\gammaM \left\langle 
		\hat{x}^{t}-x^\star,\Koper^\top\left(\hat{u}^{t+1}-u^t \right) \right\rangle  \notag\\
		&\leq &(1+\gammaM\mu) \sqnorm{\hat{x}^{t}-x^\star} +\gammaM^2\sqnorm{\Koper^\top\left(\hat{u}^{t+1}-u^t \right)}\notag\\
		&\quad&-2\gammaM  \left\langle 
		\hat{x}^{t}-x^\star,\Koper^\top\left(\hat{u}^{t+1}-u^\star \right) \right\rangle  +2\gammaM  \left\langle 
		\hat{x}^{t}-x^\star,\Koper^\top\left(u^{t}-u^\star \right) \right\rangle  \notag\\
		&\overset{(\ref{opt13})+(\ref{opt23})}{=} &  \left\langle  x^t-x^\star-\gammaM \Koper^\top\left(u^{t}-u^\star \right),\hat{x}^{t}-x^\star \right\rangle  \notag+\gammaM^2\sqnorm{\Koper^\top\left(\hat{u}^{t+1}-u^t \right)}\notag\\
		&\quad&-2\gammaM  \left\langle 
		\hat{x}^{t}-x^\star,\Koper^\top\left(\hat{u}^{t+1}-u^\star \right) \right\rangle  +  \left\langle 
		\hat{x}^{t}-x^\star,2\gammaM\Koper^\top\left(u^{t}-u^\star \right)\right\rangle  \notag\\
		&= &\left\langle  x^t-x^\star+\gammaM \Koper^\top\left(u^{t}-u^\star \right),\hat{x}^{t}-x^\star \right\rangle  +\gammaM^2\sqnorm{\Koper^\top\left(\hat{u}^{t+1}-u^t \right)}\notag\\
		&\quad&-2\gammaM  \left\langle 
		\hat{x}^{t}-x^\star,\Koper^\top\left(\hat{u}^{t+1}-u^\star \right) \right\rangle  \notag\\
		&\overset{(\ref{opt13})+(\ref{opt23})}{=}&\frac{1}{1+\gammaM\mu} \left\langle  x^t-x^\star+\gammaM \Koper^\top\left(u^{t}-u^\star \right),x^t-x^\star-\gammaM \Koper^\top\left(u^{t}-u^\star \right) \right\rangle  \notag\\
		&\quad&+\gammaM^2\sqnorm{\Koper^\top\left(\hat{u}^{t+1}-u^t \right)}-2\gammaM  \left\langle 
		\hat{x}^{t}-x^\star,\Koper^\top\left(\hat{u}^{t+1}-u^\star \right) \right\rangle  \notag\\
		&= &\frac{1}{1+\gammaM\mu}\sqnorm{x^t-x^\star}-\frac{\gammaM^2}{1+\gammaM\mu}\sqnorm{\Koper^\top\left(u^{t}-u^\star \right)}\notag\\
		&\quad&+\gammaM^2\sqnorm{\Koper^\top\left(\hat{u}^{t+1}-u^t \right)}-2\gammaM  \left\langle 
		\hat{x}^{t}-x^\star,\Koper^\top\left(\hat{u}^{t+1}-u^\star \right) \right\rangle  . \label{long31} 
	\end{eqnarray}
	Combining \myref{xbound31} and \myref{long31}
	\begin{eqnarray}
		\Exp{\sqnorm{x^{t+1}-x^\star}\;|\;\mathcal{F}_t}&\leq&  \frac{1}{1+\gammaM\mu}\sqnorm{x^t-x^\star}-\frac{\gammaM^2}{1+\gammaM\mu}\sqnorm{\Koper^\top\left(u^{t}-u^\star \right)}\notag\\
		&\quad&+\gammaM^2(1-\zeta)\sqnorm{\Koper^\top\left(\hat{u}^{t+1}-u^t \right)}-2\gammaM  \left\langle 
		\hat{x}^{t}-x^\star,\Koper^\top\left(\hat{u}^{t+1}-u^\star \right) \right\rangle \notag\\
		&\quad&+\gammaM^2\oma\sqnorm{\hat{u}^{t+1}-u^t}.\label{xbnd31}
	\end{eqnarray}
	On the other hand using the variance decomposition and conic variance of $\Rop^t$ 
	\begin{eqnarray}
		\Exp{\sqnorm{u^{t+1}-u^\star}\;|\;\mathcal{F}_t}&\overset{(\ref{vardec})+(\ref{cvar})}{\leq}& \sqnorm{u^{t}-u^\star+\frac{1}{1+\omega}\left(\hat{u}^{t+1} -u^t\right)}
		+\frac{\omega}{(1+\omega)^2}\sqnorm{\hat{u}^{t+1} -u^t }\notag\\
		&= &\sqnorm{\frac{\omega}{1+\omega}(u^{t}-u^\star)+\frac{1}{1+\omega}\left(\hat{u}^{t+1} -u^\star\right)}
		+\frac{\omega}{(1+\omega)^2}\sqnorm{\hat{u}^{t+1} - u^\star - (u^t - u^\star) }\notag\\
		&=&\frac{\omega^2}{(1+\omega)^2}\sqnorm{u^{t}-u^\star}+\frac{1}{(1+\omega)^2}\sqnorm{\hat{u}^{t+1}-u^\star}\notag\\
		&\quad&+\frac{2\omega}{(1+\omega)^2} \left\langle  u^{t}-u^\star,
		\hat{u}^{t+1}-u^\star \right\rangle +\frac{\omega}{(1+\omega)^2}\sqnorm{\hat{u}^{t+1} -u^\star }\notag\\
		&\quad&+\frac{\omega}{(1+\omega)^2}\sqnorm{u^{t} -u^\star }-\frac{2\omega}{(1+\omega)^2} \left\langle  u^{t}-u^\star,
		\hat{u}^{t+1}-u^\star \right\rangle \notag\\
		&=&\frac{1}{1+\omega}\sqnorm{\hat{u}^{t+1} -u^\star }+\frac{\omega}{1+\omega}\sqnorm{u^{t} -u^\star }. \label{ubound}
	\end{eqnarray}
	Let $(s_\iclient^{t+1})_{\iclient=1}^\Mx\in \partial F^*(\hat{u}^{t+1})$ be such that $\hat{u}_\iclient^{t+1}=u_\iclient^t + \tauM \hat{x}^{t}-\tauM s_\iclient^{t+1}$;  $s^{t+1}$ exists and is unique. We also define $s_\iclient^\star\eqdef  x^\star$; we have $s^\star\in \partial F^*(u^\star)$. 
	Therefore,
	\begin{eqnarray}
		\sqnorm{\hat{u}^{t+1}-u^\star}&=&\sqnorm{(u^t-u^\star)+(\hat{u}^{t+1}-u^t)}\notag\\
		&=&\sqnorm{u^t-u^\star}+\sqnorm{\hat{u}^{t+1}-u^t}+2 \left\langle  u^t-u^\star,\hat{u}^{t+1}-u^t \right\rangle \notag \\
		&=&\sqnorm{u^t-u^\star}+2  \left\langle  \hat{u}^{t+1}-u^\star,\hat{u}^{t+1}-u^t \right\rangle  - \sqnorm{\hat{u}^{t+1}-u^t} \notag\\
		&=&\sqnorm{u^t-u^\star} - \sqnorm{\hat{u}^{t+1}-u^t}+2\tauM  \left\langle  \Koper^\top\left(\hat{u}^{t+1}-u^\star \right), \hat{x}^{t}-x^\star \right\rangle \notag \\
		&\quad& -2\tauM  \left\langle  \hat{u}^{t+1}-u^\star,s^{t+1}-s^\star \right\rangle .\label{ubound31}
	\end{eqnarray}
	Combining \myref{ubound}, \myref{ubound31} and \myref{xbnd31} gives
	\begin{eqnarray*}
		\frac{1}{\gammaM}\Exp{\sqnorm{x^{t+1}-x^\star}\;|\;\mathcal{F}_t}&+&\frac{1+\omega}{\tauM}\Exp{\sqnorm{u^{t+1}-u^\star}\;|\;\mathcal{F}_t}\\
		&\leq&  \frac{1}{\gammaM(1+\gammaM\mu)}\sqnorm{x^t-x^\star}-\frac{\gammaM}{1+\gammaM\mu}\sqnorm{\Koper^\top\left(u^{t}-u^\star \right)}\\
		&\quad&+\gammaM(1-\zeta)\sqnorm{\Koper^\top\left(\hat{u}^{t+1}-u^t \right)}-2  \left\langle 
		\hat{x}^{t}-x^\star,\Koper^\top\left(\hat{u}^{t+1}-u^\star \right) \right\rangle \\
		&\quad&+\gammaM\oma\sqnorm{\hat{u}^{t+1}-u^t}+ \frac{1}{\tauM}\sqnorm{u^t-u^\star} - \frac{1}{\tauM}\sqnorm{\hat{u}^{t+1}-u^t}\\
		&\quad&+2  \left\langle  \Koper^\top\left(\hat{u}^{t+1}-u^\star \right), \hat{x}^{t}-x^\star \right\rangle  -2  \left\langle  \hat{u}^{t+1}-u^\star,s^{t+1}-s^\star \right\rangle  \\
		&\quad&+\frac{\omega}{\tauM}\sqnorm{u^{t} -u^\star }\\
		&\leq& \frac{1}{\gammaM(1+\gammaM\mu)}\sqnorm{x^t-x^\star}-\frac{\gammaM}{1+\gammaM\mu}\sqnorm{\Koper^\top\left(u^{t}-u^\star \right)}\\
		&\quad&+\frac{1+\omega}{\tauM}\sqnorm{u^{t} -u^\star }+ \left(\gammaM \left((1-\zeta)\Mx+\oma\right) - \frac{1}{\tauM}\right) \sqnorm{\hat{u}^{t+1}-u^t} \\
		&\quad&-2  \left\langle  \hat{u}^{t+1}-u^\star,s^{t+1}-s^\star \right\rangle \\
		&\leq& \frac{1}{\gammaM(1+\gammaM\mu)}\sqnorm{x^t-x^\star}-\frac{\gammaM}{1+\gammaM\mu}\sqnorm{\Koper^\top\left(u^{t}-u^\star \right)}\\
		&\quad&+\frac{1+\omega}{\tauM}\sqnorm{u^{t} -u^\star }-2  \left\langle  \hat{u}^{t+1}-u^\star,s^{t+1}-s^\star \right\rangle .
	\end{eqnarray*}
	By $\frac{1}{L_{F}}$-strong monotonicity of $\partial F^*$, $ \left\langle  \hat{u}^{t+1}-u^\star,s^{t+1}-s^\star \right\rangle  \geq \frac{1}{L_{F}} \sqnorm{\hat{u}^{t+1}-u^\star}$, and using \myref{ubound},
	\begin{equation*}
		\left\langle  \hat{u}^{t+1}-u^\star,s^{t+1}-s^\star \right\rangle  \geq \frac{1}{L_{F}} \left((1+\omega)\Exp{\sqnorm{u^{t+1}-u^\star}\;|\;\mathcal{F}_t}-\omega \sqnorm{u^{t} -u^\star }\right).
	\end{equation*}
	Hence, 
	\begin{eqnarray}
		\frac{1}{\gammaM}\Exp{\sqnorm{x^{t+1}-x^\star}\;|\;\mathcal{F}_t}&+&(1+\omega)\left(\frac{1}{\tauM}+2\frac{1}{L_{F}}\right)\Exp{\sqnorm{u^{t+1}-u^\star}\;|\;\mathcal{F}_t}\notag\\
		& \leq& \frac{1}{\gammaM(1+\gammaM\mu)}\sqnorm{x^t-x^\star}+\left(\frac{1+\omega}{\tauM}+2\omega \frac{1}{L_{F}} \right)
		\sqnorm{u^{t} -u^\star }\notag\\
		&\quad &-\frac{\gammaM}{1+\gammaM\mu}\sqnorm{\Koper^\top\left(u^{t}-u^\star \right)}.\label{kkkk}
	\end{eqnarray}
	Ignoring the last term in \myref{kkkk}, we obtain
	\begin{equation}
		\Exp{\Psi^{t+1}}\leq \max\left(\frac{1}{1+\gammaM\mu},1-\frac{2\tauM}{(1+\omega)(L_{F}+2\tauM)}\right)\Exp{\Psi^{t}}.\label{finresult}
	\end{equation}
\end{proof}

\subsection{Proof of Corollary~\ref{cor:5GCS-infty}}
\begin{corollary}
	Choose any $0<\varepsilon<1$. If
	we choose $\gammaM=\sqrt{\frac{2\Cx}{L_F\mu \Mx^2}}$ and $\tauM=\sqrt{\frac{L_F\mu}{2\Cx}}$, then  in order to guarantee $\Exp{\Psi^{\Tx}}\leq \varepsilon \Psi^0$, it suffices to take
	\[
		\Tx \geq	\left(\frac{\Mx}{\Cx}+\sqrt{\frac{\Mx}{ \Cx} \frac{L-\mu}{2\mu}}\right)\log \frac{1}{\varepsilon} =\tilde{\cO}\left(\frac{\Mx}{\Cx}+\sqrt{\frac{\Mx}{\Cx}\frac{L}{\mu}}\right)
	\]
	communication rounds.
\end{corollary}
\begin{proof}
	Firstly, note that choosing $\gammaM=\sqrt{\frac{2\Cx}{L_{F}\mu \Mx^2}}$ and $\tauM=\sqrt{\frac{L_{F}\mu}{2\Cx}}$ we satisfy $\gammaM\tauM=\frac{1}{\Mx}$, than that we get the contraction constant from the proof to be equal to:
	\begin{eqnarray} \max\left\{1-\frac{\sqrt{\frac{2\Cx\mu}{L_{F}\Mx^2}}}{1+\sqrt{\frac{2\Cx\mu}{L_{F}\Mx^2}}},1-\frac{\sqrt{\frac{2L_{h}\mu}{\Cx}}}{\frac{\Mx}{\Cx}\left(L_{F}+\sqrt{\frac{2L_{F}\mu}{\Cx}}\right)}\right\}\notag & = &
		\max\left\{1-\frac{\sqrt{2\Cx\mu}}{\Mx\sqrt{L_{F}}+\sqrt{2\Cx\mu}},1-\frac{\sqrt{2\Cx\mu}}{\Mx\sqrt{L_{F}}+\sqrt{\frac{2\mu \Mx^2}{\Cx}}}\right\}\notag\\
		& =&1-\frac{\sqrt{2\Cx\mu}}{\Mx\sqrt{L_{F}}+\sqrt{\frac{2\mu \Mx^2}{\Cx}}}\notag.
	\end{eqnarray}
	This  gives a rate of 
	\begin{align*}
		T = \mathcal{O}\left(\frac{\Mx\sqrt{L_{F}}+\sqrt{\frac{2\mu \Mx^2}{\Cx}}}{\sqrt{2\Cx\mu}}\log\frac{1}{\varepsilon}\right)=	\mathcal{O}\left(\left(\frac{\Mx}{\Cx}+\sqrt{\frac{(L-\mu)\Mx}{2\mu \Cx}}\right)\log \frac{1}{\varepsilon} \right) .
	\end{align*}
\end{proof}

\clearpage

\section{Analysis of 5GCS}\label{proof35}
\begin{theorem} Consider Algorithm~\ref{alg:5GCS} (\algname{5GCS}) with the LT solver being \algname{GD} run for $$\Kx \geq \left(\frac{3}{4}\sqrt{\frac{\Cx}{\Mx}\frac{L}{\mu}}+2\right)\log\left(4\frac{L}{\mu}\right)$$ iterations.
	Let $0<\gammaM\leq \frac{3}{16}\sqrt{\frac{\Cx}{L\mu \Mx}}$ and $\tauM=\frac{1}{2\gammaM \Mx}$. Then for the Lyapunov function
	\begin{equation*}
		\Psi^{\kstep}\eqdef \frac{1}{\gammaM}\sqnorm{x^{\kstep}-x^\star}+\frac{\Mx}{\Cx}\left(\frac{1}{\tauM}+\frac{1}{L_F}\right)\sqnorm{u^{\kstep}-u^\star},
	\end{equation*}
	the iterates of  the method satisfy
	$$		\Exp{\Psi^{\Tx}}\leq (1-\rho)^\Tx \Psi^0,
	$$	where 
	$
	\rho\eqdef  \max\left\{\frac{\gammaM\mu}{1+\gammaM\mu},\frac{\Cx}{\Mx}\frac{\tauM}{(L_F+\tauM)}\right\}<1.
	$

\end{theorem}
\begin{proof}
	Noting that updates for $u^{t+1}$ and $x^{t+1}$ can be written as
	\begin{eqnarray}
		&u^{t+1}\eqdef u^t + \frac{1}{1+\omega} \Rop^t\left(\bar{u}^{t+1}-u^t\right),&\\
		&x^{t+1}=\hat{x}^t-\gammaM\left(\omega+1\right)\Koper^\top\left(u^{t+1}-u^t\right)&\label{xupdate35}
	\end{eqnarray}
	where $\Rop^t$ is the client sampling operator, $\omega=\frac{\Mx}{\Cx}-1$ and  $\bar{u}^{t+1} = \nabla F(\lastlocitterk)$. 
	Then using variance decomposition and Proposition 1 from \citep{condat2021murana}, we obtain
	\begin{eqnarray}
		\label{eq:starting_eq3.5}
		\Exp{\sqnorm{x^{t+1}-x^\star}\;|\;\mathcal{F}_t}&\overset{(\ref{vardec})}{=}&\sqnorm{\Exp{x^{t+1}\;|\;\mathcal{F}_t}-x^\star}+\Exp{\sqnorm{x^{t+1}-\Exp{x^{t+1}\;|\;\mathcal{F}_t}}\;|\;\mathcal{F}_t}\notag\\
		&\overset{(\ref{xupdate35})+(\ref{PP})}{=}& \underbrace{\sqnorm{\hat{x}^{t}-x^\star-\gammaM \Koper^\top(\bar{u}^{t+1}-u^t)}}_{X}+\gammaM^2\oma\sqnorm{\bar{u}^{t+1}-u^t}\notag\\
		&\quad &- \gammaM^2\zeta\sqnorm{\Koper^\top(\bar{u}^{t+1}-u^t)},
	\end{eqnarray}
	where
	\begin{equation*}
		\oma=\frac{\Mx(\Mx-\Cx)}{\Cx(\Mx-1)},\quad \zeta=\frac{\Mx-\Cx}{\Cx(\Mx-1)}.
	\end{equation*}
	Moreover, using \myref{fooc} and the definition of $\hat{x}^t$, we have
	\begin{align}
		&(1+\gammaM\mu)\hat{x}^t=x^t-\gammaM \Koper^\top u^{t},\label{opt1}\\
		&(1+\gammaM\mu)x^\star= x^\star -\gammaM \Koper^\top u^\star.\label{opt2}
	\end{align}
	Using \myref{opt1} and \myref{opt2} we obtain
	\begin{eqnarray}
		\label{eq:long3.5}
		X &=&	\sqnorm{\hat{x}^{t}-x^\star-\gammaM \Koper^\top(\bar{u}^{t+1}-u^t)}\notag\\
		&=&\sqnorm{\hat{x}^{t}-x^\star} +\gammaM^2\sqnorm{\Koper^\top (\bar{u}^{t+1}-u^t)}-2\gammaM \left\langle 
		\hat{x}^{t}-x^\star,\Koper^\top(\bar{u}^{t+1}-u^t) \right\rangle  \notag\\
		&=& (1+\gammaM\mu) \sqnorm{\hat{x}^{t}-x^\star} +\gammaM^2\sqnorm{\Koper^\top(\bar{u}^{t+1}-u^t)}\notag\\
		&\quad&-2\gammaM  \left\langle 
		\hat{x}^{t}-x^\star,\Koper^\top (\bar{u}^{t+1}-u^\star) \right\rangle  +2\gammaM  \left\langle 
		\hat{x}^{t}-x^\star,\Koper^\top (u^{t}-u^\star) \right\rangle  -\gammaM\mu\sqnorm{\hat{x}^{t}-x^\star} \notag\\
		&\overset{(\ref{opt1})+(\ref{opt2})}{=}&   \left\langle  x^t-x^\star-\gammaM \Koper^\top (u^t-u^\star),\hat{x}^{t}-x^\star \right\rangle +\gammaM^2\sqnorm{\Koper^\top (\bar{u}^{t+1}-u^t)}\notag\\
		&\quad&-2\gammaM  \left\langle 
		\hat{x}^{t}-x^\star,\Koper^\top (\bar{u}^{t+1}-u^\star) \right\rangle  +  \left\langle 
		\hat{x}^{t}-x^\star,2\gammaM \Koper^\top (u^{t}-u^\star) \right\rangle  -\gammaM\mu\sqnorm{\hat{x}^{t}-x^\star} \notag.
			\end{eqnarray}
This leads to	
			\begin{eqnarray}
	X	&= &\left\langle  x^t-x^\star+\gammaM \Koper^\top (u^t-u^\star),\hat{x}^{t}-x^\star \right\rangle  \notag\\
		&\quad&+\gammaM^2\sqnorm{\Koper^\top (\bar{u}^{t+1}-u^t)}-2\gammaM  \left\langle 
		\hat{x}^{t}-x^\star, \Koper^\top(\bar{u}^{t+1}-u^\star) \right\rangle -\gammaM\mu\sqnorm{\hat{x}^{t}-x^\star}\notag \\
		&\overset{(\ref{opt1})+(\ref{opt2})}{=}&\frac{1}{1+\gammaM\mu} \left\langle  x^t-x^\star+\gammaM \Koper^\top (u^t-u^\star),x^t-x^\star-\gammaM \Koper^\top (u^t-u^\star) \right\rangle  \notag\\
		&\quad&+\gammaM^2\sqnorm{\Koper^\top (\bar{u}^{t+1}-u^t)}-2\gammaM  \left\langle 
		\hat{x}^{t}-x^\star, \Koper^\top (\bar{u}^{t+1}-u^\star) \right\rangle -\gammaM\mu\sqnorm{\hat{x}^{t}-x^\star} \notag\\
		&=& \frac{1}{1+\gammaM\mu}\sqnorm{x^t-x^\star}-\frac{\gammaM^2}{1+\gammaM\mu}\sqnorm{\Koper^\top (u^t-u^\star)}\notag\\
		&\quad&+\gammaM^2\sqnorm{\Koper^\top (\bar{u}^{t+1}-u^t)}-2\gammaM  \left\langle 
		\hat{x}^{t}-x^\star,\Koper^\top (\bar{u}^{t+1}-u^\star) \right\rangle -\gammaM\mu\sqnorm{\hat{x}^{t}-x^\star}   . 
	\end{eqnarray}
	Combining \myref{eq:starting_eq3.5} and \myref{eq:long3.5}, we get
	\begin{eqnarray*}
		\Exp{\sqnorm{x^{t+1}-x^\star}\;|\;\mathcal{F}_t}&\leq&  \frac{1}{1+\gammaM\mu}\sqnorm{x^t-x^\star}-\frac{\gammaM^2}{1+\gammaM\mu}\sqnorm{\Koper^\top (u^t-u^\star)}\\
		&\quad&+\gammaM^2(1-\zeta)\sqnorm{\Koper^\top (\bar{u}^{t+1}-u^t)}-2\gammaM  \left\langle 
		\hat{x}^{t}-x^\star,\Koper^\top (\bar{u}^{t+1}-u^\star) \right\rangle \\
		&\quad&+\gammaM^2\oma\sqnorm{\bar{u}^{t+1}-u^t}-\frac{\gammaM\mu}{\Mx}\sqnorm{\Koper\hat{x}^t-\Koper x^\star}.
	\end{eqnarray*}
	Note that we can have the update rule for $u$ as: 
	\begin{equation*}
		u^{t+1}\eqdef u^t + \frac{1}{1+\omega} \Rop^t\left(\bar{u}^{t+1}-u^t\right),
	\end{equation*}
	where $\Rop^t$ is client sampling operator with parameter  $\omega=\frac{\Mx}{\Cx}-1$. Using conic variance formula \myref{cvar} of $\Rop^t$, we obtain
	\begin{eqnarray}
		\Exp{\sqnorm{u^{t+1}-u^\star}\;|\;\mathcal{F}_t}&\overset{(\ref{vardec})+(\ref{cvar})}{\leq}& \sqnorm{u^{t}-u^\star+\frac{1}{1+\omega}\left(\bar{u}^{t+1} -u^t\right)}
		+\frac{\omega}{(1+\omega)^2}\sqnorm{\bar{u}^{t+1} -u^t }\notag\\
		&=&\frac{\omega^2}{(1+\omega)^2}\sqnorm{u^{t}-u^\star}+\frac{1}{(1+\omega)^2}\sqnorm{\bar{u}^{t+1}-u^\star}\notag\\
		&\quad&+\frac{2\omega}{(1+\omega)^2} \left\langle  u^{t}-u^\star,
		\bar{u}^{t+1}-u^\star \right\rangle +\frac{\omega}{(1+\omega)^2}\sqnorm{\bar{u}^{t+1} -u^\star }\notag\\
		&\quad&+\frac{\omega}{(1+\omega)^2}\sqnorm{u^{t} -u^\star }-\frac{2\omega}{(1+\omega)^2} \left\langle  u^{t}-u^\star,
		\bar{u}^{t+1}-u^\star \right\rangle \notag\\
		&=&\frac{1}{1+\omega}\sqnorm{\bar{u}^{t+1} -u^\star }+\frac{\omega}{1+\omega}\sqnorm{u^{t} -u^\star }.\label{expu}
	\end{eqnarray}
	Let us consider the first term in \myref{expu}:
	\begin{eqnarray*}
		\sqnorm{\bar{u}^{t+1}-u^\star}&=&\sqnorm{(u^t-u^\star)+(\bar{u}^{t+1}-u^t)}\\
		&=&\sqnorm{u^t-u^\star}+\sqnorm{\bar{u}^{t+1}-u^t}+2 \left\langle  u^t-u^\star,\bar{u}^{t+1}-u^t \right\rangle \\
		&=&\sqnorm{u^t-u^\star}+2  \left\langle  \bar{u}^{t+1}-u^\star,\bar{u}^{t+1}-u^t \right\rangle  - \sqnorm{\bar{u}^{t+1}-u^t}.
	\end{eqnarray*}
	Combining the terms together, we get
	\begin{equation*}
		\Exp{\sqnorm{u^{t+1}-u^\star}\;|\;\mathcal{F}_t}\leq \sqnorm{u^{t} -u^\star }+\frac{1}{1+\omega}\left(2  \left\langle  \bar{u}^{t+1}-u^\star,\bar{u}^{t+1}-u^t \right\rangle  - \sqnorm{\bar{u}^{t+1}-u^t}\right).
	\end{equation*}
	Finally, we obtain
	\begin{eqnarray*}
		\frac{1}{\gammaM}\Exp{\sqnorm{x^{t+1}-x^\star}\;|\;\mathcal{F}_t}&+&\frac{1+\omega}{\tauM}\Exp{\sqnorm{u^{t+1}-u^\star}\;|\;\mathcal{F}_t}\\
		&\leq&  \frac{1}{\gammaM(1+\gammaM\mu)}\sqnorm{x^t-x^\star}-\frac{\gammaM}{1+\gammaM\mu}\sqnorm{\Koper^\top (u^t-u^\star)}\\
		&\quad&+\gammaM(1-\zeta)\sqnorm{\Koper^\top (\bar{u}^{t+1}-u^t)}\\
		&\quad&+\gammaM\oma\sqnorm{\bar{u}^{t+1}-u^t}-\frac{\mu}{\Mx}\sqnorm{\Koper \hat{x}^t- \Koper x^\star}\\
		&\quad&+\frac{1+\omega}{\tauM}\sqnorm{u^{t} -u^\star }-2  \left\langle 
		\hat{x}^{t}-x^\star,\Koper^\top (\bar{u}^{t+1}-u^\star) \right\rangle \\ &\quad&+\frac{1}{\tauM}\left(2  \left\langle  \bar{u}^{t+1}-u^\star,\bar{u}^{t+1}-u^t \right\rangle  - \sqnorm{\bar{u}^{t+1}-u^t}\right).
	\end{eqnarray*}
	Ignoring $-\frac{\gammaM}{1+\gammaM\mu}\sqnorm{\Koper^\top (u^t-u^\star)}$ and noting that
	\begin{eqnarray*}
		-\left\langle \hat{x}^{t}-x^\star,\Koper^\top (\bar{u}^{t+1}-u^\star) \right\rangle &+&\frac{1}{\tauM}\left\langle  \bar{u}^{t+1}-u^\star,\bar{u}^{t+1}-u^t \right\rangle  \\
		&=&-\left\langle \lastlocitterk-\Koper x^\star,\bar{u}^{t+1}-u^\star \right\rangle +\frac{1}{\tauM}\left\langle  \nabla\localfun (\lastlocitterk),\bar{u}^{t+1}-u^\star \right\rangle\\
		&\overset{(\ref{yi1})+(\ref{strmono})}{\leq}& -\frac{1}{L_F}\sqnorm{\bar{u}^{t+1}-u^\star}+\frac{a}{2\tauM}\sqnorm{\nabla\localfun (\lastlocitterk)}+\frac{1}{2a\tauM}\sqnorm{\bar{u}^{t+1}-u^\star }\\
		&=& -\left(\frac{1}{L_F}-\frac{1}{2a\tauM}\right)\sqnorm{\bar{u}^{t+1}-u^\star}+\frac{a}{2\tauM}\sqnorm{\nabla\localfun (\lastlocitterk)}\\
		&\overset{(\ref{expu})}{\leq}& -\left(\frac{1}{L_F}-\frac{1}{2a\tauM}\right)\left((1+\omega)\Exp{\sqnorm{u^{t+1}-u^\star}\;|\;\mathcal{F}_t}-\omega\sqnorm{u^t-u^\star}\right)\\
		&\quad&+\frac{a}{2\tauM}\sqnorm{\nabla\localfun (\lastlocitterk)},
	\end{eqnarray*}
	we get
	\begin{eqnarray*}
		\frac{1}{\gammaM}\Exp{\sqnorm{x^{t+1}-x^\star}\;|\;\mathcal{F}_t}&+&\left(1+\omega\right)\left(\frac{1}{\tauM}+\frac{1}{L_F}\right)\Exp{\sqnorm{u^{t+1}-u^\star}\;|\;\mathcal{F}_t}\\
		&\leq&  \frac{1}{\gammaM(1+\gammaM\mu)}\sqnorm{x^t-x^\star}\\
		&\quad&+\left(1+\omega\right)\left(\frac{1}{\tauM}+\frac{\omega}{1+\omega}\frac{1}{L_F}\right)\sqnorm{u^{t} -u^\star } \\
		&\quad&+\left(\gammaM\left(1-\zeta\right)\Mx+\gammaM\oma-\frac{1}{\tauM}\right)\sqnorm{\bar{u}^{t+1}-u^t}\\
		&\quad& +\frac{L_F}{\tauM^2}\sqnorm{\nabla\localfun (\lastlocitterk)}-\frac{\mu}{\Mx}\sqnorm{\Koper\hat{x}^t-\Koper x^\star}.
	\end{eqnarray*}
	Where we made the choice $a=\frac{L_F}{\tau}$.
	Using Young's inequality we have
	\begin{equation*}
		-\frac{\mu}{3\Mx}\sqnorm{\Koper \hat{x}^t-\localsolk+\localsolk-\Koper x^\star}\overset{(\ref{yi3})}{\leq}\frac{\mu}{3\Mx}\sqnorm{\localsolk-\Koper x^\star}-\frac{\mu}{6\Mx}\sqnorm{\Koper\hat{x}^t-\localsolk}.	\end{equation*}
	Noting the fact that $\localsolk=\Koper\hat{x}^t-\frac{1}{\tauM}(\hat{u}^{t+1}-u^t)$, we have
	$$\frac{\mu}{3\Mx}\sqnorm{\localsolk-\Koper x^\star}\overset{(\ref{yi2})}{\leq} 2\frac{\mu}{3\Mx}\sqnorm{\Koper\hat{x}^t-\Koper x^\star}+\frac{2}{\tauM^2}\frac{\mu}{3\Mx}\sqnorm{\hat{u}^{t+1}-u^t}.$$
	Combining those inequalities, we get
	\begin{eqnarray*}
		\frac{1}{\gammaM}\Exp{\sqnorm{x^{t+1}-x^\star}\;|\;\mathcal{F}_t}&+&\left(1+\omega\right)\left(\frac{1}{\tauM}+\frac{1}{L_F}\right)\Exp{\sqnorm{u^{t+1}-u^\star}\;|\;\mathcal{F}_t}\\
		&\leq&  \frac{1}{\gammaM(1+\gammaM\mu)}\sqnorm{x^t-x^\star}\\
		&\quad&+\left(1+\omega\right)\left(\frac{1}{\tauM}+\frac{\omega}{1+\omega}\frac{1}{L_F}\right)\sqnorm{u^{t} -u^\star } \\
		&\quad&+\frac{2}{\tauM^2}\frac{\mu}{3\Mx}\sqnorm{\hat{u}^{t+1}-u^t}\\
		&\quad&-\left(\frac{1}{\tauM}-\left(\gammaM\left(1-\zeta\right)\Mx+\gammaM\oma\right)\right)\sqnorm{\bar{u}^{t+1}-u^t}\\
		&\quad&+\frac{L_F}{\tauM^2}\sqnorm{\nabla\localfun (\lastlocitterk)}-\frac{\mu}{6\Mx}\sqnorm{\Koper\hat{x}^t-\localsolk}.
	\end{eqnarray*}
	Assuming $\gammaM$ and $\tauM$ can be chosen so that $\frac{1}{\tauM}-(\gammaM(1-\zeta)\Mx+\gammaM\oma))\geq \frac{4}{\tauM^2}\frac{\mu}{3\Mx}$ we obtain
	\begin{eqnarray*}
		\frac{1}{\gammaM}\Exp{\sqnorm{x^{t+1}-x^\star}\;|\;\mathcal{F}_t}&+&\left(1+\omega\right)\left(\frac{1}{\tauM}+\frac{1}{L_F}\right)\Exp{\sqnorm{u^{t+1}-u^\star}\;|\;\mathcal{F}_t}\\
		&\leq&  \frac{1}{\gammaM(1+\gammaM\mu)}\sqnorm{x^t-x^\star}\\
		&\quad&+\left(1+\omega\right)\left(\frac{1}{\tauM}+\frac{\omega}{1+\omega}\frac{1}{L_F}\right)\sqnorm{u^{t} -u^\star } \\
		&\quad&+\frac{4}{\tauM^2}\frac{\mu L_F^2}{3\Mx}\sqnorm{\lastlocitterk-\localsolk}+\frac{L_F}{\tauM^2}\sqnorm{\nabla\localfun (\lastlocitterk)}\\
		&\quad&-\frac{\mu}{6\Mx}\sqnorm{\Koper\hat{x}^t-\localsolk}.
	\end{eqnarray*}
	Where the point $\lastlocitterk$ is assumed to satisfy 
	\begin{eqnarray*}
		\frac{4}{\tauM^2}\frac{\mu L_F^2}{3\Mx}\sqnorm{\lastlocitterk-\localsolk}+\frac{L_F}{\tauM^2}\sqnorm{\nabla\localfun (\lastlocitterk)}\leq\frac{\mu}{6\Mx}\sqnorm{\Koper\hat{x}^t-\localsolk}.
	\end{eqnarray*}
	Thus
	\begin{eqnarray*}
		\frac{1}{\gammaM}\Exp{\sqnorm{x^{t+1}-x^\star}\;|\;\mathcal{F}_t}&+&\left(1+\omega\right)\left(\frac{1}{\tauM}+\frac{1}{L_F}\right)\Exp{\sqnorm{u^{t+1}-u^\star}\;|\;\mathcal{F}_t}\\
		&\leq&  \frac{1}{\gammaM(1+\gammaM\mu)}\sqnorm{x^t-x^\star}\\
		&\quad&+\left(1+\omega\right)\left(\frac{1}{\tauM}+\frac{\omega}{1+\omega}\frac{1}{L_F}\right)\sqnorm{u^{t} -u^\star }.
	\end{eqnarray*}
	
	By taking the expectation on both sides we get
	\begin{equation*}
		\Exp{\Psi^{t+1}}\leq  \max\left\{\frac{1}{1+\gammaM\mu} ,\frac{ L_F+\frac{\Mx-\Cx}{\Cx}\tauM}{L_F+\tauM} \right\}\Exp{\Psi^{t}},
	\end{equation*}
	which finishes the proof. 
	Note that our standard choice of constants is $$\omega=\frac{\Mx}{\Cx}-1,\quad \oma=\frac{\Mx(\Mx-\Cx)}{\Cx(\Mx-1)},\quad \zeta=\frac{\Mx-\Cx}{\Cx(\Mx-1)}.$$ Using these parameters the requirement for stepsizes becomes:$$\frac{1}{\tauM}-\gammaM \Mx\geq \frac{4\mu}{3\Mx\tauM^2}.$$
	This inequality is satisfied, when $0<\gammaM\leq \frac{3}{16}\sqrt{\frac{\Cx}{L\mu \Mx}}$ and $\tauM=\frac{1}{2\Mx\gammaM}.$
\end{proof}
\subsection{Proof of Corollary \ref{cor:5GCS}}

\begin{corollary}
Choose any $0<\varepsilon<1$ and $\gammaM = \frac{3}{16}\sqrt{\frac{\Cx}{L\mu \Mx}}$. In order to guarantee $\Exp{\Psi^{\Tx}}\leq \varepsilon \Psi^0$, it suffices to take 
\begin{eqnarray}
	T &\geq	&
	\max\left\{1+\frac{16}{3}\sqrt{\frac{\Mx}{\Cx}\frac{L}{\mu}},\frac{\Mx}{\Cx}+\frac{3}{8}\sqrt{\frac{\Mx}{\Cx}\frac{L}{\mu} }\right\}\log\frac{1}{\varepsilon} =
	\tilde{\cO}\left(\frac{\Mx}{\Cx}+\sqrt{\frac{\Mx}{\Cx}\frac{L}{\mu }}\right) \notag
\end{eqnarray}
communication rounds.
\end{corollary}
\begin{proof}
	Choosing the maximal $\gammaM=\frac{3}{16}\sqrt{\frac{\Cx}{L\mu \Mx}}$ and $a=\frac{L_F}{\tauM}$ we have 
	\begin{eqnarray*}
		\max\left\{\frac{1}{1+\gammaM\mu},\frac{\frac{1}{\tauM}+\frac{\Mx-\Cx}{\Mx}\frac{1}{L_F}}{\frac{1}{\tauM}+\frac{1}{L_F}}\right\}&=&\max\left\{\frac{1}{1+\frac{3}{16}\sqrt{\frac{\mu \Cx}{L \Mx}}},\frac{\frac{1}{\tauM}+\frac{\Mx-\Cx}{\Mx}\frac{1}{L_F}}{\frac{1}{\tauM}+\frac{1}{L_F}}\right\}\\
		&=&\max\left\{\frac{1}{1+\frac{3}{16}\sqrt{\frac{\mu \Cx}{L \Mx}}},1-\frac{\frac{8\Cx}{3\Mx L_F}\sqrt{\frac{L\mu}{\Mx \Cx}}}{1+\frac{8\Mx}{3\Mx L_F}\sqrt{\frac{L\mu}{\Mx \Cx}}}\right\}\\
		&\leq& \max\left\{\frac{1}{1+\frac{3}{16}\sqrt{\frac{\mu \Cx}{L \Mx}}},1-\frac{\frac{8}{3}\sqrt{\frac{\Cx\mu}{\Mx L}}}{1+\frac{8}{3}\sqrt{\frac{\Mx\mu}{L\Cx}}}\right\}.
	\end{eqnarray*}
	Thus Algorithm~\ref{alg:5GCS} finds $\varepsilon$-solution in:
	\begin{equation*}
		T\geq\mathcal{O}\left(\max\left\{1+\frac{16}{3}\sqrt{\frac{L\Mx}{\mu \Cx}},\frac{\Mx}{\Cx}+\frac{3}{8}\sqrt{\frac{L\Mx}{\mu \Cx}}\right\}\log\frac{1}{\varepsilon}\right)
	\end{equation*}
	communications.	
\end{proof}

\clearpage
\section{Analysis of 5GCS${}_0$}
\subsection{Proof of Theorem~\ref{thm:5GCS-0}}
\begin{algorithm*}
	\caption{\algn{5GCS} with 0 local \algname{GD} steps }
	\begin{algorithmic}[1]\label{alg:5GCS-0}
		\STATE  \textbf{input:} initial points $x^0\in\mathbb{R}^d$, $u_\iclient^0\in\mathbb{R}^d$ for all $\iclient=\{1,\dots,\Mx\}$; 
		\STATE stepsize $\gammaM>0$, $\tauM>0$
		\STATE $v^0\eqdef \sum_{\iclient=1}^\Mx u_\iclient^0$
		\FOR{$\kstep=0, 1, \ldots$}
		\STATE $\hat{x}^{\kstep} \eqdef \frac{1}{1+\gammaM\mu} \left(x^\kstep - \gammaM v^\kstep\right)$
		\STATE Pick $\set\subset \{1,\ldots,\Mx\}$ of size $\Cx$ uniform at random
		\FOR{$\iclient\in \set$}
		\STATE $u_\iclient^{\kstep+1}\eqdef \nabla F_\iclient(\hat{x}^\kstep)=\frac{1}{\Mx}\left(\nabla f_\iclient(\hat{x}^\kstep)-\mu\hat{x}^\kstep\right)$
		\ENDFOR
		\FOR{$\iclient\in\{1, \ldots,\Mx\}\backslash \set$}
		\STATE $u_{\iclient}^{\kstep+1}\eqdef u_{\iclient}^\kstep $
		\ENDFOR
		\STATE $v^{\kstep+1}\eqdef \sum_{\iclient=1}^\Mx u_\iclient^{\kstep+1}$
		\STATE  $x^{\kstep+1} \eqdef \hat{x}^{\kstep}- \gammaM \frac{\Mx}{\Cx}(v^{\kstep+1}-v^\kstep)$
		\ENDFOR
	\end{algorithmic}
\end{algorithm*}

\begin{theorem}
	Consider Algorithm~\ref{alg:5GCS} (\algname{5GCS}) with the LT solver being \algname{GD} run for $\Kx=0$ iterations (this is equivalent to Algorithm~\ref{alg:5GCS-0}; we shall also call the method \algname{5GCS${}_0$}).	Let $0<\gammaM\leq \frac{\Cx}{4L\Mx}$. Then for the Lyapunov function
	\begin{equation*}
		\Psi^{\kstep}\eqdef  \frac{\Cx}{\Mx^2\gammaM^2}\left(1-\sqrt{\frac{\gammaM \Mx L_F}{2}}\right)\sqnorm{x^{\kstep}-x^\star}+\sqnorm{u^{\kstep}-u^\star},
	\end{equation*}
	the iterates of the method satisfy
	$$		\Exp{\Psi^{\Tx}}\leq \left(1-\rho\right)^\Tx \Psi^0,
	$$	
	where 
	$
	\rho\eqdef  \min\left(\frac{\gammaM\mu}{1+\gammaM\mu},\frac{\Cx}{\Mx+2\gammaM  L_F\Mx^2}\right)<1.
	$

\end{theorem}
\begin{proof}
	Noting that updates for $u^{t+1}$ and $x^{t+1}$ can be written as
	\begin{eqnarray}
		&u^{t+1}\eqdef u^t + \frac{1}{1+\omega} \Rop^t\left(\bar{u}^{t+1}-u^t\right),&\\
		&x^{t+1}=\hat{x}^t-\gammaM\left(\omega+1\right)\Koper^\top\left(u^{t+1}-u^t\right)&\label{xupdate33}
	\end{eqnarray}
	where $\Rop^t$ is a client sampling operator, $\omega=\frac{\Mx}{\Cx}-1$ and  $\bar{u}^{t+1} = \nabla F(\Koper\hat{x}^t)$. 
	Then using variance decomposition and Proposition 1 from \citep{condat2021murana}, we obtain
	\begin{eqnarray}
		\label{eq:starting_eq3.3}
		\Exp{\sqnorm{x^{t+1}-x^\star}\;|\;\mathcal{F}_t}&\overset{(\ref{vardec})}{=}&\sqnorm{\Exp{x^{t+1}\;|\;\mathcal{F}_t}-x^\star}+\Exp{\sqnorm{x^{t+1}-\Exp{x^{t+1}\;|\;\mathcal{F}_t}}\;|\;\mathcal{F}_t}\notag\\
		&\overset{(\ref{xupdate33})+(\ref{PP})}{=}& \underbrace{\sqnorm{\hat{x}^{t}-x^\star-\gammaM \Koper^\top(\bar{u}^{t+1}-u^t)}}_{X}+\gammaM^2\oma\sqnorm{\bar{u}^{t+1}-u^t}\notag\\
		&\quad &- \gammaM^2\zeta\sqnorm{\Koper^\top(\bar{u}^{t+1}-u^t)},
	\end{eqnarray}
	where
	\begin{equation*}
		\oma=\frac{\Mx(\Mx-\Cx)}{\Cx(\Mx-1)},\quad \zeta=\frac{\Mx-\Cx}{\Cx(\Mx-1)}.
	\end{equation*}
	Moreover, using \myref{fooc} and the definition of $\hat{x}^t$, we have
	\begin{align}
		&(1+\gammaM\mu)\hat{x}^t=x^t-\gammaM \Koper^\top u^{t},\label{opt12}\\
		&(1+\gammaM\mu)x^\star= x^\star -\gammaM \Koper^\top u^\star.\label{opt22}
	\end{align}
	Using \myref{opt12} and \myref{opt22} we obtain
	\begin{eqnarray}\label{eq:long3.3}
		X &=&	\sqnorm{\hat{x}^{t}-x^\star-\gammaM \Koper^\top(\bar{u}^{t+1}-u^t)}\notag\\
		&=&\sqnorm{\hat{x}^{t}-x^\star} +\gammaM^2\sqnorm{\Koper^\top (\bar{u}^{t+1}-u^t)}\notag\\
		&\quad&-2\gammaM \left\langle 
		\hat{x}^{t}-x^\star,\Koper^\top(\bar{u}^{t+1}-u^t) \right\rangle  \notag\\
		&\leq& (1+\gammaM\mu) \sqnorm{\hat{x}^{t}-x^\star} +\gammaM^2\sqnorm{\Koper^\top(\bar{u}^{t+1}-u^t)}\notag\\
		&\quad&-2\gammaM  \left\langle 
		\hat{x}^{t}-x^\star,\Koper^\top (\bar{u}^{t+1}-u^\star) \right\rangle  +2\gammaM  \left\langle 
		\hat{x}^{t}-x^\star,\Koper^\top (u^{t}-u^\star) \right\rangle  \notag\\
		&\overset{(\ref{opt12})+(\ref{opt22})}{=}&   \left\langle  x^t-x^\star-\gammaM \Koper^\top (u^t-u^\star),\hat{x}^{t}-x^\star \right\rangle +\gammaM^2\sqnorm{\Koper^\top (\bar{u}^{t+1}-u^t)}\notag\\
		&\quad&-2\gammaM  \left\langle 
		\hat{x}^{t}-x^\star,\Koper^\top (\bar{u}^{t+1}-u^\star) \right\rangle  +  \left\langle 
		\hat{x}^{t}-x^\star,2\gammaM \Koper^\top (u^{t}-u^\star) \right\rangle  \notag\\
		&= &\left\langle  x^t-x^\star+\gammaM \Koper^\top (u^t-u^\star),\hat{x}^{t}-x^\star \right\rangle  \notag\\
		&\quad&+\gammaM^2\sqnorm{\Koper^\top (\bar{u}^{t+1}-u^t)}-2\gammaM  \left\langle 
		\hat{x}^{t}-x^\star, \Koper^\top(\bar{u}^{t+1}-u^\star) \right\rangle  \notag\\
		&\overset{(\ref{opt12})+(\ref{opt22})}{=}&\frac{1}{1+\gammaM\mu} \left\langle  x^t-x^\star+\gammaM \Koper^\top (u^t-u^\star),x^t-x^\star-\gammaM \Koper^\top (u^t-u^\star) \right\rangle  \notag\\
		&\quad&+\gammaM^2\sqnorm{\Koper^\top (\bar{u}^{t+1}-u^t)}-2\gammaM  \left\langle 
		\hat{x}^{t}-x^\star, \Koper^\top (\bar{u}^{t+1}-u^\star) \right\rangle  \notag\\
		&=& \frac{1}{1+\gammaM\mu}\sqnorm{x^t-x^\star}-\frac{\gammaM^2}{1+\gammaM\mu}\sqnorm{\Koper^\top (u^t-u^\star)}\notag\\
		&\quad&+\gammaM^2\sqnorm{\Koper^\top (\bar{u}^{t+1}-u^t)}-2\gammaM  \left\langle 
		\hat{x}^{t}-x^\star,\Koper^\top (\bar{u}^{t+1}-u^\star) \right\rangle.
	\end{eqnarray}
	Combining \myref{eq:starting_eq3.3} and \myref{eq:long3.3} we have
	\begin{eqnarray}
		\Exp{\sqnorm{x^{t+1}-x^\star}\;|\;\mathcal{F}_t}&\leq&  \frac{1}{1+\gammaM\mu}\sqnorm{x^t-x^\star}-\frac{\gammaM^2}{1+\gammaM\mu}\sqnorm{\Koper^\top (u^t-u^\star)}\notag\\
		&\quad&+\gammaM^2(1-\zeta)\sqnorm{\Koper^\top (\bar{u}^{t+1}-u^t)}-2\gammaM  \left\langle 
		\hat{x}^{t}-x^\star,\Koper^\top (\bar{u}^{t+1}-u^\star) \right\rangle \notag\\
		&\quad&+\gammaM^2\oma\sqnorm{\bar{u}^{t+1}-u^t}.\label{xbnd35}
	\end{eqnarray}
	On the other hand using the variance decomposition and conic variance of $\Rop^t$ 
	\begin{eqnarray}
		\Exp{\sqnorm{u^{t+1}-u^\star}\;|\;\mathcal{F}_t}&\overset{(\ref{vardec})+(\ref{cvar})}{\leq}& \sqnorm{u^{t}-u^\star+\frac{1}{1+\omega}\left(\bar{u}^{t+1} -u^t\right)}
		+\frac{\omega}{(1+\omega)^2}\sqnorm{\bar{u}^{t+1} -u^t }\notag\\ 
		&=& \sqnorm{\frac{\omega}{1+\omega}(u^{t}-u^\star)+\frac{1}{1+\omega}\left(\bar{u}^{t+1} -u^\star\right)}
		+\frac{\omega}{(1+\omega)^2}\sqnorm{\bar{u}^{t+1} - u^\star - (u^t - u^\star) }\notag\\
		&=&\frac{\omega^2}{(1+\omega)^2}\sqnorm{u^{t}-u^\star}+\frac{1}{(1+\omega)^2}\sqnorm{\bar{u}^{t+1}-u^\star}\notag\\
		&\quad&+\frac{2\omega}{(1+\omega)^2} \left\langle  u^{t}-u^\star,
		\bar{u}^{t+1}-u^\star \right\rangle +\frac{\omega}{(1+\omega)^2}\sqnorm{\bar{u}^{t+1} -u^\star }\notag\\
		&\quad&+\frac{\omega}{(1+\omega)^2}\sqnorm{u^{t} -u^\star }-\frac{2\omega}{(1+\omega)^2} \left\langle  u^{t}-u^\star,
		\bar{u}^{t+1}-u^\star \right\rangle\notag \\
		&=&\frac{1}{1+\omega}\sqnorm{\bar{u}^{t+1} -u^\star }+\frac{\omega}{1+\omega}\sqnorm{u^{t} -u^\star }.\label{ubound33}
	\end{eqnarray}
	Where
	\begin{eqnarray}
		\sqnorm{\bar{u}^{t+1}-u^\star}&=&\sqnorm{(u^t-u^\star)+(\bar{u}^{t+1}-u^t)}\notag\\
		&=&\sqnorm{u^t-u^\star}+\sqnorm{\bar{u}^{t+1}-u^t}+2 \left\langle  u^t-u^\star,\bar{u}^{t+1}-u^t \right\rangle \notag\\
		&=&\sqnorm{u^t-u^\star}+2  \left\langle  \bar{u}^{t+1}-u^\star,\bar{u}^{t+1}-u^t \right\rangle  - \sqnorm{\bar{u}^{t+1}-u^t}.\label{ueq}
	\end{eqnarray}
	Combining \myref{ubound33} and \myref{ueq}, we get
	\begin{equation}
		\Exp{\sqnorm{u^{t+1}-u^\star}\;|\;\mathcal{F}_t}\leq \sqnorm{u^{t} -u^\star }+\frac{1}{1+\omega}\left(2  \left\langle  \bar{u}^{t+1}-u^\star,\bar{u}^{t+1}-u^t \right\rangle  - \sqnorm{\bar{u}^{t+1}-u^t}\right).\label{ubound35}
	\end{equation}
	Now let $c>0$ and combine \myref{xbnd35} with \myref{ubound35} to get
	\begin{eqnarray*}
		c\Exp{\sqnorm{x^{t+1}-x^\star}\;|\;\mathcal{F}_t}&+&\Exp{\sqnorm{u^{t+1}-u^\star}\;|\;\mathcal{F}_t}\\
		&\leq&  \frac{c}{1+\gammaM\mu}\sqnorm{x^t-x^\star}+c\gammaM^2(1-\zeta)\sqnorm{\Koper^\top(\bar{u}^{t+1}-u^t)}\\
		&\quad& -2c\gammaM  \left\langle \Koper(\hat{x}^{t}-x^\star),\bar{u}^{t+1}-u^\star \right\rangle+c\gammaM^2\oma\sqnorm{\bar{u}^{t+1}-u^t}\\
		&\quad&+ \sqnorm{u^{t} -u^\star }+\frac{1}{1+\omega}\left(2  \left\langle  \bar{u}^{t+1}-u^\star,\bar{u}^{t+1}-u^t \right\rangle  - \sqnorm{\bar{u}^{t+1}-u^t}\right)\\
		&\overset{(\ref{yi1})}{\leq}& \frac{c}{1+\gammaM\mu}\sqnorm{x^t-x^\star}+c\gammaM^2(1-\zeta)\sqnorm{\Koper^\top(\bar{u}^{t+1}-u^t)}\\
		&\quad& -\frac{2c\gammaM}{L_F}  \sqnorm{\bar{u}^{t+1}-u^\star }+c\gammaM^2\oma\sqnorm{\bar{u}^{t+1}-u^t}\\
		&\quad&+ \sqnorm{u^{t} -u^\star }+\frac{1}{1+\omega}\left(a\sqnorm{ \bar{u}^{t+1}-u^\star}+\frac{1}{a}\sqnorm{\bar{u}^{t+1}-u^t} - \sqnorm{\bar{u}^{t+1}-u^t}\right)\\
		&\leq &\frac{c}{1+\gammaM\mu}\sqnorm{x^t-x^\star}+ \sqnorm{u^{t} -u^\star }\\
		&\quad& +\left(c\gammaM^2\left(1-\zeta\right)\Mx+c\gammaM^2\oma+\frac{1}{1+\omega}\left(\frac{1}{a}-1\right)\right)\sqnorm{\bar{u}^{t+1}-u^t}\\
		&\quad&-\left(\frac{2c\gammaM}{L_F}-\frac{1}{1+\omega}a\right)  \sqnorm{\bar{u}^{t+1}-u^\star }.
	\end{eqnarray*}
	Using \myref{ubound33} and assuming $a$,$c$ and $\gammaM$ can be chosen so that $\frac{2c\gammaM}{L_F}-\frac{1}{1+\omega}a\geq 0 $
	\begin{eqnarray*}
		c\Exp{\sqnorm{x^{t+1}-x^\star}\;|\;\mathcal{F}_t}&+&\left(1+\left(1+\omega\right)\frac{2c\gammaM}{L_F}-a\right)\Exp{\sqnorm{u^{t+1}-u^\star}\;|\;\mathcal{F}_t}\\
		&\leq &\frac{c}{1+\gammaM\mu}\sqnorm{x^t-x^\star}+\left(1+\omega\left(\frac{2c\gammaM}{L_F}-\frac{1}{1+\omega}a\right)\right) \sqnorm{u^{t} -u^\star }\\
		&\quad &+\left(c\gammaM^2\left(1-\zeta\right)\Mx+c\gammaM^2\oma+\frac{1}{1+\omega}\left(\frac{1}{a}-1\right)\right)\sqnorm{\bar{u}^{t+1}-u^t}.
	\end{eqnarray*}
	In our case we have
	
	\begin{equation}
		\omega=\frac{\Mx}{\Cx}-1,\quad \oma=\frac{\Mx(\Mx-\Cx)}{\Cx(\Mx-1)},\quad \zeta=\frac{\Mx-\Cx}{\Cx(\Mx-1)},
	\end{equation}

	the term next to $\sqnorm{\hat{u}^{t+1}-u^t}$ becomes
	\begin{equation*}
		c\gammaM^2\Mx+\frac{\Cx}{\Mx}\left(\frac{1}{a}-1\right),
	\end{equation*}
	to get rid of it, we set
	\begin{equation*}
		c=\frac{\frac{\Cx}{\Mx}\left(1-\frac{1}{a}\right)}{\gammaM^2\Mx}\quad ,\quad a\geq1.
	\end{equation*}
	An $a$ that maximizes the contration on $\Exp{\sqnorm{u^{t+1}-u^\star}\;|\;\mathcal{F}_t}$ is given by $a=\sqrt{\frac{2}{\gammaM \Mx L_F}}$, thus we need $\gammaM\leq\frac{2}{\Mx L_F}$ and
	\begin{equation*}
		\frac{1}{\gammaM L_F \Mx}-\sqrt{\frac{2}{L_F\gammaM \Mx}}>0.
	\end{equation*} 
	Thus we need $\gammaM<\frac{1}{2\Mx L_F}$
	and we can write a contraction constant of Lyapunov function as
	\begin{eqnarray*}
		&\max&\left\{\frac{1}{1+\gammaM\mu},\frac{1+\omega(\frac{2c\gammaM}{L_F}-\frac{1}{1+\omega}a)}{1+(1+\omega)(\frac{2c\gammaM}{L_F}-\frac{1}{1+\omega}a)}\right\}=\max\left\{\frac{1}{1+\gammaM\mu},\frac{1+\frac{\Mx-\Cx}{\Cx}\left(\frac{2\Cx}{\gammaM L_F \Mx^2}-\frac{2\Cx}{\Mx}\sqrt{\frac{2}{L_F\gammaM \Mx}}\right)}{1+\frac{\Mx}{\Cx}\left(\frac{2\Cx}{\gammaM L_F \Mx^2}-\frac{2\Cx}{\Mx}\sqrt{\frac{2}{L_F\gammaM \Mx}}\right)}\right\}.
	\end{eqnarray*}
\end{proof}
\subsection{Proof of Corollary~\ref{cor:5GCS-0}}

\begin{corollary}
	Choose any $0<\varepsilon<1$ and
	$\gammaM=\frac{\Cx}{4 L \Mx}$. In order to guarantee $\Exp{\Psi^{\Tx}}\leq \varepsilon \Psi^0$, it suffices to take
	\[
		\Tx \geq 	\max\left\{1+\frac{4\Mx}{\Cx}\frac{L}{\mu},\frac{\Mx}{\Cx}+\frac{L_F\Mx}{L}\right\}\log \frac{1}{\varepsilon}= \tilde{O}\left(\frac{\Mx}{\Cx}\frac{L}{\mu}\right)
	\]
	communication rounds.
\end{corollary}
\begin{proof}
	If we let $\gammaM=\frac{\Cx}{4 L_F \Mx^2}\theta$, where $\theta=\frac{\Mx L_F}{L}$ then
	\begin{eqnarray*}
		&\max&\left\{\frac{1}{1+\gammaM\mu},\frac{1+\frac{\Mx-\Cx}{\Cx}\left(\frac{2\Cx}{\gammaM L_F \Mx^2}-\frac{2\Cx}{\Mx}\sqrt{\frac{2}{L_F\gammaM \Mx}}\right)}{1+\frac{\Mx}{\Cx}\left(\frac{2\Cx}{\gammaM L_F \Mx^2}-\frac{2\Cx}{\Mx}\sqrt{\frac{2}{L_F\gammaM \Mx}}\right)}\right\}
		=\max\left\{\frac{1}{1+\frac{\Cx}{4 L_F \Mx^2}\mu},\frac{1+\frac{\Mx-\Cx}{\Cx}\left(8\frac{1}{\theta}-8\sqrt{\frac{\Cx}{2\Mx\theta}}\right)}{1+\frac{\Mx}{\Cx}\left(8\frac{1}{\theta}-8\sqrt{\frac{\Cx}{2\Mx\theta}}\right)}\right\}\\
		&\leq&\max\left\{\frac{1}{1+\frac{\Cx}{4 L_F \Mx^2}\mu},1-\frac{8-8\sqrt{\frac{1}{2}}}{\theta+\frac{\Mx}{\Cx}\left(8-8\sqrt{\frac{1}{2}}\right)}\right\}
		\leq \max\left\{\frac{1}{1+\frac{\mu \Cx}{4 L\Mx}},1-\frac{2\Cx}{2\Mx+2\Cx\frac{\Mx L_F}{L}}\right\} 
		\leq\frac{1}{1+\mathcal{O}\left(\frac{\mu \Cx}{L\Mx}\right)}.
	\end{eqnarray*}
	Thus Algorithm~\ref{alg:5GCS-0} finds $\varepsilon$-solution in 
	$$T = \mathcal{O}\left(\frac{\Mx L}{\Cx\mu}\log\frac{1}{\varepsilon}\right)$$
	iterations.
\end{proof}

\clearpage
\section{Analysis of 5GCS for arbitary solvers $\cA_m$}

In real-life applications we might be in a situation where we would want local solvers to be personalized to each client, one such reason might be the amount of data or the type of software on a local machine. Thanks to the structure of the lifted space, the inner problem is separable which allows us to use arbitrary solvers to minimize the local function.
The general local problem 
\begin{equation*}
	\argmin \limits_{y\in\mathbb{R}^{dn}}\left\{\localfun(y) \eqdef F(y)+\frac{\tauM}{2} \sqnorm{y-\left(\Koper\hat{x}^\kstep+\frac{1}{\tauM}u^t\right)}\right\},
\end{equation*}
can be separated into
\begin{eqnarray*}
	\argmin \limits_{y\in\mathbb{R}^d}\left\{\localfuni(y) \eqdef  F_\iclient(y)+\frac{\tauM}{2} \sqnorm{y-\left(\hat{x}^\kstep+\frac{1}{\tauM}u_\iclient^t\right)}\right\},
\end{eqnarray*}
for $\iclient\in\left\{1,\dots,\Mx\right\}$ as the vector components are independent. This means that the Algorithm $\localsolver$ can be interpreted as concatenation of solutions that Algorithms $\localsolver_m$ find to respective local problems $\localfuni$. Noting that Assumption~\ref{ass:TPS} implies Assumption~\ref{ass:GTPS}, we can note that since local problems are independent there is no constraint on what local solver each client uses nor on a shared number of local steps that each method uses. 

\subsection{Proof of Theorem~\ref{thm:INEXACTPPanyM}}
\begin{theorem}
	Consider Algorithm~\ref{alg:5GCS} (\algname{5GCS}) with the LT solvers $\{\cA_1,\dots,\cA_\Mx\}$ satisfying Assumption~\ref{ass:GTPS}. Let $0<\gammaM\leq \frac{3}{16}\sqrt{\frac{\Cx}{L\mu \Mx}}$ and $\tauM=\frac{1}{2\gammaM \Mx}$. 
	Then for the Lyapunov function
	\begin{equation*}
			\Psi^{\kstep}\eqdef \frac{1}{\gammaM}\sqnorm{x^{\kstep}-x^\star}+\frac{\Mx}{\Cx}\left(\frac{1}{\tauM}+\frac{1}{L_F}\right)\sqnorm{u^{\kstep}-u^\star}\label{PPINEXACTpsianyM},
	\end{equation*}
	the iterates of the method satisfy
	$
	\Exp{\Psi^{\Tx}}\leq (1-\rho)^\Tx \Psi^0,
	$
	where 
	$	\rho\eqdef  \max\left\{\frac{\gammaM\mu}{1+\gammaM\mu},\frac{\Cx}{\Mx}\frac{\tauM}{(L_F+\tauM)}\right\}<1.
	$

\end{theorem}

\begin{proof}
	Noting that updates for $u^{t+1}$ and $x^{t+1}$ can be written as
	\begin{eqnarray}
		&u^{t+1}\eqdef u^t + \frac{1}{1+\omega} \Rop^t\left(\bar{u}^{t+1}-u^t\right),&\\
		&x^{t+1}=\hat{x}^t-\gammaM\left(\omega+1\right)\Koper^\top\left(u^{t+1}-u^t\right),&\label{xupdate35}
	\end{eqnarray}
	where $\Rop^t$ is the client sampling operator, $\omega=\frac{\Mx}{\Cx}-1$ and  $\bar{u}^{t+1} = \nabla F(\lastlocitterk)$. 
	Then using variance decomposition and Proposition~1 from \citep{condat2021murana}, we obtain
	\begin{eqnarray}
		\label{eq:starting_eq3.5}
		\Exp{\sqnorm{x^{t+1}-x^\star}\;|\;\mathcal{F}_t}&\overset{(\ref{vardec})}{=}&\sqnorm{\Exp{x^{t+1}\;|\;\mathcal{F}_t}-x^\star}+\Exp{\sqnorm{x^{t+1}-\Exp{x^{t+1}\;|\;\mathcal{F}_t}}\;|\;\mathcal{F}_t}\notag\\
		&\overset{(\ref{xupdate35})+(\ref{PP})}{=}& \underbrace{\sqnorm{\hat{x}^{t}-x^\star-\gammaM \Koper^\top(\bar{u}^{t+1}-u^t)}}_{X}+\gammaM^2\oma\sqnorm{\bar{u}^{t+1}-u^t}\notag\\
		&\quad &- \gammaM^2\zeta\sqnorm{\Koper^\top(\bar{u}^{t+1}-u^t)},
	\end{eqnarray}
	where
	\begin{equation*}
		\oma=\frac{\Mx(\Mx-\Cx)}{\Cx(\Mx-1)},\quad \zeta=\frac{\Mx-\Cx}{\Cx(\Mx-1)}.
	\end{equation*}
	Moreover, using \myref{fooc} and the definition of $\hat{x}^t$, we have
	\begin{align}
		&(1+\gammaM\mu)\hat{x}^t=x^t-\gammaM \Koper^\top u^{t},\label{opt1}\\
		&(1+\gammaM\mu)x^\star= x^\star -\gammaM \Koper^\top u^\star.\label{opt2}
	\end{align}
	Using \myref{opt1} and \myref{opt2} we obtain
	\begin{eqnarray}
		\label{eq:long3.5}
		X &=&	\sqnorm{\hat{x}^{t}-x^\star-\gammaM \Koper^\top(\bar{u}^{t+1}-u^t)}\notag\\
		&=&\sqnorm{\hat{x}^{t}-x^\star} +\gammaM^2\sqnorm{\Koper^\top (\bar{u}^{t+1}-u^t)}\notag\\
		&\quad&-2\gammaM \left\langle 
		\hat{x}^{t}-x^\star,\Koper^\top(\bar{u}^{t+1}-u^t) \right\rangle  \notag\\
		&=& (1+\gammaM\mu) \sqnorm{\hat{x}^{t}-x^\star} +\gammaM^2\sqnorm{\Koper^\top(\bar{u}^{t+1}-u^t)}\notag\\
		&\quad&-2\gammaM  \left\langle 
		\hat{x}^{t}-x^\star,\Koper^\top (\bar{u}^{t+1}-u^\star) \right\rangle  +2\gammaM  \left\langle 
		\hat{x}^{t}-x^\star,\Koper^\top (u^{t}-u^\star) \right\rangle  \notag\\
		&\quad&-\gammaM\mu\sqnorm{\hat{x}^{t}-x^\star} \notag\\
		&\overset{(\ref{opt1})+(\ref{opt2})}{=}&   \left\langle  x^t-x^\star-\gammaM \Koper^\top (u^t-u^\star),\hat{x}^{t}-x^\star \right\rangle +\gammaM^2\sqnorm{\Koper^\top (\bar{u}^{t+1}-u^t)}\notag\\
		&\quad&-2\gammaM  \left\langle 
		\hat{x}^{t}-x^\star,\Koper^\top (\bar{u}^{t+1}-u^\star) \right\rangle  +  \left\langle 
		\hat{x}^{t}-x^\star,2\gammaM \Koper^\top (u^{t}-u^\star) \right\rangle  \notag\\
		&\quad&-\gammaM\mu\sqnorm{\hat{x}^{t}-x^\star} \notag.
	\end{eqnarray}
	It leads to	
	\begin{eqnarray}
		X	&= &\left\langle  x^t-x^\star+\gammaM \Koper^\top (u^t-u^\star),\hat{x}^{t}-x^\star \right\rangle  \notag\\
		&\quad&+\gammaM^2\sqnorm{\Koper^\top (\bar{u}^{t+1}-u^t)}-2\gammaM  \left\langle 
		\hat{x}^{t}-x^\star, \Koper^\top(\bar{u}^{t+1}-u^\star) \right\rangle  \notag\\
		&\quad&-\gammaM\mu\sqnorm{\hat{x}^{t}-x^\star}\notag \\
		&\overset{(\ref{opt1})+(\ref{opt2})}{=}&\frac{1}{1+\gammaM\mu} \left\langle  x^t-x^\star+\gammaM \Koper^\top (u^t-u^\star),x^t-x^\star-\gammaM \Koper^\top (u^t-u^\star) \right\rangle  \notag\\
		&\quad&+\gammaM^2\sqnorm{\Koper^\top (\bar{u}^{t+1}-u^t)}-2\gammaM  \left\langle 
		\hat{x}^{t}-x^\star, \Koper^\top (\bar{u}^{t+1}-u^\star) \right\rangle  \notag\\
		&\quad&-\gammaM\mu\sqnorm{\hat{x}^{t}-x^\star} \notag\\
		&=& \frac{1}{1+\gammaM\mu}\sqnorm{x^t-x^\star}-\frac{\gammaM^2}{1+\gammaM\mu}\sqnorm{\Koper^\top (u^t-u^\star)}\notag\\
		&\quad&+\gammaM^2\sqnorm{\Koper^\top (\bar{u}^{t+1}-u^t)}-2\gammaM  \left\langle 
		\hat{x}^{t}-x^\star,\Koper^\top (\bar{u}^{t+1}-u^\star) \right\rangle\notag\\
		&\quad&-\gammaM\mu\sqnorm{\hat{x}^{t}-x^\star}   . 
	\end{eqnarray}
	Combining \myref{eq:starting_eq3.5} and \myref{eq:long3.5} we have
	\begin{eqnarray*}
		\Exp{\sqnorm{x^{t+1}-x^\star}\;|\;\mathcal{F}_t}&\leq&  \frac{1}{1+\gammaM\mu}\sqnorm{x^t-x^\star}-\frac{\gammaM^2}{1+\gammaM\mu}\sqnorm{\Koper^\top (u^t-u^\star)}\\
		&\quad&+\gammaM^2(1-\zeta)\sqnorm{\Koper^\top (\bar{u}^{t+1}-u^t)}-2\gammaM  \left\langle 
		\hat{x}^{t}-x^\star,\Koper^\top (\bar{u}^{t+1}-u^\star) \right\rangle \\
		&\quad&+\gammaM^2\oma\sqnorm{\bar{u}^{t+1}-u^t}-\frac{\gammaM\mu}{\Mx}\sqnorm{\Koper\hat{x}^t-\Koper x^\star}.
	\end{eqnarray*}
	Note that we can have the update rule for $u$ as: 
	\begin{equation*}
		u^{t+1}\eqdef u^t + \tfrac{1}{1+\omega} \Rop^t\left(\bar{u}^{t+1}-u^t\right),
	\end{equation*}
	where $\Rop^t$ is the client sampling operator with parameter  $\omega=\frac{\Mx}{\Cx}-1$. Using conic variance formula \myref{cvar} of $\Rop^t$ we obtain
	\begin{eqnarray}
		\Exp{\sqnorm{u^{t+1}-u^\star}\;|\;\mathcal{F}_t}&\overset{(\ref{vardec})+(\ref{cvar})}{\leq}& \sqnorm{u^{t}-u^\star+\frac{1}{1+\omega}\left(\bar{u}^{t+1} -u^t\right)}
		+\frac{\omega}{(1+\omega)^2}\sqnorm{\bar{u}^{t+1} -u^t }\notag\\
		&=&\frac{\omega^2}{(1+\omega)^2}\sqnorm{u^{t}-u^\star}+\frac{1}{(1+\omega)^2}\sqnorm{\bar{u}^{t+1}-u^\star}\notag\\
		&\quad&+\frac{2\omega}{(1+\omega)^2} \left\langle  u^{t}-u^\star,
		\bar{u}^{t+1}-u^\star \right\rangle +\frac{\omega}{(1+\omega)^2}\sqnorm{\bar{u}^{t+1} -u^\star }\notag\\
		&\quad&+\frac{\omega}{(1+\omega)^2}\sqnorm{u^{t} -u^\star }-\frac{2\omega}{(1+\omega)^2} \left\langle  u^{t}-u^\star,
		\bar{u}^{t+1}-u^\star \right\rangle \notag\\
		&=&\frac{1}{1+\omega}\sqnorm{\bar{u}^{t+1} -u^\star }+\frac{\omega}{1+\omega}\sqnorm{u^{t} -u^\star }.\label{expu}
	\end{eqnarray}
	Let us consider the first term in \myref{expu}:
	\begin{eqnarray*}
		\sqnorm{\bar{u}^{t+1}-u^\star}&=&\sqnorm{(u^t-u^\star)+(\bar{u}^{t+1}-u^t)}\\
		&=&\sqnorm{u^t-u^\star}+\sqnorm{\bar{u}^{t+1}-u^t}+2 \left\langle  u^t-u^\star,\bar{u}^{t+1}-u^t \right\rangle \\
		&=&\sqnorm{u^t-u^\star}+2  \left\langle  \bar{u}^{t+1}-u^\star,\bar{u}^{t+1}-u^t \right\rangle  - \sqnorm{\bar{u}^{t+1}-u^t}.
	\end{eqnarray*}
	Combining terms together we get
	\begin{equation*}
		\Exp{\sqnorm{u^{t+1}-u^\star}\;|\;\mathcal{F}_t}\leq \sqnorm{u^{t} -u^\star }+\frac{1}{1+\omega}\left(2  \left\langle  \bar{u}^{t+1}-u^\star,\bar{u}^{t+1}-u^t \right\rangle  - \sqnorm{\bar{u}^{t+1}-u^t}\right).
	\end{equation*}
	Finally, we obtain
	\begin{eqnarray*}
		\frac{1}{\gammaM}\Exp{\sqnorm{x^{t+1}-x^\star}\;|\;\mathcal{F}_t}&+&\frac{1+\omega}{\tauM}\Exp{\sqnorm{u^{t+1}-u^\star}\;|\;\mathcal{F}_t}\\
		&\leq&  \frac{1}{\gammaM(1+\gammaM\mu)}\sqnorm{x^t-x^\star}-\frac{\gammaM}{1+\gammaM\mu}\sqnorm{\Koper^\top (u^t-u^\star)}\\
		&\quad&+\gammaM(1-\zeta)\sqnorm{\Koper^\top (\bar{u}^{t+1}-u^t)}\\
		&\quad&+\gammaM\oma\sqnorm{\bar{u}^{t+1}-u^t}-\frac{\mu}{\Mx}\sqnorm{\Koper \hat{x}^t- \Koper x^\star}\\
		&\quad&+\frac{1+\omega}{\tauM}\sqnorm{u^{t} -u^\star }-2  \left\langle 
		\hat{x}^{t}-x^\star,\Koper^\top (\bar{u}^{t+1}-u^\star) \right\rangle \\ &\quad&+\frac{1}{\tauM}\left(2  \left\langle  \bar{u}^{t+1}-u^\star,\bar{u}^{t+1}-u^t \right\rangle  - \sqnorm{\bar{u}^{t+1}-u^t}\right).
	\end{eqnarray*}
	Ignoring $-\frac{\gammaM}{1+\gammaM\mu}\sqnorm{\Koper^\top (u^t-u^\star)}$ and noting
	\begin{eqnarray*}
		-\left\langle \hat{x}^{t}-x^\star,\Koper^\top (\bar{u}^{t+1}-u^\star) \right\rangle &+&\frac{1}{\tauM}\left\langle  \bar{u}^{t+1}-u^\star,\bar{u}^{t+1}-u^t \right\rangle  \\
		&=&-\left\langle \lastlocitterk-\Koper x^\star,\bar{u}^{t+1}-u^\star \right\rangle +\frac{1}{\tauM}\left\langle  \nabla\localfun (\lastlocitterk),\bar{u}^{t+1}-u^\star \right\rangle\\
		&\overset{(\ref{yi1})+(\ref{strmono})}{\leq}& -\frac{1}{L_F}\sqnorm{\bar{u}^{t+1}-u^\star}+\frac{a}{2\tauM}\sqnorm{\nabla\localfun (\lastlocitterk)}+\frac{1}{2a\tauM}\sqnorm{\bar{u}^{t+1}-u^\star }\\
		&=& -\left(\frac{1}{L_F}-\frac{1}{2a\tauM}\right)\sqnorm{\bar{u}^{t+1}-u^\star}+\frac{a}{2\tauM}\sqnorm{\nabla\localfun (\lastlocitterk)}\\
		&\overset{(\ref{expu})}{\leq}& -\left(\frac{1}{L_F}-\frac{1}{2a\tauM}\right)\left((1+\omega)\Exp{\sqnorm{u^{t+1}-u^\star}\;|\;\mathcal{F}_t}-\omega\sqnorm{u^t-u^\star}\right)\\
		&\quad&+\frac{a}{2\tauM}\sqnorm{\nabla\localfun (\lastlocitterk)},
	\end{eqnarray*}
	we get
	\begin{eqnarray*}
		\frac{1}{\gammaM}\Exp{\sqnorm{x^{t+1}-x^\star}\;|\;\mathcal{F}_t}&+&\left(1+\omega\right)\left(\frac{1}{\tauM}+\frac{1}{L_F}\right)\Exp{\sqnorm{u^{t+1}-u^\star}\;|\;\mathcal{F}_t}\\
		&\leq&  \frac{1}{\gammaM(1+\gammaM\mu)}\sqnorm{x^t-x^\star}\\
		&\quad&+\left(1+\omega\right)\left(\frac{1}{\tauM}+\frac{\omega}{1+\omega}\frac{1}{L_F}\right)\sqnorm{u^{t} -u^\star } \\
		&\quad&+\left(\gammaM\left(1-\zeta\right)\Mx+\gammaM\oma-\frac{1}{\tauM}\right)\sqnorm{\bar{u}^{t+1}-u^t}\\
		&\quad& +\frac{L_F}{\tauM^2}\sqnorm{\nabla\localfun (\lastlocitterk)}-\frac{\mu}{\Mx}\sqnorm{\Koper\hat{x}^t-\Koper x^\star}.
	\end{eqnarray*}
	Where we made the choice $a=\frac{L_F}{\tau}$.
	Using Young's inequality we have
	\begin{equation*}
		-\frac{\mu}{3\Mx}\sqnorm{\Koper \hat{x}^t-\localsolk+\localsolk-\Koper x^\star}\overset{(\ref{yi3})}{\leq}\frac{\mu}{3\Mx}\sqnorm{\localsolk-\Koper x^\star}-\frac{\mu}{6\Mx}\sqnorm{\Koper\hat{x}^t-\localsolk}.	\end{equation*}
	Noting the fact that $\localsolk=\Koper\hat{x}^t-\frac{1}{\tauM}(\hat{u}^{t+1}-u^t)$, we have
	$$\frac{\mu}{3\Mx}\sqnorm{\localsolk-\Koper x^\star}\overset{(\ref{yi2})}{\leq} 2\frac{\mu}{3\Mx}\sqnorm{\Koper\hat{x}^t-\Koper x^\star}+\frac{2}{\tauM^2}\frac{\mu}{3\Mx}\sqnorm{\hat{u}^{t+1}-u^t}.$$
	Combining those inequalities we get
	\begin{eqnarray*}
		\frac{1}{\gammaM}\Exp{\sqnorm{x^{t+1}-x^\star}\;|\;\mathcal{F}_t}&+&\left(1+\omega\right)\left(\frac{1}{\tauM}+\frac{1}{L_F}\right)\Exp{\sqnorm{u^{t+1}-u^\star}\;|\;\mathcal{F}_t}\\
		&\leq&  \frac{1}{\gammaM(1+\gammaM\mu)}\sqnorm{x^t-x^\star}\\
		&\quad&+\left(1+\omega\right)\left(\frac{1}{\tauM}+\frac{\omega}{1+\omega}\frac{1}{L_F}\right)\sqnorm{u^{t} -u^\star } \\
		&\quad&+\frac{2}{\tauM^2}\frac{\mu}{3\Mx}\sqnorm{\hat{u}^{t+1}-u^t}\\
		&\quad&-\left(\frac{1}{\tauM}-\left(\gammaM\left(1-\zeta\right)\Mx+\gammaM\oma\right)\right)\sqnorm{\bar{u}^{t+1}-u^t}\\
		&\quad&+\frac{L_F}{\tauM^2}\sqnorm{\nabla\localfun (\lastlocitterk)}-\frac{\mu}{6\Mx}\sqnorm{\Koper\hat{x}^t-\localsolk}.
	\end{eqnarray*}
	Assuming $\gammaM$ and $\tauM$ can be chosen so that $\frac{1}{\tauM}-(\gammaM(1-\zeta)\Mx+\gammaM\oma))\geq \frac{4}{\tauM^2}\frac{\mu}{3\Mx}$ we obtain
	\begin{eqnarray*}
		\frac{1}{\gammaM}\Exp{\sqnorm{x^{t+1}-x^\star}\;|\;\mathcal{F}_t}&+&\left(1+\omega\right)\left(\frac{1}{\tauM}+\frac{1}{L_F}\right)\Exp{\sqnorm{u^{t+1}-u^\star}\;|\;\mathcal{F}_t}\\
		&\leq&  \frac{1}{\gammaM(1+\gammaM\mu)}\sqnorm{x^t-x^\star}\\
		&\quad&+\left(1+\omega\right)\left(\frac{1}{\tauM}+\frac{\omega}{1+\omega}\frac{1}{L_F}\right)\sqnorm{u^{t} -u^\star } \\
		&\quad&+\frac{4}{\tauM^2}\frac{\mu L_F^2}{3\Mx}\sqnorm{\lastlocitterk-\localsolk}+\frac{L_F}{\tauM^2}\sqnorm{\nabla\localfun (\lastlocitterk)}\\
		&\quad&-\frac{\mu}{6\Mx}\sqnorm{\Koper\hat{x}^t-\localsolk}.
	\end{eqnarray*}
Using Assumption~\ref{ass:GTPS}, we have 
	\begin{align*}
	\sum\limits_{\iclient=1}^{\Mx}\frac{4}{\tauM^2}\frac{\mu L_F^2}{3\Mx}\sqnorm{\lastlocittermk-\localsolmk} +\sum\limits_{\iclient=1}^{\Mx}\frac{L_F}{\tauM^2}\sqnorm{\nabla\localfuni(\lastlocittermk)} \leq\sum\limits_{\iclient=1}^{\Mx}\frac{\mu}{6\Mx}\sqnorm{\hat{x}^\kstep-\localsolmk},
\end{align*}
This is enough to have similar bound in lifted space for the point $\lastlocitterk$:
	\begin{eqnarray*}
		\frac{4}{\tauM^2}\frac{\mu L_F^2}{3\Mx}\sqnorm{\lastlocitterk-\localsolk}+\frac{L_F}{\tauM^2}\sqnorm{\nabla\localfun (\lastlocitterk)}\leq\frac{\mu}{6\Mx}\sqnorm{\Koper\hat{x}^t-\localsolk}.
	\end{eqnarray*}
	Thus
	\begin{eqnarray*}
		\frac{1}{\gammaM}\Exp{\sqnorm{x^{t+1}-x^\star}\;|\;\mathcal{F}_t}&+&\left(1+\omega\right)\left(\frac{1}{\tauM}+\frac{1}{L_F}\right)\Exp{\sqnorm{u^{t+1}-u^\star}\;|\;\mathcal{F}_t}\\
		&\leq&  \frac{1}{\gammaM(1+\gammaM\mu)}\sqnorm{x^t-x^\star}\\
		&\quad&+\left(1+\omega\right)\left(\frac{1}{\tauM}+\frac{\omega}{1+\omega}\frac{1}{L_F}\right)\sqnorm{u^{t} -u^\star }.
	\end{eqnarray*}
	
	By taking the expectation on both sides we get
	\begin{equation*}
		\Exp{\Psi^{t+1}}\leq  \max\left\{\frac{1}{1+\gammaM\mu} ,\frac{ L_F+\frac{\Mx-\Cx}{\Cx}\tauM}{L_F+\tauM} \right\}\Exp{\Psi^{t}},
	\end{equation*}
	which finishes the proof. 
	Note that our standard choice of constants is $$\omega=\frac{\Mx}{\Cx}-1,\quad \oma=\frac{\Mx(\Mx-\Cx)}{\Cx(\Mx-1)},\quad \zeta=\frac{\Mx-\Cx}{\Cx(\Mx-1)}.$$ Using these parameters the requirement for stepsizes becomes:$$\frac{1}{\tauM}-\gammaM \Mx\geq \frac{4\mu}{3\Mx\tauM^2}.$$
	This inequality is satisfied, when $0<\gammaM\leq \frac{3}{16}\sqrt{\frac{\Cx}{L\mu \Mx}}$ and $\tauM=\frac{1}{2\Mx\gammaM}.$
\end{proof}

\subsection{Reallocation of resources}

\begin{assumption}\label{ass:TPS}
	Let $\mathcal{A}_\iclient$ be an Algorithm that can find a point $\lastlocittermk$ after $\Kx$ local steps applied to the local function $\localfuni$ from \myref{localprob} and starting point $y_{\iclient}^{0,\kstep}=\hat{x}^\kstep$, which satisfies
	\begin{eqnarray*}
		\frac{4}{\tauM^2}\frac{\mu L_F^2}{3\Mx}\sqnorm{\lastlocittermk-\localsolmk}	+ \frac{L_F}{\tauM^2}\sqnorm{\nabla\localfuni(\lastlocittermk)}
		\leq	\squeeze \frac{\mu}{6\Mx}\sqnorm{\hat{x}^\kstep-\localsolmk},
	\end{eqnarray*}
	where $\localsolmk$ is the unique minimizer of $\localfuni$, and $\tauM\geq \frac{8}{3}\sqrt{\frac{L\mu}{\Mx \Cx}}.$
\end{assumption}

The general local problem is
\begin{equation}
	\argmin \limits_{y\in\mathbb{R}^{dn}}\left\{\localfun(y) \eqdef  F(y)+\frac{\tauM}{2} \sqnorm{y-\left(\Koper\hat{x}^\kstep+\frac{1}{\tauM}u^t\right)}\right\}\label{locprobG},
\end{equation}
and the condition necessary for Theorem~\ref{thm:INEXACTPPanyM} is
\begin{align*}
	\frac{4}{\tauM^2}\frac{\mu L_F^2}{3\Mx}\sqnorm{\lastlocitterk-\localsolk}+\squeeze\frac{L_F}{\tauM^2}\sqnorm{\nabla\localfun (\lastlocitterk)}\leq\frac{\mu}{6\Mx}\sqnorm{\Koper\hat{x}^t-\localsolk}.
\end{align*}
This is actually a restriction in $\mathbb{R}^{dn}$ (a dual/lifted space), which can be equivalently written as
\begin{align*}
	\sum_{\iclient=1}^{\Mx}\frac{4}{\tauM^2}\frac{\mu L_F^2}{3\Mx}\sqnorm{\lastlocittermk-\localsolmk}+\sum_{\iclient=1}^{\Mx}\frac{L_F}{\tauM^2}\sqnorm{\nabla\localfuni(\lastlocittermk)}\leq\sum_{\iclient=1}^{\Mx}\frac{\mu}{6\Mx}\sqnorm{\hat{x}^\kstep-\localsolmk}.
\end{align*}

Assumption~\ref{ass:GTPS}, which is necessary to hold for Theorem~\ref{thm:INEXACTPPanyM} arises due to the definition of the lifted space. The strength of this condition is that it allows for provable convergence even in situations where some clients can not find the required by Assumption~\ref{ass:TPS} accuracy as long other clients compensate for it by doing more iterations.

\subsection{Number of local steps in LT subroutine of 5GCS}
\label{mylabel}
In this section, we would like to present different guarantees that various Algorithms $\localsolver$ can give us.
Algorithm~$\mathcal{A}$ is simply taking current iterates $\hat{x}^t$ and $u^t$ and applies Algorithms $\localsolver_m$ to the local problem \myref{localprob}  (at each clients)and finally concatenates the result in $\localsolk$. 
To guarantee convergence of Algorithm~\ref{alg:5GCS}, we need to do locally $\Kx$ iterations of Algorithm~$\localsolver$ which would guarantee:
\begin{eqnarray*}
	\frac{4}{\tauM^2}\frac{\mu L_F^2}{3\Mx}\sqnorm{\lastlocitterk-\localsolk}+\frac{a}{\tauM}\sqnorm{\nabla\localfun (\lastlocitterk)}	&\leq&\left(\frac{4\mu L_F^2}{3\Mx\tauM^2}+\frac{a(L_F+\tauM)^2}{\tauM}\right)\sqnorm{\lastlocitterk-\localsolk}\\
	&\leq&\frac{\mu}{6\Mx}\sqnorm{\Koper\hat{x}^t-\localsolk}.
\end{eqnarray*}
Thus, we need:
\begin{equation}
	\sqnorm{\lastlocitterk-\localsolk}
	\leq\betaM\sqnorm{\Koper\hat{x}^t-\localsolk}.\label{betasol}  
\end{equation}
Where $$\betaM=\frac{\frac{\mu}{6\Mx}}{\left(\frac{4\mu L_F^2}{3\Mx\tauM^2}+\frac{a(L_F+\tauM)^2}{\tauM}\right)}.$$
For $a=\frac{L_F}{\tauM}$, the term that will appear in most of those analysis is 
\begin{equation*}
	\frac{1}{\betaM}=\frac{\left(\frac{4\mu L_F^2}{3\Mx\tauM^2}+\frac{a(L_F+\tauM)^2}{\tauM}\right)}{\frac{\mu}{6\Mx}}\leq\frac{8L_F^2}{\tauM^2}+\frac{12 L_F^3\Mx}{\tauM^2\mu}+\frac{12 L_F}{\mu}.
\end{equation*}  
Note that $\tauM$ is smallest for the optimal choice of $\gammaM$, thus
\begin{equation*}
	\frac{1}{\betaM}\leq\frac{9 L_F^2\Cx\Mx}{8L\mu}+\frac{108 L_F^3\Mx^2\Cx}{64L\mu^2}+\frac{12 L_F \Mx}{\mu}\leq\left(4\frac{L}{\mu}\right)^2,
\end{equation*}  
where in the last inequality we used  bounds such as $\Mx\geq \Cx$, $L\geq \Mx L_F$ and $\frac{L}{\mu}\geq 1$.

\subsubsection{Gradient descent for local problem}
\algname{GD} with stepsize $\frac{1}{L_F+\tauM}$ would need:
\begin{equation*}
	\Kx\geq \left(\frac{L_F+\tauM}{\tauM}\right)\log\left(\frac{1}{\betaM}\right).
\end{equation*}
Again noting that $\tauM$ is smallest when we choose stepsizes optimally:
$$\frac{L_F+\tauM}{\tauM}\leq \frac{3}{8}\sqrt{\frac{L\Cx}{\mu \Mx}}+1.$$
Thus, if $\mathcal{A}$ is \algname{GD}, then:
\begin{equation*}
	\Kx\geq \left(\frac{3}{4}\sqrt{\frac{L\Cx}{\mu \Mx}}+2\right)\log\left(4\frac{L}{\mu}\right).
\end{equation*}

\subsection{Local speed up due to personalized condition number of each client}
Dependence of the local condition number on $\tauM$ and how can we use this dependence to control the speed of local convergence is described in Section~\ref{mylabel}. Here we would like to focus on the case where each function has a different smoothness parameter.
Suppose each $f_\iclient$ is $L_m$-smooth and $\mu$-convex. If we let $L=\max_\iclient L_\iclient$ then we can note that each $f_\iclient$ is $L$-smooth, thus we have that $L_F=\frac{1}{\nclients}\left(L-\mu\right)$ and we recover the whole communication result for our Algorithm. However, locally we can note that each clients needs to find $\betaM$-solution to the local problem \myref{localprob}, which is $\left(\frac{1}{\nclients}\left(L_\iclient-\mu\right)+\tauM\right)$-smooth and $\tauM$-convex.
Remembering $\tauM\geq\frac{8}{3}\sqrt{\frac{\mu L}{\nclients\kcohort}} $, \algname{GD} needs
\begin{eqnarray*}
	2\left(\frac{1}{\nclients}\left(L_\iclient-\mu\right)\frac{1}{\tauM}+1\right)\log\left(4\frac{L}{\mu}\right)&\leq&2\left(\frac{3}{8}\sqrt{\frac{L_\iclient\kcohort}{\mu \nclients }}+1\right)\log\left(4\frac{L}{\mu}\right),
\end{eqnarray*}
iterations. Which is better, then if we were using the upper bound $\max_\iclient L_\iclient$ on each $L_\iclient$. To illustrate this we can formulate the following Corollary~\ref{cor:personal} to the general Theorem~\ref{thm:5GCS}
\begin{corollary}\label{cor:personal}
	Consider Algorithm~\ref{alg:5GCS} with LT solver being GD. In the new personalized setting with $L=\max_{\iclient}L_{\iclient}$, we can run the LT for $\localsteps\geq 2\left(\frac{3}{8}\sqrt{\frac{L_\iclient\kcohort}{\mu \nclients }}+1\right)\log\left(4\frac{L}{\mu}\right)$ and still accomplish guarantees of Theorem \ref{thm:5GCS}.
\end{corollary}

\clearpage
\subsection{Local solvers $\localsolver_m$ may be stochastic}
Until now we assumed that Algorithms $\localsolver_m$ were deterministic(in a sense that they do not introduce any randomness to the system). However, with a small change in the analysis from Section~\ref{proof35}, we can allow for local solvers to be stochastic, we can present a more general condition which includes stochastic local solvers. To analyze the stochastic local solvers we need to modify Assumption~\ref{ass:GTPS} with respect to stochasticity. We introduce a new assumption, where the inequality appearing in Assumption~\ref{ass:GTPS} should be satisfied in expectation. 
\begin{assumption}\label{ass:STPS}
	Let $\mathcal{A}$ be stochastic Algorithm that can find a point $\lastlocitterk$ in $\Kx$ local steps applied to the local function $\localfun$ from \myref{localprob} and starting point $y_{\iclient}^{0,\kstep}=\hat{x}^\kstep$, which satisfies
	\begin{align*}
		\Exp{\sum_{\iclient=1}^{\Mx}\frac{4}{\tauM^2}\frac{\mu L_F^2}{3\Mx}\sqnorm{\lastlocittermk-\localsolmk}+\sum_{\iclient=1}^{\Mx}\frac{L_F}{\tauM^2}\sqnorm{\nabla\localfuni(\lastlocittermk)}\;|\;\mathcal{F}_t}\leq\sum_{\iclient=1}^{\Mx}\frac{\mu}{6\Mx}\sqnorm{\hat{x}^\kstep-\localsolmk},
	\end{align*}
	where $\localsolmk$ is the unique minimizer of $\localfuni$, and $\tauM\geq \frac{8}{3}\sqrt{\frac{L\mu}{\Mx \Cx}}.$
\end{assumption}
 The conditioning on $\mathcal{F}^t$ simply means that $\hat{x}^t$ is not treated as a random vector and the only randomness comes from the local . Let us consider $\Exp{X\;|\;A}$, which represents the expectation of a random variable $X$ condition on the randomness accumulated due to local solvers being stochastic. Then conditioning on both $A^t$ and $\mathcal{F}^t$, we can get
	\begin{eqnarray*}
		\frac{1}{\gammaM}\Exp{\sqnorm{x^{t+1}-x^\star}\;|\;\mathcal{F}_t\cup A^t}&+&\left(1+\omega\right)\left(\frac{1}{\tauM}+\frac{1}{L_F}\right)\Exp{\sqnorm{u^{t+1}-u^\star}\;|\;\mathcal{F}_t\cup A^t}\\
		&\leq&  \frac{1}{\gammaM(1+\gammaM\mu)}\sqnorm{x^t-x^\star}\\
		&\quad&+\left(1+\omega\right)\left(\frac{1}{\tauM}+\frac{\omega}{1+\omega}\frac{1}{L_F}\right)\sqnorm{u^{t} -u^\star } \\
		&\quad&+\frac{4}{\tauM^2}\frac{\mu L_F^2}{3\Mx}\sqnorm{\lastlocitterk-\localsolk}+\frac{L_F}{\tauM^2}\sqnorm{\nabla\localfun (\lastlocitterk)}\\
		&\quad&-\frac{\mu}{6\Mx}\sqnorm{\Koper\hat{x}^t-\localsolk}.
	\end{eqnarray*}
	Taking expectation condition on $\mathcal{F}^t$ on both sides we get
	\begin{eqnarray*}
		\frac{1}{\gammaM}\Exp{\sqnorm{x^{t+1}-x^\star}\;|\;\mathcal{F}_t}&+&\left(1+\omega\right)\left(\frac{1}{\tauM}+\frac{1}{L_F}\right)\Exp{\sqnorm{u^{t+1}-u^\star}\;|\;\mathcal{F}_t}\\
		&\leq&  \frac{1}{\gammaM(1+\gammaM\mu)}\sqnorm{x^t-x^\star}\\
		&\quad&+\left(1+\omega\right)\left(\frac{1}{\tauM}+\frac{\omega}{1+\omega}\frac{1}{L_F}\right)\sqnorm{u^{t} -u^\star } \\
		&\quad&+\Exp{\frac{4}{\tauM^2}\frac{\mu L_F^2}{3\Mx}\sqnorm{\lastlocitterk-\localsolk}+\frac{L_F}{\tauM^2}\sqnorm{\nabla\localfun (\lastlocitterk)}\;|\;\mathcal{F}_t}\\
		&\quad&-\frac{\mu}{6\Mx}\sqnorm{\Koper\hat{x}^t-\localsolk}.
	\end{eqnarray*}
Crucial practical benefit comes from the expected improvement in gradient calculation, when each local function has a finite sum structure, which is common in practice.

\subsubsection{L-SVRG for local problem}
\begin{algorithm}[H]
	\caption{\algn{L-SVRG}}
	\begin{algorithmic}[1]\label{alg:LSVRG}
		\STATE  \textbf{input:} initial points $x^0\in\mathbb{R}^d$, $y^0=x^0$, gradient estimator $g$; 
		\STATE stepsize $\gammaM>0$
		\FOR{$k=0, 1, \ldots$}
		\STATE$g^k=g(x^k)-g(y^k)+\nabla f(y^k)$
		\STATE$x^{k+1}=x^k-\gammaM g^k$
		\STATE$y^{k+1}=\begin{cases}
			x^k \quad \text{with probability $p$}\\
			y^k \quad \text{with probability $1-p$}
		\end{cases}$
		\ENDFOR
	\end{algorithmic}
\end{algorithm}
In this section we consider \algname{L-SVRG}\citep{kovalev2020don} method as local stochastic solver with variance reduction mechanism. Our analysis is based on general expected smoothness assumption \citep{gower2019sgd}.
\begin{assumption}
The gradient estimator $g$ is unbiased, and satisfies the expected smoothness bound\label{ass:expsmooth}
	\begin{align*}
		\Exp{g(x)} &= \nabla f(x),\\
		\Exp{\sqnorm{g(x)-g(x^\star)}}&\leq 2A''\mathcal{D}_f(x,x^\star).
	\end{align*}
\end{assumption}
We apply convergence guarantees of \algname{L-SVRG} for the subproblem in Algorithm~\ref{alg:5GCS} for a gradient estimator $g$ satisfying Assumption~\ref{ass:expsmooth} and stepsize $\gammaM_2=\frac{1}{6A''}$. We obtain the following bound:
\begin{equation*}
	\Exp{\sqnorm{\lastlocitterk-\localsolk}}\leq\left(1-\min\left\{\gammaM_2\tauM,\frac{p}{2}\right\}\right)^T\left(1+2\gammaM_2^2\frac{L_F+\tauM}{p}\right)\sqnorm{\Koper\hat{x}^t-\localsolk} .
\end{equation*}
This means that Algorithm~\ref{alg:LSVRG} with $p=2\tauM\gammaM_2$ finds $\betaM$-solution to the local problem of Algorithm~\ref{alg:5GCS} in
\begin{equation*}
K=	\frac{6A''}{\tauM}\log\left(\frac{\tauM+\left(L_F + \tauM\right)\gammaM_2}{\tauM}\frac{1}{\betaM}\right)
\end{equation*}
local steps. Particularly interesting and practical example of $g$ in the Algorithm~\ref{alg:LSVRG} is mini-batch gradient estimator.
Thus, we assume the finite sum structure:
\begin{eqnarray*}
	f_\iclient(x)=\frac{1}{n_\iclient}\sum_{i=1}^{n_\iclient}f_{\iclient,i}(x),
\end{eqnarray*}
where each $f_{\iclient,i}$ is convex and $L_{i}$ smooth. Than $\localfuni$ can be put in the finite sum structure, by writing
\begin{eqnarray*}
	\localfuni(y)=\frac{1}{n_\iclient}\sum_{i=1}^{n_\iclient}g_i(y),
\end{eqnarray*}
where
\begin{eqnarray*}
	g_i(y)=\frac{1}{\Mx}\left(f_{\iclient,i}(y)-\frac{\mu}{2}\sqnorm{y}\right)+\frac{\tauM}{2}\sqnorm{y-(\hat{x}^t+\frac{1}{\tauM}u_{\iclient}^t)}.
\end{eqnarray*}
Since $\tau\geq\frac{4\mu}{3\Mx}$, $g_{i}$ is $\left(\frac{1}{\Mx}\left(L_{i}-\mu\right)+\tauM\right)$-smooth and $\left(\tauM-\frac{\mu}{\Mx}\right)$-convex. 
Fix a mini-batch size $b_\iclient\in\{1,2,\dots,\Mx_\iclient\}$ and let $S_\iclient$ be a random subset of $ \{1,\ldots,\Mx_\iclient\}$ of size $\Cx$, chosen uniformly at random, then the mini-batch gradient estimator is 
\begin{eqnarray*}
	g(y)=\frac{1}{b_\iclient}\sum_{i\in S_\iclient} \nabla g_i(y).
\end{eqnarray*}
For this gradient estimator 
\begin{eqnarray*}
	A''=\frac{n_\iclient-b_\iclient}{b_\iclient\left(n_\iclient-1\right)}\underset{i}{\max}L_{g_i}+\frac{n_\iclient\left(b_\iclient-1\right)}{b_\iclient\left(n_\iclient-1\right)}\left(L_F+\tauM\right),
\end{eqnarray*}
where $L_{g_i}=\frac{1}{\Mx}\left(L_{i}-\mu\right)+\tau$.

\clearpage

\section{Relation between the \# of communication rounds $\Tx$ on the \# of local steps $\Kx$}
\subsection{Proof of Theorem~\ref{thm:relation2}}
\begin{theorem}
	Consider Algorithm~\ref{alg:5GCS} (\algname{5GCS}) with the LT solver being \algname{GD}. Let $\gammaM=\frac{3}{16L}$  and $\tauM=\frac{8L}{3\Mx}.$
	With such chosen stepsizes, it is enough to run \algname{GD} for
	\begin{eqnarray*}
		\squeeze		\Kx\geq \left(2+\frac{3\Mx L_F}{4L}\right)\log\left(4\frac{L}{\mu}\right)=\mathcal{O}\left(\log\frac{L}{\mu}\right).
	\end{eqnarray*}
	Whereas, the number of communication rounds to reach $\varepsilon$-solution is
	\begin{equation*}
		\squeeze	T\geq\max\left\{1+\frac{16}{3}\frac{L}{\mu}, \frac{\Mx}{\Cx}+\frac{3\Mx}{8\Cx}\frac{\Mx L_F}{L}\right\}\log\frac{1}{\epsilon}=\tilde{\mathcal{O}}\left(\frac{\Mx}{\Cx}+\frac{L}{\mu}\right).
	\end{equation*}
\end{theorem}
\begin{proof}
	Note that by choosing $\tauM=\frac{8L}{3\Mx}$ and $\gammaM=\frac{3}{16L}$ stepsizes satisfy the condition from Theorem~\ref{thm:INEXACTPPanyM} and the number of local iterations of \algname{GD} to guarantee convergence is:
	\begin{eqnarray*}
		\Kx\geq 2\frac{L_F+\frac{8L}{3\Mx}}{\frac{8L}{3\Mx}}\log\left(4\frac{L}{\mu}\right) = \left(2+\frac{3\Mx L_F}{4L}\right)\log\left(4\frac{L}{\mu}\right)=\mathcal{O}\left(\log\frac{L}{\mu}\right).
	\end{eqnarray*}
	Whereas, the number of communication rounds to reach $\varepsilon$-solution is:
	\begin{equation*}
		\max\left\{1+\frac{16}{3}\frac{L}{\mu}, \frac{\Mx}{\Cx}+\frac{3\Mx}{8\Cx}\frac{\Mx L_F}{L}\right\}\log\frac{1}{\epsilon}=\mathcal{O}\left(\left(\frac{\Mx}{\Cx}+\frac{L}{\mu}\right)\log\frac{1}{\epsilon}\right).
	\end{equation*}
\end{proof}
\subsection{Proof of Theorem~\ref{thm:relation}}
\begin{theorem}
	Consider Algorithm~\ref{alg:5GCS} (\algname{5GCS}) with the LT solver being \algname{GD} run for $\Kx \geq \Kx(\deltaM)\eqdef 2\deltaM \log\left(\nicefrac{4L}{\mu}\right)$ iterations, where  $1<\deltaM< 1+\frac{3}{8}\sqrt{\frac{L\Cx}{\mu \Mx}}$ . Let $\gammaM=\frac{1}{2\Mx\tauM}$  and $\tauM=\max\left\{\frac{L}{\Mx (\deltaM-1)},\frac{8}{3}\sqrt{\frac{L\mu}{\Mx \Cx}}\right\}.$
	Then for the Lyapunov function
	\begin{equation*}
			\Psi^{\kstep}\eqdef \frac{1}{\gammaM}\sqnorm{x^{\kstep}-x^\star}+\frac{\Mx}{\Cx}\left(\frac{1}{\tauM}+\frac{1}{L_F}\right)\sqnorm{u^{\kstep}-u^\star}\label{PPINEXACTpsianyM},
	\end{equation*}
	the iterates of the method satisfy
	$
	\Exp{\Psi^{\Tx}}\leq (1-\rho)^\Tx \Psi^0,
	$
	where 
	$	\rho\eqdef  \max\left\{\frac{\gammaM\mu}{1+\gammaM\mu},\frac{\Cx}{\Mx}\frac{\tauM}{(L_F+\tauM)}\right\}<1.
	$

\end{theorem}
\begin{proof}
	Firstly, we can note that at each step we need to find $\betaM$-solution to the local problem \myref{localprob}. Here, noting that we can restrict ourself to $\tauM\geq \frac{8}{3}\sqrt{\frac{L\mu}{\Mx \Cx}}$ since for this choice we get optimal number of communication rounds, thus we can note:
	\begin{eqnarray*}
		&6\frac{L}{\mu} \leq \frac{1}{\betaM}\leq \left(4\frac{L}{\mu}\right)^2&\\
		& \log\left(6\frac{L}{\mu}\right)\leq \log\frac{1}{\betaM}\leq 2\log\left(4\frac{L}{\mu}\right)&.
	\end{eqnarray*}
	Thus, the speed of local convergence depends fully on the condition number of the local problem $\left(\text{i.e., on }\frac{L_F+\tauM}{\tauM}\right)$. 
	For general result we can ask for the guarantee such that
	\begin{equation*}
		\Kx\geq \Kx(\deltaM)\eqdef \deltaM\left(2\log\left(4\frac{L}{\mu}\right)\right),\quad\deltaM>1.
	\end{equation*}
 For that we would need:
	\begin{equation*}
		\frac{L_F+\tauM}{\tauM}\leq\frac{\frac{L}{\Mx}+\tauM}{\tauM} \leq\deltaM \implies \tauM\geq \frac{\frac{L}{\Mx}}{\deltaM-1}.
	\end{equation*}
We use the choice 
		\begin{equation*}
			\gammaM\leq \frac{1}{\Mx\tauM}\left(1-\frac{4\mu}{3\Mx\tauM}\right).
		\end{equation*}
		Thus if $\tauM\geq \frac{8\mu}{3\Mx}$, then we can choose $\gammaM=\frac{1}{2\Mx\tauM}$.	Thus, let us take $\tauM=\max\left\{\frac{\frac{L}{\Mx}}{\deltaM-1},\frac{8}{3}\sqrt{\frac{L\mu}{\Mx \Cx}}\right\}$, so that we can choose $\gammaM=\frac{1}{2\Mx\tauM}$. With $\Kx$ \algname{GD} local iterations and this stepsize choice the contraction of the Lyapunov function follows from Theorem \ref{thm:INEXACTPPanyM}.
	\end{proof}
\subsection{Proof of Corollary~\ref{cor:rela}}
\begin{corollary}
	Choose any $0<\varepsilon<1$. In order to guarantee $\Exp{\Psi^{\Tx}}\leq \varepsilon \Psi^0$, it suffices to take 
	\begin{equation*}
		\squeeze	T\geq	\max\left\{1+\frac{2L}{\left(\deltaM-1\right)\mu},\frac{\Mx}{\Cx}\deltaM\right\}\log\frac{1}{\epsilon}.
	\end{equation*}
	We can note that when $\deltaM\leq \frac{\Mx+\Cx}{2\Mx}+\sqrt{\frac{2L\Cx}{\mu \Mx}+\left(\frac{\Mx-\Cx}{2\Mx}\right)^2}$, then
	\begin{equation*}
		\squeeze 	\Tx \geq	T(\alpha)\eqdef \left(1+\frac{2}{\deltaM-1} \frac{L}{\mu}\right)\log \frac{1}{\varepsilon}.
	\end{equation*}
\end{corollary}
\begin{proof}
		To satsify Assumption \ref{ass:GTPS}, assume that the Local Solver is \algname{GD} run for 
		\begin{equation*}
			\Kx\geq \Kx(\deltaM)\eqdef \deltaM\left(2\log\left(4\frac{L}{\mu}\right)\right),\quad\deltaM>1.
		\end{equation*}
	To ensure that choose $\tauM=\max\left\{\frac{\frac{L}{\Mx}}{\deltaM-1},\frac{8}{3}\sqrt{\frac{L\mu}{\Mx \Cx}}\right\}$ and $\gammaM=\frac{1}{2\Mx\tauM}$. Than the communication complexity is:
	\begin{eqnarray*}
		\max\left\{1+\frac{1}{\gammaM\mu},\frac{\Mx}{\Cx}+\frac{\Mx}{\Cx}\frac{L_F}{\tauM}\right\}
		\leq\max\left\{\max\left\{1+\frac{2L}{\left(\deltaM-1\right)\mu},1+\frac{16}{3}\sqrt{\frac{L\Mx}{\mu \Cx}}\right\},\frac{\Mx}{\Cx}\min\left\{\deltaM,1+\frac{3}{8}\sqrt{\frac{L\Cx}{\mu \Mx}}\right\}\right\}.
	\end{eqnarray*}
	For $\deltaM\leq 1+\frac{3}{8}\sqrt{\frac{L\Cx}{\mu \Mx}}$, this simplifies to:
	\begin{equation*}
	T\geq	\max\left\{1+\frac{2L}{\left(\deltaM-1\right)\mu},\frac{\Mx}{\Cx}\deltaM\right\}\log\frac{1}{\epsilon}.
	\end{equation*}
	We can note that when $\deltaM\leq \frac{\Mx+\Cx}{2\Mx}+\sqrt{\frac{2L\Cx}{\mu \Mx}+\left(\frac{\Mx-\Cx}{2\Mx}\right)^2}$, then:
	\begin{equation*}
	T\geq	\max\left\{1+\frac{2L}{\left(\deltaM-1\right)\mu},\frac{\Mx}{\Cx}\deltaM\right\}\log\frac{1}{\epsilon}=\left(1+\frac{2L}{\left(\deltaM-1\right)\mu}\right)\log\frac{1}{\epsilon}.
	\end{equation*}
	Thus, we get a relation  between the number of local steps and communication rounds.
\end{proof}

\clearpage
\section{Implementation-friendly version of Algorithm~\ref{alg:5GCS}}

We now present Algorithm~\ref{alg:membet}, which is Algorithm~\ref{alg:5GCS} written in a memory-efficient manner. We use the fact that we do not need any information on specific $u_\iclient^t$ and that not all $u_\iclient^t$ are updated in each communication round.
\begin{algorithm}[H]
	\caption{\algn{Client sampling with a new update for $u$ and memory-efficient update for $v$} [new]}
	\begin{algorithmic}[1]\label{alg:membet}
		\STATE  \textbf{input:} initial points $x^0\in\mathbb{R}^d$, $u_\iclient^0\in\mathbb{R}^d$ for all $\iclient=\{1,\dots,\Mx\}$; 
		\STATE stepsize $\gammaM>0$, $\tauM>0$;  $\Cx\in \{1,\dots,\Mx\}$
		\STATE $v^0\eqdef \sum_{\iclient=1}^\Mx u_\iclient^0$
		\FOR{$t=0, 1, \ldots$}
		\STATE $\hat{x}^{t} \eqdef \frac{1}{1+\gammaM\mu} \left(x^t - \gammaM v^t\right)$
		\STATE Pick $\set\subset \{1,\ldots,\Mx\}$ of size $\Cx$ uniformly at random
		\FOR{$\iclient\in \set$}
		\STATE Find $\lastlocittermk$ as a final point of $\Kx$ iteration of some Algorithm~$\mathcal{A}_\iclient$ starting with $y_\iclient^0=\hat{x}^t$ for following problem:
		\begin{eqnarray}
			\lastlocittermk \approx \argmin_{y\in\mathbb{R}^d}\left\{\localfuni(y) =  F_\iclient(y)+\frac{\tauM}{2} \sqnorm{y-\left(\hat{x}^\kstep+\frac{1}{\tauM}u_\iclient^t\right)}\right\}\label{localprob2}
		\end{eqnarray}
		\STATE $u_\iclient^{t+1}=\nabla F_\iclient(\lastlocittermk)$
		\STATE $\Delta u_\iclient^{t+1}=u_\iclient^{t+1}-u_\iclient^t$
		\ENDFOR
		\FOR{$\iclient\in\{1, \ldots,\Mx\}\backslash \set$}
		\STATE $u_{\iclient}^{t+1}\eqdef u_{\iclient}^t $
		\ENDFOR
		\STATE $\Delta v^{t+1}\eqdef \sum_{\iclient\in\set} \Delta u_\iclient^{t+1}$
		\STATE  $x^{t+1} \eqdef \hat{x}^{t}- \gammaM \frac{\Mx}{\Cx} \Delta v^{t+1}$
		\STATE $v^{t+1}=v^t+\Delta v^{t+1} $
		\ENDFOR
	\end{algorithmic}
\end{algorithm}
\end{document}